\documentclass[final,12pt]{clear2024} 

\usepackage[utf8]{inputenc} 
\usepackage[T1]{fontenc}    
\usepackage{hyperref}       
\usepackage{booktabs}       
\usepackage{amsfonts}       
\usepackage{nicefrac}       
\usepackage{microtype}      
\usepackage{xcolor}         

\usepackage{float}
\usepackage{wrapfig}
\usepackage{adjustbox}

\usepackage{tcolorbox}
\definecolor{lightorange}{HTML}{ff7f2a}
\definecolor{lighterorange}{HTML}{ffe6d5}
\definecolor{lavender}{HTML}{E6E6FA}
\definecolor{watergreen}{HTML}{AFEEEE}
\definecolor{lightpastelpurple}{rgb}{0.69, 0.61, 0.85}
\newtcolorbox{summarybox}{colback=lavender,colframe=lightpastelpurple}

\definecolor{blueviolet}{HTML}{2D2F92}

\usepackage{enumitem}
\setlist{nosep}

\title[On the Identifiability of Quantized Factors]{On the Identifiability of Quantized Factors}

\usepackage{times}

\clearauthor{%
  \Name{Vitória Barin-Pacela} \\
  \addr FAIR, Meta; Mila; DIRO, Université de Montréal
  \AND
  \Name{Kartik Ahuja} \\
  \addr FAIR, Meta 
  \AND
  \Name{Simon Lacoste-Julien} \\
  \addr Mila; DIRO, Université de Montréal; Canada CIFAR AI Chair 
  \AND
  \Name{Pascal Vincent} \\
  \addr FAIR, Meta; Mila; DIRO, Université de Montréal; CIFAR  
}

\begin{document}
\maketitle

\begin{abstract}%
Disentanglement  aims to recover meaningful latent ground-truth factors from the observed distribution solely, and is formalized through the theory of identifiability. The identifiability of independent latent factors has been proven to be impossible in the unsupervised i.i.d. setting under a general nonlinear map from factors to observations.
In this work, however, we demonstrate that it is possible to recover \emph{quantized} latent factors under a generic nonlinear diffeomorphism. We only assume that the latent factors have \emph{independent discontinuities} in their density, without requiring the factors to be statistically independent. We introduce this novel form of identifiability, termed \emph{quantized factor identifiability}, and provide a comprehensive proof of the recovery of the quantized factors.
\end{abstract}


\section{Introduction}

A large part of intelligence is based on the ability to make sense of observed sensory data without explicit supervision. The goal of representation learning is, thus, to detect and model relevant structure in the distribution of observed data, and expose it into useful compact representations, to facilitate generalization and sample-efficient learning of subsequent tasks.
One long-standing goal in that respect has been that of structuring the representation into \emph{disentangled factors}~\citep{bengio2013representation}. These may be conceived of as ``natural'' ground truth, descriptive, or causal variables that underlie the observations.
A vector representation consisting of recovered disentangled factors may be viewed as corresponding to a natural Cartesian coordinate system for the observations, whereby each varying factor is associated with an axis.

Identifiability theory formalizes the foundations of disentanglement by precisely delimiting the conditions under which it is possible. Unsupervised disentanglement of latent factors has been found impossible in the general nonlinear setting in the absence of further inductive bias~\citep{locatello2019challenging}. This result echoes an older identification impossibility result on nonlinear Independent Component Analysis~\citep{hyvarinen1999nonlinear}.
As a result, much subsequent work has sidestepped the issue either via stronger inductive biases, such as more restrictive assumptions on the function that maps latent factors to observations~\citep{buchholz2022function,kivva2022identifiability, ahuja2022interventional, brady2023provably, lachapelle2023additive}, for instance sparsity of its Jacobian~\citep{moran2022identifiable, zheng2022identifiability, zheng2023generalizing},
or by turning to weakly supervised disentanglement, using some form of additional information (see related works in Appendix~\ref{sec:relatedwork}).

Provided that they corresponds to valid assumptions, inductive biases should undoubtedly be used in practice whenever available, as well as any additional supervisory signals.
However, in the present theoretical work, we revisit and tackle the problem of fully unsupervised identifiability of latent factors, the most challenging setting. We assume a generic smooth invertible nonlinear mapping: a diffeomorphism. No additional assumptions are made on the mapping, and the assumption of the factors being mutually independent is also discarded. 

Given the previous theoretical impossibility results for unsupervised identifiability under a diffeomorphism, a shift in our approach was necessary. We relax the notion of identifiability of continuous factors to that of identifiability of quantized continuous factors. 

The promise of quantized, grid-like representations has been argued empirically in both machine learning and neuroscience.
It has been suggested that the brain of humans and other animals organizes spatial knowledge and relational concepts into codes that have an hexagonal grid-like pattern \citep{constantinescu16, Whittington770495}.
In representation learning, vector quantization has shown enormous success in image generation \citep{vqvae}. Concurrent studies investigate empirically this explicit relationship with disentanglement \citep{hsu2023disentanglement}, and further explore quantization in grid-structured representations \citep{mentzer2023finite, irie2023topological, friede2023learning}.

However, none of these provide a supporting identifiability theory for quantized factors. 
In the present work, we first formalize this novel relaxed form of identifiability. We, then, provide a full proof of the identifiability of quantized factors under a general diffeomorphism. This is achieved by assuming, rather than the mutual independence of factors, the presence of \emph{independent discontinuities} in the joint probability density \footnote{{More precisely, these are \textit{non-removable discontinuities} in the PDF, as will be elaborated in Section~\ref{sec:discontinuities-preservation}.}} of the latent factors.

\begin{figure}[!ht]
    \centering
    \includegraphics[scale=0.26]{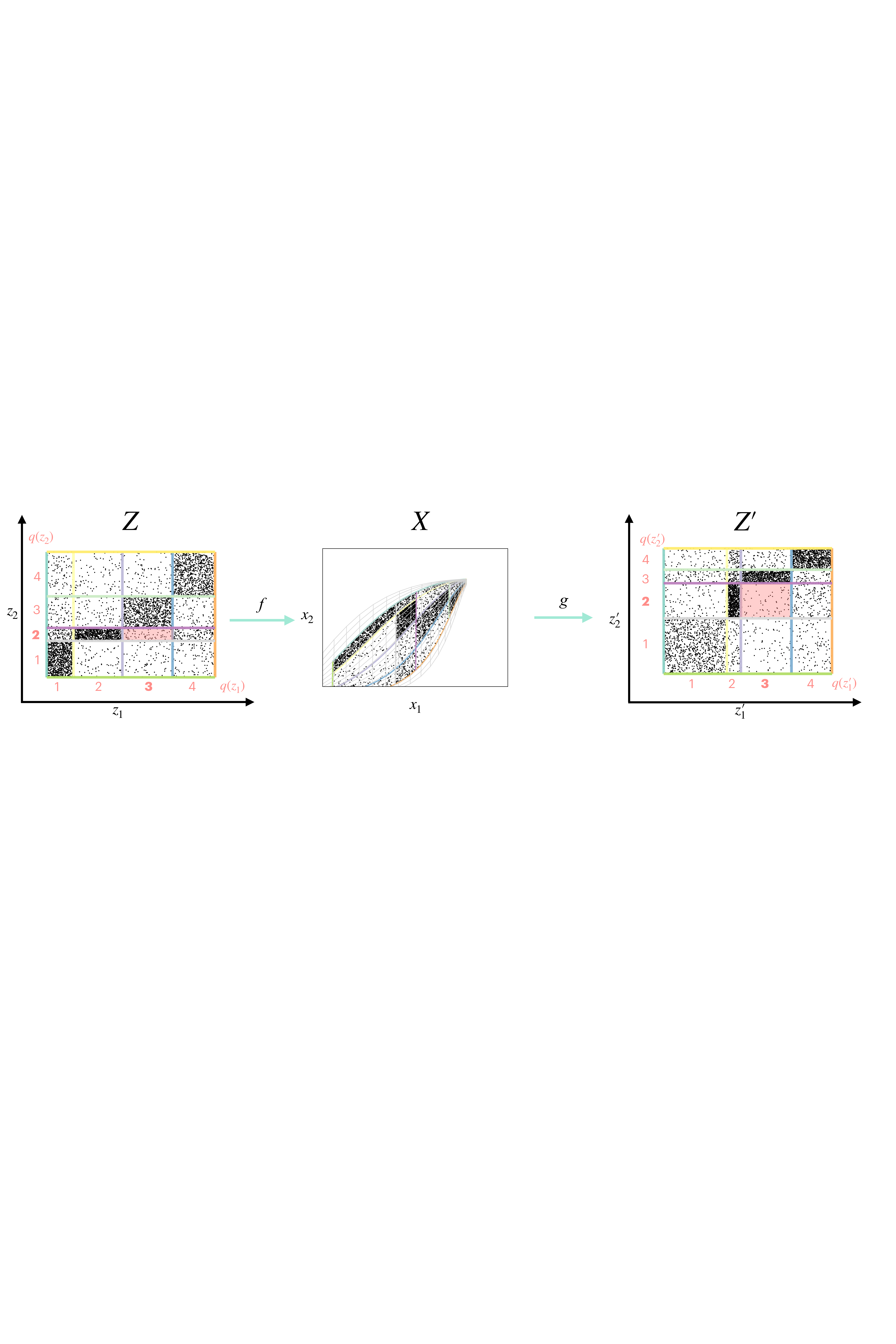}
\caption{
\textbf{ Recovery of quantized factors.
}
\textbf{Left:} The true (continuous) latent factors $Z_1$ and $Z_2$ are not independent, but their joint probability density $p_Z$ has \emph{independent discontinuities}: sharp changes in the density that are aligned with the axes and form a grid. \textbf{Middle:} The factors get warped and entangled by the diffeomorphism $f$ into observations $X$, but the discontinuities in their density survive in the observed space. \textbf{Right:} We can learn a diffeomorphism $g$ that yields a density $p_{Z'}$ having axis-aligned discontinuities. This suffices to recover a grid whose cells match the initial grid's cells (up to possible permutation and axis reversal). 
\textbf{{\color{pink}Pink cell} example:} the points $Z'$ in cell $(3,2)$ originated from the points $Z$ in cell $(3,2)$. To construct these cells, the quantization of each continuous factor to an integer depends on thresholds based on the location of the discontinuities. The quantizations of $Z'_1$ and $Z'_2$ match precisely the quantizations of $Z_1$ and $Z_2$, up to possible permutation and axis reversal. This summarizes the \emph{identifiability of quantized factors} under diffeomorphisms.
}
    \label{fig:main}
\end{figure}

Our contributions are the following:
\begin{itemize}
\item We introduce and formalize a novel relaxed form of representation identifiability: \emph{quantized factor identifiability}.
\item We provide the first proof of representation identifiability under a general diffeomorphic map, which sets itself apart from the impossibility results that dominate the field.
\end{itemize}

We hope that this novel theoretical foundation may provide useful insights to develop algorithms of practical relevance for robustly learning disentangled representations.

\section{From precise factor identifiability to quantized factor identifiability}

In this section, we define and contrast the standard form of factor identifiability, which we term ``precise factor identifiability'', with our new relaxed ``quantized factor identifiability'' paradigm.

\subsection{Setup}

We suppose that we have access to observations in $\mathcal{X}\subset\mathbb{R}^{D}$.
They are realizations of the vector random variable $X=(X_{1},\ldots,X_{D})$,
which is assumed to be a transformation of a real vector of unobserved
latent factors $Z=(Z_{1},\ldots,Z_{d})$, i.e. $X=f(Z)$, via a \emph{bijective} mapping
$f:\mathcal{Z}\rightarrow\mathcal{X}$ where $\mathcal{Z\subset\mathbb{R}}^{d}$.
The mapping $f$ is called the \emph{mixing map}, which is unknown but is assumed to belong
to a broad function class. The latent factors $Z$ follow a distribution represented by the probability density function (PDF) $p_{Z}$, which is also unknown but is typically subject to assumptions.
This induces a distribution for $X$ whose PDF is denoted by $p_{X}$.
The mapping $g:\mathcal{X}\rightarrow\mathcal{Z}$ approximates $f^{-1}$ at the optimum.

The setup is summarized in the following diagram. \vspace{-1mm}

\[
\overset{h=g\circ f}{  \underbrace{Z}_{\underset{\textrm{true factors}}{\subset\mathbb{R}^{d}}} \sim p_{Z} \overrightarrow{\xrightarrow[\mathrm{unknown}]{\;f\;} \underbrace{X}_{\underset{\textrm{observed data}}{\subset\mathbb{R}^{D}}}\xrightarrow[\mathrm{learned}]{g\approx f^{-1}} } 
\underbrace{Z'}_{\underset{\textrm{recovered factors}}{\subset\mathbb{R}^{d}}}  }
\]

In identifiability theory, the distribution of observations $p_X$ is supposedly known.
Alternatively, for this level of precision, observed samples from $p_X$ can be considered with the sample size approaching infinity.
Identifiability theorems also need to make clear assumptions on the mixing map $f$ and on the density of factors $p_X$.

In the remainder of this section, we first formalize the usual factor identifiability as well as our proposed relaxation to quantized factor identifiability in the general case. Subsequently, we focus on the case where $f$ is a diffeomorphism.

\subsection{Precise Factor Identifiability}

The usual, precise factor identifiability theorems amount to statements of the following form:

\begin{summarybox}
\paragraph{Precise Identifiability of Factors:}
Knowledge of $p_X$ is sufficient to determine a reverse mapping $g: \mathbb{R}^D \rightarrow \mathbb{R}^d$ that will yield recovered factors $(Z'_1, \ldots, Z'_d) = g(X)$ that correspond one-to-one to the ground-truth factors $(Z_1, \ldots, Z_d)$, up to permutation and component-wise invertible transformations (ideally monotonic). 
\end{summarybox}

Formally: there exists an indices permutation function $\sigma$ and invertible scalar functions $\gamma_i$ such that $\forall i \in \{1,\ldots,d\}, \gamma_i(Z'_i) = Z_j$ with $j=\sigma(i)$.
Precise factor identifiability theorems require specifying assumptions on $f$ and on $p_Z$.

\subsection{Quantization of factors}
\label{sec:quantizationop}

Let us now specify how the factors can be quantized.
For simplicity, we consider that each factor is a real-valued scalar. A real number $z$ may be quantized to an integer based on a tuple of real thresholds $T$ via the following quantization operation: 

\begin{equation}
    Q(z; T) = \sum_{k=1}^{|T|} \mathbf{1}_{z \geq T_k}. 
    \label{eq:quantization}
\end{equation}

For example, consider $z \in [0,4]$ and $T=(0.5, 2.0)$. Then $Q(z; T) = \mathbf{1}_{z \geq 0.5} + \mathbf{1}_{z \geq 2.0}$. So  $Q(z; T) = \begin{cases}
    0, \quad 0 \leq z < 0.5 \\
    1, \quad 0.5 \leq z < 2 \\
    2, \quad 2 \leq z \leq 4.
\end{cases}$

We also define quantization with order reversal as
$Q^-(z; T) = \sum_{k=1}^{|T|} \mathbf{1}_{z \leq T_k}$. 
For convenience, we will use the notation $Q^{(s)}$ to mean $Q$ if $s=+1$ and $Q^-$ if $s=-1$.

The set of specific thresholds used for quantizing a random variable $Z_i$ is typically derived from some properties of its distribution $p_{Z_i}$ (e.g. a set of $|T|$ specific quantiles).
We will consider the more general case, where the thresholds for quantizing $Z_i$ might be determined not only based on $p_{Z_i}$, but more generally, on $i$ and on the joint probability density of all factors $p_Z$.
The operation returning a set of thresholds to be used for a factor $Z_i$ is denoted by
$\mathcal{T}(p_Z, i)$. Thus, the quantization of $Z_i$ may be written as: 
$q_i(Z_i) = Q(Z_i; \mathcal{T}(p_Z,i))$.

\subsection{Quantized Factor Identifiability}

Quantized factor identifiability theorems will  be statements of the following form:

\begin{summarybox}
\paragraph{Identifiability of Quantized Factors:}
Knowledge of $p_X$ is sufficient to determine a reverse mapping $g: \mathbb{R}^D \rightarrow \mathbb{R}^d$ that will yield recovered factors $(Z'_1, \ldots, Z'_d) = g(X)$ such that their quantization $(q'_1(Z'_1), \ldots, q'_d(Z'_d))$ will correspond one-to-one to the quantized ground-truth factors $(q_1(Z_1), \ldots, q_d(Z_d))$, up to possible permutation of indices and order reversal. 
\end{summarybox}

Formally, there exists an indices permutation function $\sigma$ and order-reversal indicators $s_i \in \{-1,+1\}$ such that: $\forall i \in \{1,\ldots,d\},\; q'_i(Z'_i) = q_j(Z_j)$, with $j=\sigma(i)$, where $q_j$ and $q'_i$ are monotonic quantization functions.
We can, more precisely, define $q_j$ as $q_j(Z_j) = Q(Z_j; \mathcal{T}(p_Z,j))$, and $q'_i(Z'_i) = Q^{s_i}(Z'_i; \mathcal{T}(p_{Z'},i))$.
The precise operation $\mathcal{T}$ that determines how the quantization thresholds are obtained from properties of the distributions remain to be specified by the particular quantized factor identifiability theorem.

Hence, quantized factor identifiability theorems require specifying assumptions on $f$, assumptions on $p_Z$, as well as a precise quantization operation.
For the quantization to be meaningful, it should produce at least two non-empty bins. That is, for any given factor, the respective factor samples will be mapped to at least two different quantized values. Quantization to a single all-encompassing bin is trivially identifiable and useless.

We highlight that quantized factor identifiability does not intend to prove identifiability when the true factors take a discrete set of values.
Instead, we define a relaxed form of identifiability for \emph{continuous} ground-truth factors. Quantization leads to a loss of precision/resolution, resulting in a coarser identification.

\section{What to assume on \texorpdfstring{$p_Z$ when} \texorpdfstring{$f$} is a diffeomorphism}

From now on, we will turn our attention to the case where the mixing map $f$ is assumed to be a general \emph{diffeomorphism}, that is, a continuously differentiable function with a continuously differentiable inverse. For the remaining of the paper, we will assume that all diffeomorphisms considered are smooth, i.e.\ $C^{\infty}$, in order to simplify the statements. The goal is to learn the approximate inverse diffeomorphism $g$. First, let us discuss what assumptions we should make on the distribution of factors $p_Z$ that may yield a positive identifiability result.

\subsection{Disentanglement, independence, and discontinuities}

Disentanglement has been equated to finding statistically independent factors \citep{khemakhem2020variational}, but statistical independence has been criticized as an unrealistic and problematic assumption \citep{ traeuble2021correlation,dittadi2021realistic,roth2022disentanglement} whose association to disentanglement is misleading. For example, the usual \emph{descriptive factors} with which we describe scenes are usually not statistically independent. Consider the variables color, shape, and background: bananas tend to be yellow; cows tend to be on grass backgrounds; camels tend to be on sand. 

Moreover, a primary interest for learning disentangled representations is as an enabler of robust generalization under distribution shifts. From this perspective, we should aim for factor discovery approaches that are stable and insensitive to broad changes in the (unknown) distribution of the factors, such as whether they happen to be independent or correlated in the data. Aiming for extracting statistically independent factors will, by construction, be very sensitive to this, which goes contrary to the desired robustness.

Lastly, and most importantly for our goal of characterizing what form of unsupervised identification is possible under a diffeomorphism, assuming statistical independence is insufficient, as previous impossibility results for nonlinear ICA have shown \citep{hyvarinen1999nonlinear}. 
This is fundamentally due to the extreme flexibility of diffeomorphisms. Even more discouraging, \citet{buchholz2022function} have shown that even if we knew $p_Z$ precisely, we could not achieve precise factor identifiability. This is because a diffeomorphism can move data points along isosurfaces of $p_Z$ while keeping the same $p_Z$, thus rendering entangled representations indistinguishable from disentangled ones.

To prevent this movement along isolines of $p_Z$ from taking points from one region to another of the factor space, there could be barriers of discontinuity separating different regions of $p_Z$.
We develop this justification for the need for discontinuities more precisely in Appendix \ref{app:necessity_landmarks}, based on the result from \citet{buchholz2022function}.

Another perspective to consider is that discontinuities are among the few characteristics of a density that diffeomorphisms can neither erase nor create (Theorem \ref{thm:discontinuities-preservation}). Hence, they are good candidates for holding cues in $p_Z$ guaranteed to survive in some form when mapped to $X$ via any diffeomorphism. Thus, if they were indicative of coordinate axes in $Z$, there is a prospect of recovering them from the resulting discontinuities in $p_X$.

Therefore, to enable the identifiability of quantized factors under a diffeomorphism, we will not assume that $p_Z$ implies statistically independent factors, but rather that it has \emph{independent discontinuities}, which we define precisely in the next section. 

\subsection{Independent discontinuities in the probability density}

Here, we contrast the statistical independence of factors with our approach, the independence of discontinuities. We will assume that there are discontinuities in the PDF of the factors, and that the location of these discontinuities in the density of any given factor is independent of the values of all the other factors. 

\begin{summarybox}
\begin{definition}
Let $\mathcal{S}$ be the support of $p_Z$. We say that $p_Z$ has an \textbf{independent discontinuity} at $Z_i=\tau$ when
every point in the intersection of the \emph{coordinate hyperplane} $\{\mathbf{z}_i = \tau\}$ with $\mathcal{S}$ is a non-removable discontinuity of $p_Z$. Formally, this independent discontinuity at $Z_i=\tau$ is defined as the set $\Gamma_\mathcal{S}(i,\tau)=\{\mathbf{z}\in \mathcal{S}|\mathbf{z}_i=\tau\}$ under the condition that $\forall \mathbf{z} \in \Gamma_\mathcal{S}(i,\tau)$, $p_Z$ has a non-removable discontinuity at $\mathbf{z}$.
\end{definition}
\end{summarybox}

Such discontinuities are ``independent'' in the sense that we have a discontinuity at $Z_i=\tau$ regardless of the values taken by the other factors.
Only the locations of the discontinuities in the density need to be independent of the other factors. This does not impose statistical independence of the factors, nor anything else wherever the density is continuous. Thus, assuming the presence of independent discontinuities can accommodate statistically independent factors as well as correlated factors. 

Geometrically, an independent discontinuity in $p_Z$ corresponds to a \emph{coordinate hyperplane} restricted to the support of $p_Z$.
This hyperplane is orthogonal to the $Z_i$ axis and parallel to all the other axes. We will, thus, interchangeably call it an \textbf{independent discontinuity} or \textbf{axis-aligned discontinuity}.

For our theorems, we will further require that the interior of the support of the density is connected. A connected independent discontinuity that splits this support in two is said to be an \textbf{axis-separator} (formally defined in Section~\ref{sec:grid-structure-recovery}). 
If the set of all non-removable discontinuities of $p_Z$ is the union of a finite set of such axis-separators, with at least one along each axis, then we say that they form an \textbf{axis-aligned grid}.
Figure~\ref{fig:main} (left) gives an example of two factors that are clearly not statistically independent but that have independent discontinuities in their PDF, appearing to the eye as axis-aligned discontinuities along each axis. Altogether, they from an axis-aligned grid.

Independent discontinuities are striking landmarks in the PDF landscape $p_Z$ that remain detectable in $p_{X}$ and $p_{Z'}$. A diffeomorphic map, even though it can warp the space in almost arbitrary ways, will not be able to erase such discontinuities. These are the robust cues that we can rely on to achieve quantized factor disentanglement under a diffeomorphism.

\section{Overview of the main quantized identifiability result}

In a nutshell:

\begin{summarybox}
\label{sec:discretized-coordinates-identifiability}

\paragraph{Assumptions}
\begin{itemize}
    \item $f$ is a diffeomorphism
    \item $(Z_1,\ldots,Z_d) \sim p_Z$ are $d$ continuous random variables.
    \item The interior of the support of $p_Z$ is a connected set.
    \item The set of non-removable discontinuities of $p_Z$ is the union of a finite set of independent discontinuities -- at least one along each dimension -- that together form an \emph{axis-aligned grid}. This grid must also possess a \emph{backbone} (precisely defined in the next section).
\end{itemize}
\end{summarybox}

\begin{summarybox}
\paragraph{Quantized factor identifiability theorem}
Under the above assumptions:
\begin{itemize}
    \item It suffices to learn a diffeomorphism $g$ yielding $Z'=g(X)$ such that the PDF of $p_{Z'}$ has independent discontinuities forming an axis-aligned grid.
    \item Then, the quantized reconstructed factors $(q'_1(Z'_1), \ldots, q'_d(Z'_d))$ will correspond one-to-one to the quantized ground-truth factors $(q_1(Z_1), \ldots, q_d(Z_d))$, up to possible permutation of indices (and order reversal). 
    \item The quantization thresholds used for $q_i$ and $q'_i$ are obtained as the locations of the independent discontinuities.
\end{itemize}
\end{summarybox}

This result is illustrated in Figure~\ref{fig:main}.
The formal \textbf{Quantized factor identifiability theorem} is Theorem~\ref{thm:discretized-coordinates-identifiability}, found in section~\ref{app:proof_quantized} together with its proof. It builds on two other theorems: the \emph{Non-removable discontinuity preservation theorem} (Theorem~\ref{thm:discontinuities-preservation}  in Section~\ref{sec:discontinuities-preservation})
and the \emph{Grid structure recovery theorem} (Theorem~\ref{thm:gridStructure} in Section~\ref{sec:grid-structure-recovery}) and its corollary.
The following section presents these theorems and the required definitions in their logical  order.

\section{Main theorems}

Most of the theory will concern the diffeomorphism $h:=g \circ f$ that maps $Z$ to $Z'$.

\subsection{Non-removable discontinuity preservation theorem}
\label{sec:discontinuities-preservation}

We will show that discontinuities in the PDF are preserved by a a diffeomorphism.
However, one subtlety is that the PDF corresponding
to a given distribution is not unique, as elaborated in Appendix~\ref{app:nonremovable}. The many PDFs representing the same distribution actually form an equivalence class, whose elements may take arbitrarily different values on sets of points of measure zero. So not all the discontinuities in a PDF are meaningful. Since
we care about observable characteristics of the actual distribution, we must focus on aspects of the PDF that are immune to erasure by changes of measure zero. We use the following definitions:

\begin{definition}
\textbf{Removable discontinuity:} A PDF $p$ has a removable discontinuity
at $z$ if $p$ is discontinuous at $z$ but there exists another
 $p'$ in the same equivalence class (i.e. $p$ and $p'$ yield
the exact same probability measure) that is continuous at $z$.
\end{definition}

\begin{definition}
\textbf{Non-removable discontinuity:} A PDF $p$ has a non-removable discontinuity at $z$ if $p$ is discontinuous at $z$ but this discontinuity is not removable. Equivalently, \emph{all} PDFs in the equivalence class
of $p$ are discontinuous at $z$. Note that a non-removable discontinuitiy is a property of an equivalence class of PDFs, thus of the distribution, not just of a single PDF.
\end{definition}

\begin{theorem}
\label{thm:discontinuities-preservation}
    \textbf{Non-removable discontinuity preservation theorem.}
Let $Z$ be a latent random variable with values in $\mathcal{Z}\subset\mathbb{R}^{d}$,
whose distribution is represented by a PDF $p_{Z}$. Let $h:\mathcal{Z}\rightarrow\mathcal{Z}'\subset\mathbb{R}^{d}$ be a diffeomorphism, and let $Z'=h(Z)$ be a transformed random variable whose distribution is represented by a probability density function $p_{Z'}$. Then, $p_{Z'}$ has a non-removable discontinuity at a point $z'$ if and only if $p_{Z}$ has a non-removable discontinuity at the point $z=h^{-1}(z')$.
\end{theorem}

\textbf{Proof:} Appendix \ref{app:proof_nonremovable}.

\subsection{Grid structure recovery theorem and corollary}
\label{sec:grid-structure-recovery}

\subsubsection{Definition of grid structure}

The notions we use to define the grid structure are related to usual hyperplanes
and hypersurfaces of $\mathbb{R}^{d}$, but they are restricted to
a connected subset $\mathcal{S}$ of $\mathbb{R}^{d}$. In our setting, $\mathcal{S}$ will be the interior of the support of the density on which the grids
can be defined. 
Let $\mathcal{S}\subset\mathbb{R}^{d}$ be a connected open smooth submanifold of dimension $d$ ($\mathcal{S}$ such as an open $d$-ball).
We will use the following definitions, which are \textbf{illustrated in Appendix~\ref{app:illustrations_def}}.
\vspace{-5pt}
\begin{definition}
    The \textbf{splitting} of a set $\mathcal{S}$ by another set $\mathcal{C}$,
denoted $\mathrm{split}(\mathcal{S},\mathcal{C})$, is the set of
\emph{connected components} of $\mathcal{S}\setminus\mathcal{C}$.
\end{definition}

\begin{definition}
    We say that \textbf{$\mathcal{C}$ splits $\mathcal{S}$ in two} to mean $\mathrm{|split}(\mathcal{S},\mathcal{C})|=2$ (we denote the cardinality of a countable set $A$ by $|A|$, and similarly for the number of elements in an ordered list or a tuple).
\end{definition}

\begin{definition}
    We say that $\mathcal{C}$ is a \textbf{separator} of $\mathcal{S}$ if $\mathcal{C}$ is a \emph{connected} \emph{subset} of $\mathcal{S}$ and $\mathcal{C}$ splits $\mathcal{S}$ in two. The two connected components that result from the split are called the two \textbf{halves} resulting form the split, denoted $\mathcal{C}^{+}$ and $\mathcal{C}^{-}$, i.e. $\{\mathcal{C}^{+},\mathcal{C}^{-}\}=\mathrm{split}(\mathcal{S},\mathcal{C})$.
\end{definition}

\begin{definition}
    $\mathcal{C}$ is a\textbf{ smooth separator }of $\mathcal{S}$ if
$\mathcal{C}$ is a \emph{separator} of $\mathcal{S}$ and is a \emph{smooth
hypersurface} of $\mathcal{S}$ (i.e. a smooth embedded submanifold
of dimension $d-1$).
\end{definition}

\begin{definition}
    An\textbf{ axis-separator} of $\mathcal{S}$ is a special case of
\emph{smooth separator} of $\mathcal{S}$ that is the intersection
of $\mathcal{S}$ with an \emph{axis-aligned hyperplane} of $\mathbb{R}^{d}$
(a coordinate hyperplane). It can be defined as $\mathcal{H}=\Gamma_{\mathcal{S}}(i,\tau)=\{z\in\mathcal{S}|z_{i}=\tau\}$ (Figure \ref{fig:setup1}).
Because it is a separator, it splits $\mathcal{S}$ in two halves
$\Gamma_{\mathcal{S}}^{+}(i,\tau)=\{z\in\mathcal{S}|z_{i}>\tau\}$
and $\Gamma_{\mathcal{S}}^{-}(i,\tau)=\{z\in\mathcal{S}|z_{i}<\tau\}$,
which are each nonempty and connected \emph{(Figure \ref{fig:setup3})}. 
\label{def:axis-separator}
\end{definition} 

\begin{definition}
    An \textbf{axis-separator-set} $\mathcal{G}$ on $\mathcal{S}$ is
a finite set of axis separators of $\mathcal{S}$.
\end{definition}

\begin{definition}
    An axis-aligned \textbf{grid} $G\subset\mathcal{S}$ is a subset of
$\mathcal{S}$ that can be obtained as a union of all the separators
in an axis-separator-set $\mathcal{G}$. i.e. $G=\cup\mathcal{G}=\cup_{H\in \mathcal{G}} H$.
\textbf{}\\
\emph{Note the important distinction we make between a \emph{grid}, which is a subset of $\mathcal{S}$ and hence a set of points, and an $\emph{axis-separator-set}$, which is a set of axis separators (which themselves are sets of points). An \emph{axis-separator-set} thus has more explicit structure than a $grid$. The proof we will
unroll depends conceptually on the ability to rebuild, in several steps, the entire grid internal structure, starting from only the unstructured \emph{grid }as a set of points. The first step of this program will be the recoverability of \emph{axis-separator-set
}from \emph{grid.}}
\end{definition}

\begin{definition}
    A \textbf{parallel-separator-set} is a \emph{set} of axis-separators all defined on the same $i^{th}$ axis (and are thus parallel). In particular, we denote the subset of axis-separator set $\mathcal{G}$ that are all defined on the $i^{th}$ axis as $\mathcal{G}^{(i)}$.
\end{definition}

\begin{definition}
    A \textbf{discrete coordination} $\mathbf{A}$
is a tuple $\mathbf{A} = (\mathbf{A}_{1},\ldots,\mathbf{A}_{d})$ where
each $\mathbf{A}_{i}$ is itself a tuple of real numbers in increasing order $\mathbf{A}_{i}=(\mathbf{A}_{i,1},\ldots,\mathbf{A}_{i,n_{i}})$
such that $\mathbf{A}_{i,k+1}>\mathbf{A}_{i,k}$. These represent the coordinates of axis-separators along each of the $d$ coordinate axes (Figure \ref{fig:setup2}). \\
\emph{
{\bf Note:} $\mathbf{A}_i$ contains the list of quantization \emph{thresholds} to quantize the $i^{th}$ coordinate (or factor) as  $Q(Z_i; \mathbf{A}_i)$, as defined in equation \ref{eq:quantization}.\\
A discrete coordination defines the entire grid structure. One can easily obtain the various constituent sets from it:
\begin{enumerate}[label=(\alph*)]
    \item the individual \textit{separators} ($\approx$``hyperplanes'') $\Gamma_{\mathcal{S}}(i,\mathbf{A}_{i,k})$, and their positive and negative halves ($\approx$``half-spaces'') $\Gamma_{\mathcal{S}}^{+}(i,\mathbf{A}_{i,k})$ and $\Gamma_{\mathcal{S}}^{-}(i,\mathbf{A}_{i,k})$  respectively;
    \item the \emph{parallel-separator-sets} $\mathcal{G}^{(1)},\ldots,\mathcal{G}^{(d)}$, where $\mathcal{G}^{(i)} = \{\Gamma_{\mathcal{S}}(i,\mathbf{A}_{i,k})\}_{k=1}^{|\mathbf{A}_{i}|}$;
    \item the \emph{axis-separator-set} $\mathcal{G}=\mathcal{G}^{(1)}\cup\ldots\cup\mathcal{G}^{(d)}$; 
    \item the \emph{grid} $G=\mathrm{grid}_{\mathcal{S}}(\mathbf{A})=\cup\mathcal{G}$.
\end{enumerate}
}
\label{def:discrete_coordination}
\end{definition}

\begin{definition}
    A \textbf{backbone} $\mathcal{H}^{*}$ of a grid is a \emph{list} $\mathcal{H}^{*}=(\mathcal{H}_{1}^{*},\ldots,\mathcal{H}_{d}^{*})$ of $d$ separators of that grid, each defined on the corresponding axis, that have a non-empty intersection (they meet at a single point $z^{*}$). That is, for $\text{\ensuremath{\mathcal{H}_{1}^{*}}}\in\mathcal{G}^{(1)},\ldots,\text{\ensuremath{\mathcal{H}_{d}^{*}}}\in\mathcal{G}^{(d)}$, we must have $\bigcap_{i=1}^{d}\text{\ensuremath{\mathcal{H}_{i}^{*}}}=\{z^{*}\}$.
    In addition, for $\mathcal{H}^{*}$ to be a backbone, it is also required that each of its separators $\mathcal{H}_{i}^{*}$ intersect \emph{all} the other separators $H\in\mathcal{G}^{(j)}$ of the grid that are defined on the other axes $j\ne i$ (those not in the same parallel-separator-set); namely,
    $\forall i,\forall j\ne i,\forall H\in\mathcal{G}^{(j)},\text{\ensuremath{\mathcal{H}_{i}^{*}}}\cap H\ne\emptyset$
    \emph{(example in Figure \ref{fig:backbone}).} \\
    \emph{A backbone functions as a set of ``main axes'', and we will require a proper \textbf{grid} to have at least one backbone. This is a weaker requirement than requiring a ``complete grid'' where \emph{each separator of the grid} would be required to intersect all the separators that are not in the same parallel-separator-set. Here, we require only that the separators of the backbone intersect all the other separators on the other axes of the grid.}
\label{def:backbone}
\end{definition}

\subsubsection{Grid structure preservation and recovery theorem}

\begin{theorem}
    \label{thm:gridStructure}
\textbf{Grid structure preservation and recovery theorem.} 
Let $h: \mathcal{S}\subset \mathbb{R}^d \rightarrow \mathcal{S}' \subset \mathbb{R}^d$ be a diffeomorphism, where
both $\mathcal{S}$ and $\mathcal{S}'$ are open connected subsets of $\mathbb{R}^d$. 
Suppose we have an axis-aligned grid $G\subset\mathcal{S}$, associated
with its axis-separator-set $\mathcal{G}$ and discrete coordination
$\mathbf{A}$, that is, $G=\mathrm{grid}_{\mathcal{S}}(\mathbf{A})$. While
the grid does not need to be ``complete'', we suppose that $\mathcal{G}$ has
at least one \emph{backbone}. Now, suppose that we have another axis-aligned
grid in $\mathcal{S}'$, associated with its discrete coordination
$\mathbf{B}$, with $G'=\mathrm{grid}_{\mathcal{S}'}(\mathbf{B})$.
Suppose $G'=h(G)$. Then, there exists a permutation function $\mathrm{\sigma}$
over dimension indexes $1,\ldots,d$ and a direction reversal vector
$s\in\{-1,+1\}^{d}$ such that $\forall j\in\{1,\ldots,d\},\;i=\sigma^{-1}(j),\;K=|\mathbf{A}_{i}|=|\mathbf{B}_{j}|$, $\forall k\in\{1,\ldots,K\},\forall z'\in\mathcal{S}'$,

If $s_{i}=+1$, then: 
$$\begin{cases}
    z'_{j}=\mathbf{B}{}_{j,k}\Longleftrightarrow h^{-1}(z')_{i}=\mathbf{A}_{i,k}, \\
    z'_{j}>\mathbf{B}_{j,k}\Longleftrightarrow h^{-1}(z')_{i}>\mathbf{A}_{i,k}, \\
    z'_{j}<\mathbf{B}_{j,k}\Longleftrightarrow h^{-1}(z')_{i}<\mathbf{A}_{i,k};\\
\end{cases}
$$

If $s_{i}=-1$, then:  
$$\begin{cases}
    z'_{j}=\mathbf{B}_{j,k}\Longleftrightarrow h^{-1}(z')_{i}=\mathbf{A}_{i,K-k+1}, \\
    z'_{j}>\mathbf{B}_{j,k}\Longleftrightarrow h^{-1}(z')_{i}<\mathbf{A}_{i,K-k+1}, \\
    z'_{j}<\mathbf{B}_{j,k}\Longleftrightarrow h^{-1}(z')_{i}>\mathbf{A}_{i,K-k+1}. \\
\end{cases}
$$
\end{theorem}

\paragraph{Principle of the proof} 

Starting from the premise $G'=h(G)$, we know that $h$ maps every point
of $G$ to a point of $G'$. The proof recovers the entire underlying
grid \emph{structure} in 3 steps: 
\begin{description}
\item [{Step 1}] recover a one-to-one mapping of the individual separators:
$\mathcal{G}'=h(\mathcal{G})$.
\item [{Step 2}] recover the partition into subsets of parallel separators
(each subset associated to an axis): $\mathcal{G}'^{(j)}=h(\mathcal{G}^{(i)})$
(with permutation $j=\sigma(i)$).
\item [{Step 3}] show that the ordering of the separators in a parallel-separators-set
is preserved (up to possible order reversal): \\
$[h(\Gamma_{\mathcal{S}}(i,\mathbf{A}_{i,1})),\ldots,h(\Gamma_{\mathcal{S}}(i,\mathbf{A}_{i,K}))]=[\Gamma_{\mathcal{S}'}(j,\mathbf{B}_{j,1}),\ldots,\Gamma_{\mathcal{S}'}(j,\mathbf{B}_{j,K})]$
or in reversed order $[h(\Gamma_{\mathcal{S}}(i,\mathbf{A}_{i,1})),\ldots,h(\Gamma_{\mathcal{S}}(i,\mathbf{A}_{i,K}))]=[\Gamma_{\mathcal{S}'}(j,\mathbf{A}_{j,K}),\ldots,\Gamma_{\mathcal{S}'}(j,\mathbf{A}_{j,1})]$.
And similarly, the ordering of the halves corresponding to each of these separators is preserved.
Knowing to which half (either $\Gamma^+_{\mathcal{S}'}(j,\tau)$ or $\Gamma^-_{\mathcal{S}'}(j,\tau)$) a point $\mathbf{z}'$ belongs tells us whether $z'_j$ is above or below the threshold $\tau$. 
\end{description}

For example, seeing that $z'_j > \mathbf{B}_{j,k}$ tells us that $z' \in \Gamma^+_{\mathcal{S}'}(j,\mathbf{B}_{j,k})$, which implies from step 3 (in the case of no order reversal) that its preimage $z=h^{-1}(z')$ belongs to $ \Gamma^+_{\mathcal{S}}(i,\mathbf{A}_{i,k})$, which yields $z_i > \mathbf{A}_{i,k}$. This is what Theorem~\ref{thm:gridStructure} expresses.
We refer the reader to Appendix \ref{app:grid_struct_proof} for the full proof.

\begin{corollary}
\label{cor:quantized}
    \textbf{Recovery of quantized factors.}
Under the same premises as Theorem~\ref{thm:gridStructure}, consider random variables $Z$ and $Z' = h(Z)$.
Using the quantization operation $Q$ (previously defined in Section~\ref{sec:quantizationop}, equation \ref{eq:quantization}), we recover quantized factors up to permutation $\sigma$ of the axes and possible direction reversal indicated by
$s$: $\forall i\in {1,\ldots,d}$, 
$Q(Z_{i};\mathbf{A}_{i}) = 
Q^{s_i}(Z'_{j};\mathbf{B}_{j})$
with $j=\sigma(i)$.
\end{corollary}

\begin{proof}
$Z'=h(Z)$ implies that $Z_i = h^{-1}(Z')_i$. 
\\
Now if $s_i = +1$, 
Theorem~\ref{thm:gridStructure}  yields $Z'_j \geq \mathbf{B}_{j,k} \Longleftrightarrow 
 Z_i \geq \mathbf{A}_{i,k}$.\\
Thus, $Q(Z_i, \mathbf{A}_i) 
= \sum_k \mathbf{1}_{Z_i \geq \mathbf{A}_{i,k}} 
= \sum_k \mathbf{1}_{Z'_j \geq \mathbf{B}_{j,k}}
= Q(Z'_j; \mathbf{B}_j).$
\\
Similarly, if $s_i = -1$, 
Theorem~\ref{thm:gridStructure}  yields $Z'_j \leq \mathbf{B}_{j,k} \Longleftrightarrow 
 Z_i \geq \mathbf{A}_{i,k}$.\\
Thus, $Q(Z_i, \mathbf{A}_i) 
= \sum_k \mathbf{1}_{Z_i \geq \mathbf{A}_{i,k}} 
= \sum_k \mathbf{1}_{Z'_j \leq \mathbf{B}_{j,k}}
= Q^-(Z'_j; \mathbf{B}_j).$\\
So in both cases, we have 
$Q(Z_{i};\mathbf{A}_{i}) = 
Q^{s_i}(Z'_{j};\mathbf{B}_{j})$.
\end{proof}

\subsection{Quantized factor identifiability theorem}
\label{app:proof_quantized}

\begin{theorem}
    \label{thm:discretized-coordinates-identifiability}
\textbf{Quantized factors identifiability theorem.}
Let $Z$ be a latent random variable with values in $\mathcal{Z}\subset\mathbb{R}^{d}$
and whose PDF is $p_{Z}$. Let $f:\mathcal{Z}\rightarrow\mathcal{X}\subset\mathbb{R}^{D}$ be a diffeomorphism, and $X=f(Z)$ be the observed random variable. Assume that the support of the PDF $p_{Z}$ is an open connected set\footnote{Alternatively, if the support is not open, we can consider its interior.}. Further assume that $p_{Z}$ has at least one connected independent discontinuity in each dimension, such that the set of non-removable discontinuities of $p_{Z}$ forms an axis-aligned grid with a backbone. Let $\mathbf{A}$ be the discrete coordination of this grid.
Then, there exists a diffeomorphism $g:\mathcal{X}\rightarrow\mathcal{Z}'$
yielding a variable $Z'=g(X)$ such that the set of non-removable discontinuities of the PDF $p_{Z'}$ is an axis-aligned grid. Consider any such diffeomorphism $g$, and let \textbf{$\mathbf{B}$} be the discrete coordination
of its resulting axis-aligned grid. Then, there exists a permutation function $\mathrm{\sigma}$ over the dimension indexes $1,\ldots,d$, and a direction reversal vector $s\in\{-1,+1\}^{d}$ such that 
$q'_j(Z'_{j})=q_i(Z_i)$ with $i=\sigma^{-1}(j)$, where $q'_j(Z'_{j}) = Q^{s_i}(Z'_{j}; \mathbf{B}_{j})$ and $q_i(Z_i) = Q(Z_i; \mathbf{A}_{i})$.
In other words, the quantized factors in $Z'$ agree with the
quantized factors in $Z$, up to permutation and possible axis
reversal.
\end{theorem}

\begin{proof}
Note that existence is trivial (it suffices to take $g=f^{-1}$, which yields $Z'=Z$). But the fact that \emph{any} $g$ that yields a PDF whose non-removable discontinuities form an axis-aligned grid will
have this property can now easily be proven from our previous results. It suffices to consider
$h=g\circ f$ to be a diffeomorphism (the composition of two diffeomorphisms), so that $Z'=h(Z)$, and to combine the non-removable discontinuity preservation theorem~(Thm.~\ref{thm:discontinuities-preservation}) with the grid structure preservation and recovery theorem~(Thm.~\ref{thm:gridStructure}). Let $G=\mathrm{grid}_{\mathcal{S}}(\mathbf{A})$ and $G'=\mathrm{grid}_{\mathcal{S}}(\mathbf{B})$ be the set of non-removable discontinuity points of $p_{Z}$ and $p_{Z'}$, respectively. From the non-removable
discontinuity preservation theorem, we have that $G'=h(G)$. And from the grid structure preservation and recovery theorem and its corollary, we have that $G'=h(G)$ implies that there exists a permutation function
$\mathrm{\sigma}$ over dimension indexes $1,\ldots,d$ and a direction reversal vector $s\in\{-1,+1\}^{d}$ such that $Q^{s_i}(Z'_{j};\mathbf{B}_{j})=Q(Z_{i};\mathbf{A}_{i})\;\mathrm{\mathrm{with\;}}i=\sigma^{-1}(j)$.
We have, thus, proved that the quantized factors of $Z'$ agree with the quantized factors of $Z$, up to permutation and axis reversal. 
\end{proof}

\section{Independent discontinuities in real-world disentangled factors}

\begin{wrapfigure}[16]{r}{0.5\textwidth}
  \vspace{-\intextsep}
  \centering
    \includegraphics[width=0.5\textwidth]{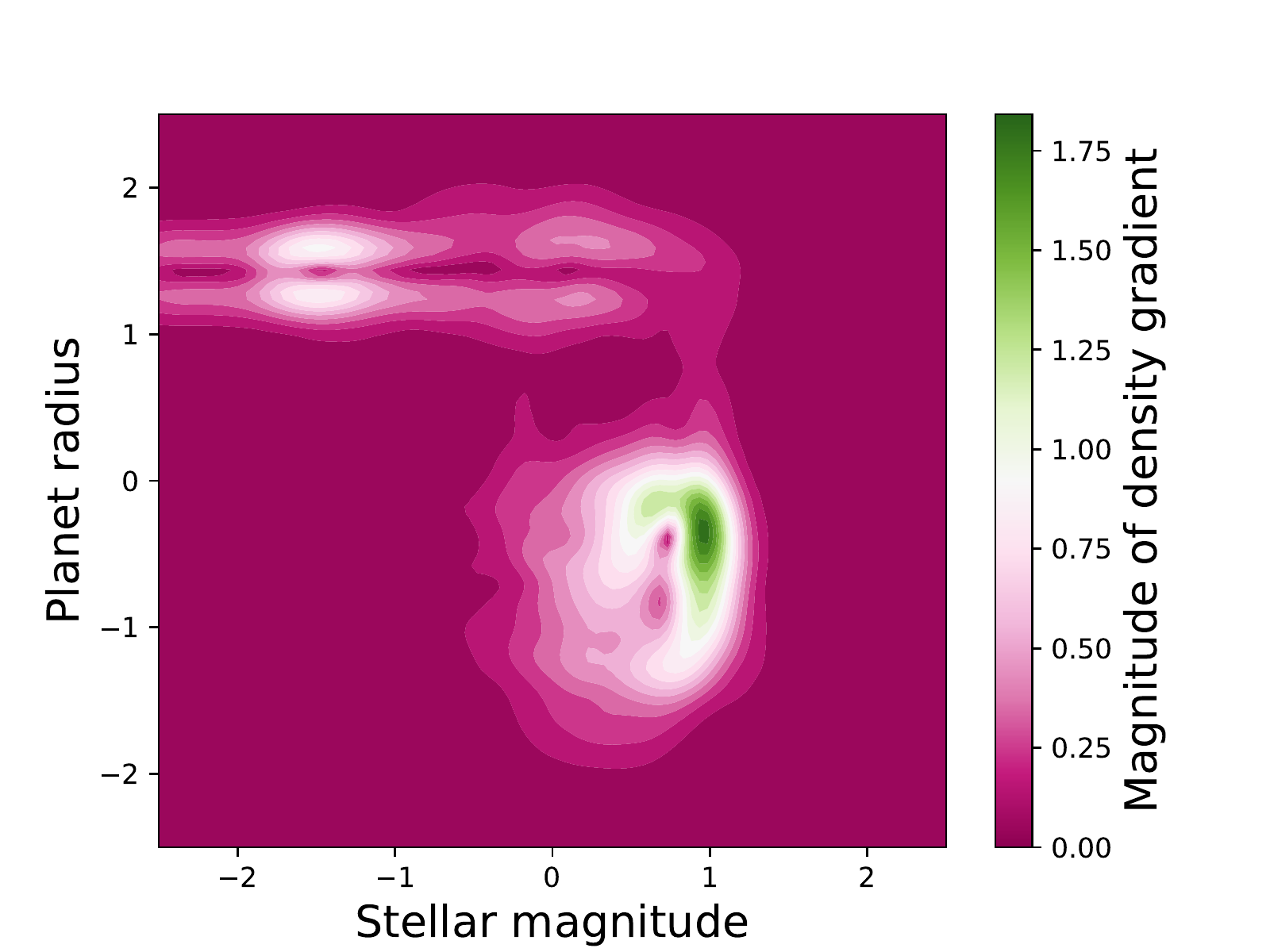}
  \vspace{-19pt}
  \caption{Grid structure observed in the PDF of the NASA exoplanet dataset (standardized log of factors). }
  \label{fig:exo}
\end{wrapfigure}

We have motivated independent discontinuities as a theoretical requirement to be able to identify quantized factors even after they passed through highly flexible diffeomorphic maps. In real data under finite samples, we can only hope for a smooothed density estimate that will never show true discontinuities, but merely sharp changes (gradients of large magnitude) in the density. Also if one is willing to assume a slightly less flexible map, such as Lipschitz, the requirement for true discontinuities may likely be relaxed to merely sharp changes.

Still, one may wonder why and how such sharp density changes could appear in latent factors of real-world data. Here is a simple example: due to gravity, people and most objects tend to be either in a standing or lying position. One will seldom see them with a $45^{\circ}$ pitch angle irrespective of how other factors appear (e.g. background color). This results in a sharp change (discontinuity) in the PDF of the pitch angle factor, independent of the values of the other factors. This is an example of a density jump due to a physical equilibrium point, of which we can expect many variants in nature.

Empirically, we found evidence of independent sharp density changes forming a grid structure in descriptive factors of the NASA Exoplanet Archive~\citep{Akeson_2013}. Figure~\ref{fig:exo} shows the magnitude of the gradient of the density of the factors \emph{stellar magnitude} and \emph{planet radius}. Locations of high magnitude gradient show an axis-aligned grid, compatible with independent jumps in the density similar to the synthetic data from Figure~\ref{fig:main}.
We provide another evidence of axis-alignment in real motion-capture data in Appendix~\ref{app:evidence}.

\section{Experiments}
\begin{figure}[!ht]
    \centering

    \subfigure[True factors.]{
        \includegraphics[width=0.31\linewidth]{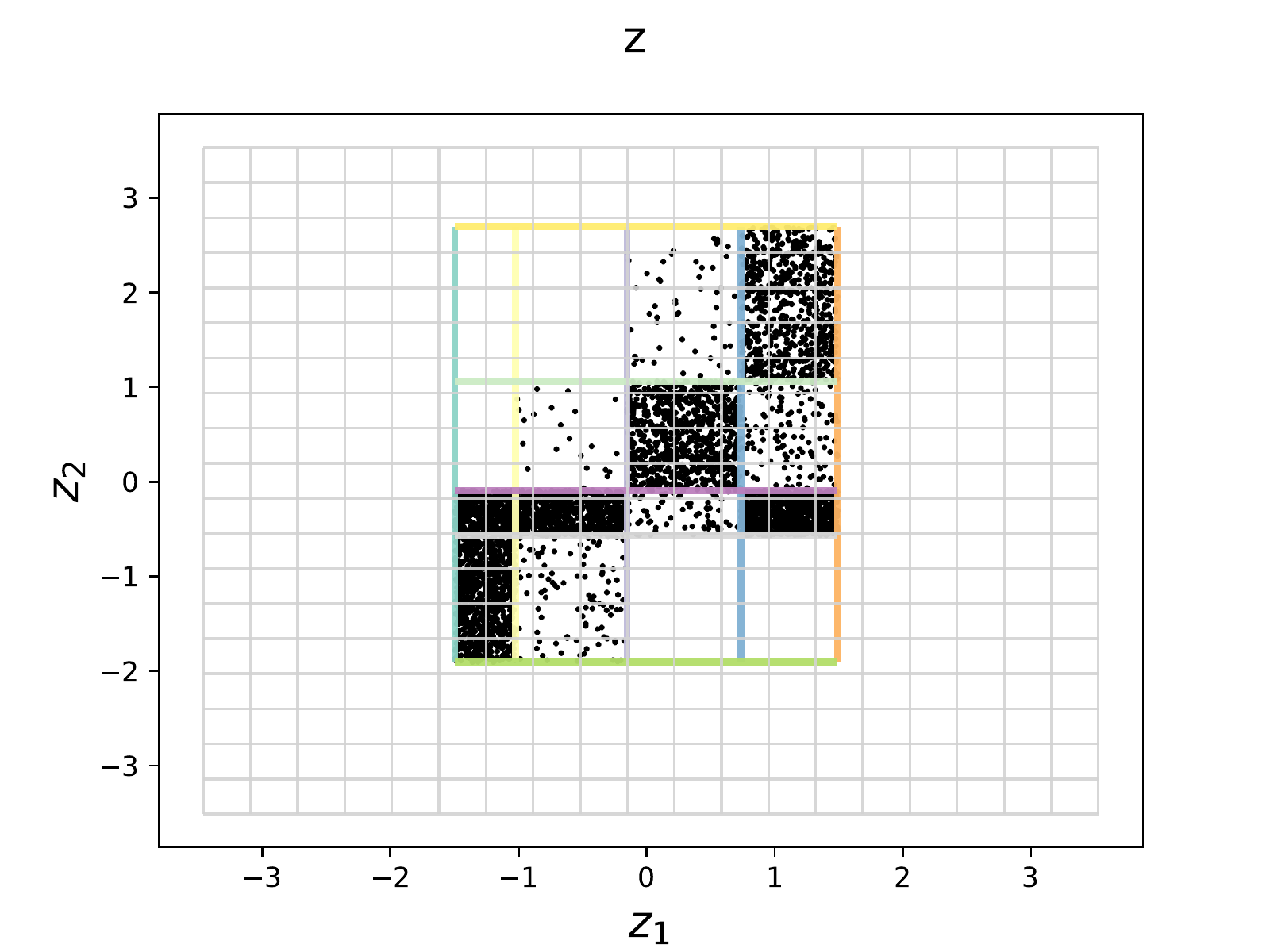}
        \label{fig:truez_unfac}
    }\hfill
    \subfigure[Observed variables.]{
        \includegraphics[width=0.31\linewidth]{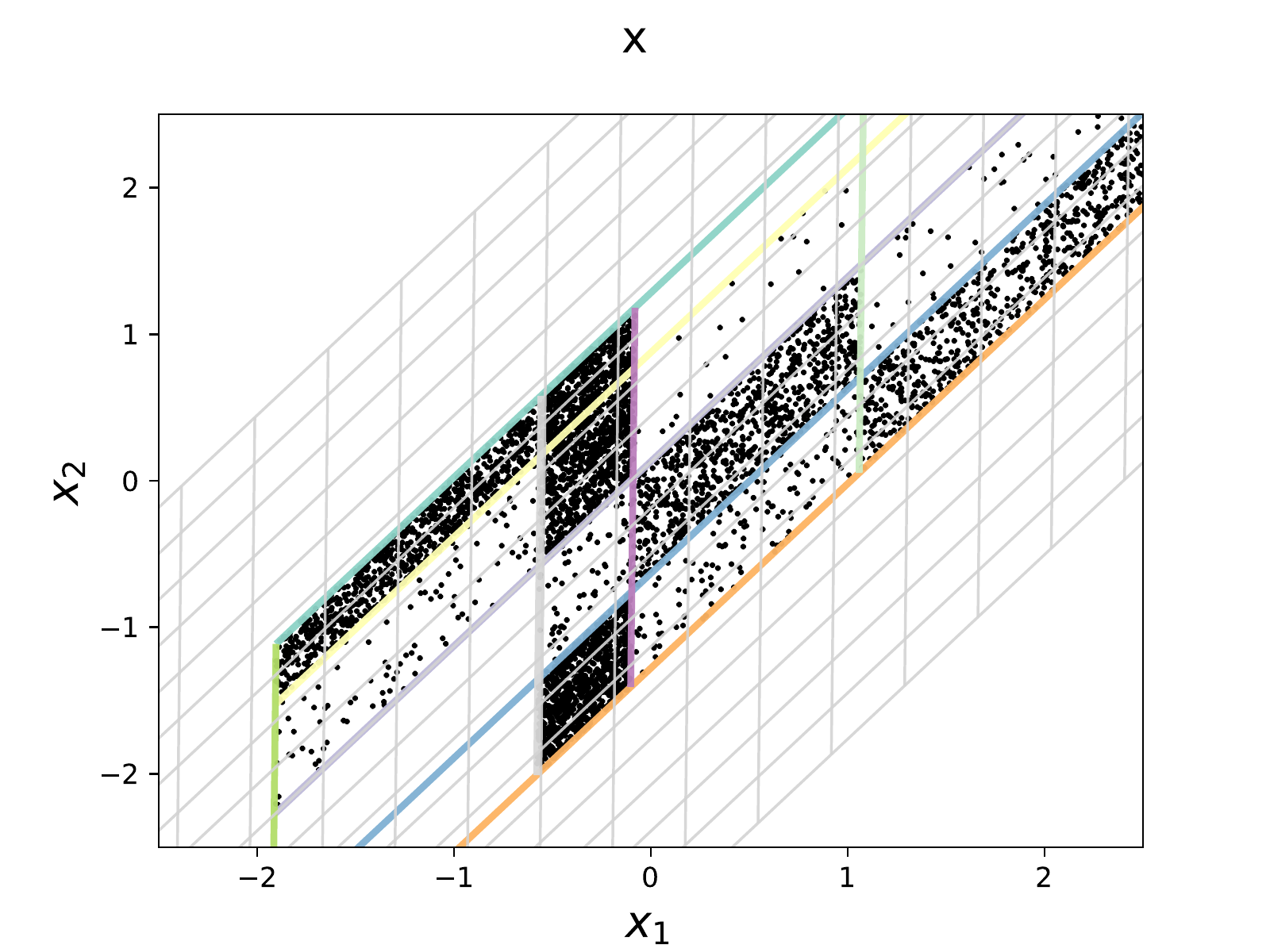}
        \label{fig:x_unfac}
    } \hfill
    \subfigure[Linear ICA reconstruction of the factors.]{
        \includegraphics[width=0.32\linewidth]{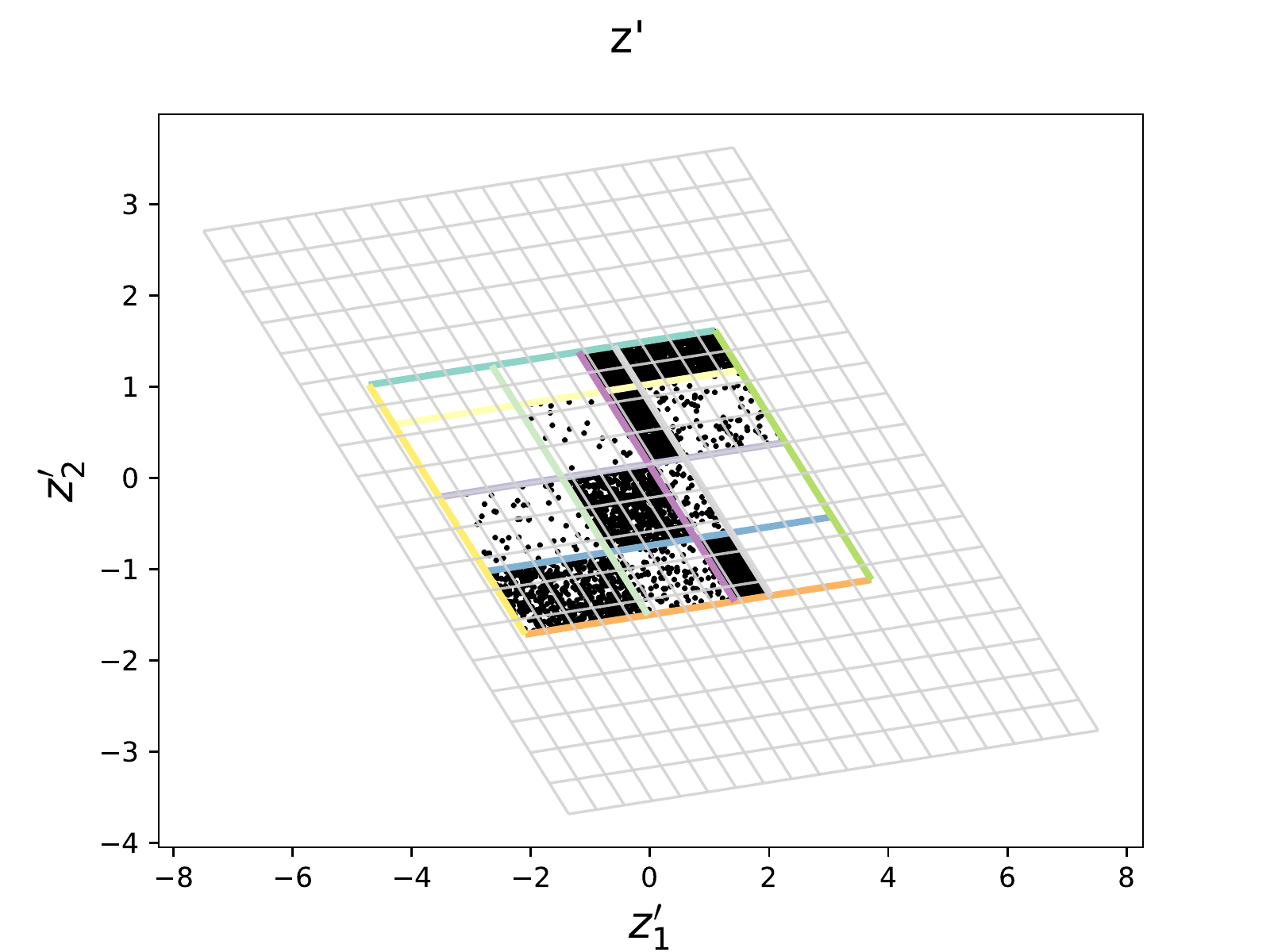}
        \label{fig:ica_unfac}
    } \\  
    \adjustbox{valign=c}{
        \subfigure[Reconstruction of the factors by Hausdorff Factorized Support.]{
            \includegraphics[width=0.31\linewidth]{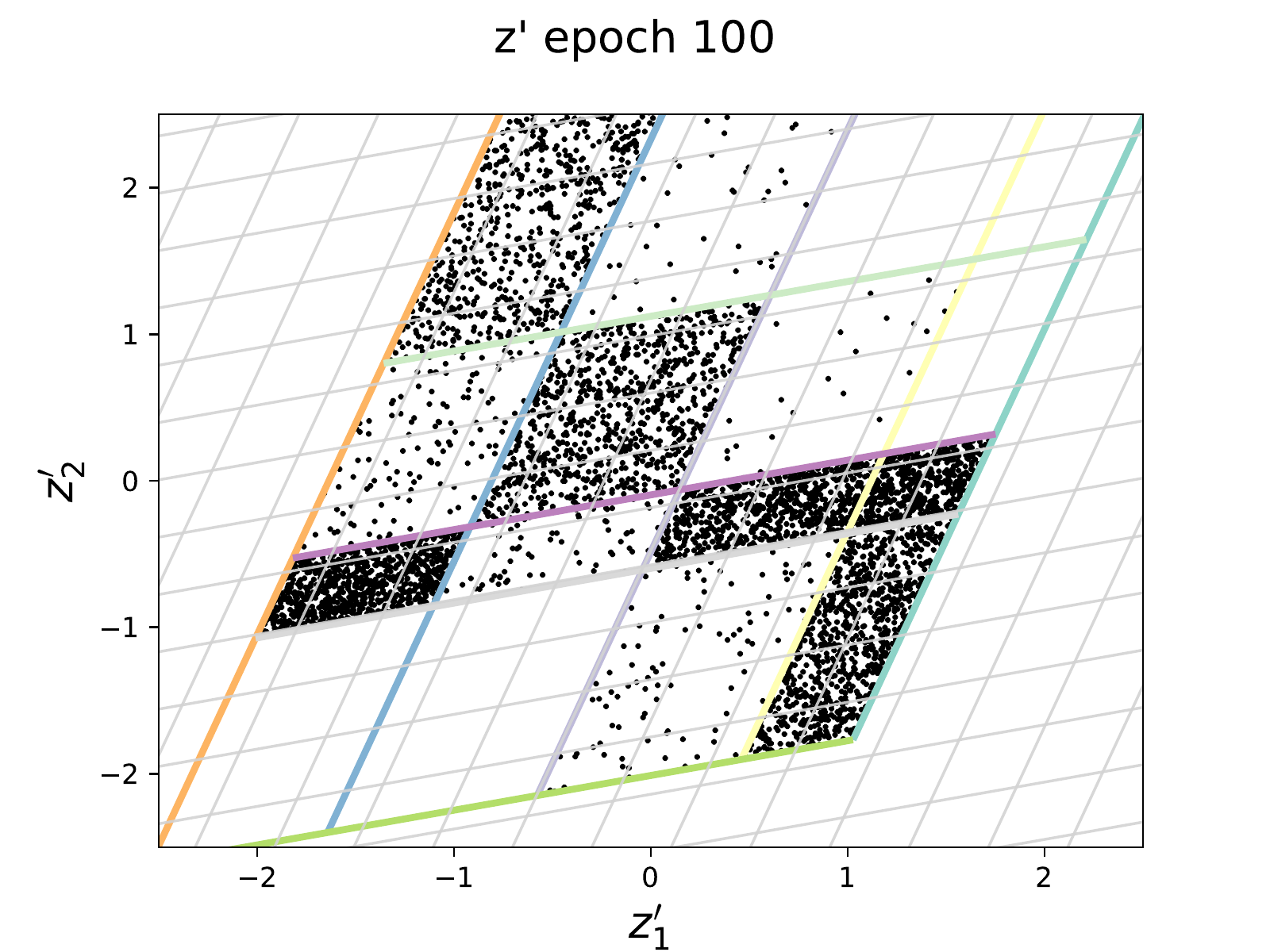}
            \label{fig:fact_corr_unfac}
        }\hfill \hspace{1cm}
        \subfigure[Our model's reconstruction of the factors.]{
            \includegraphics[width=0.31\linewidth]{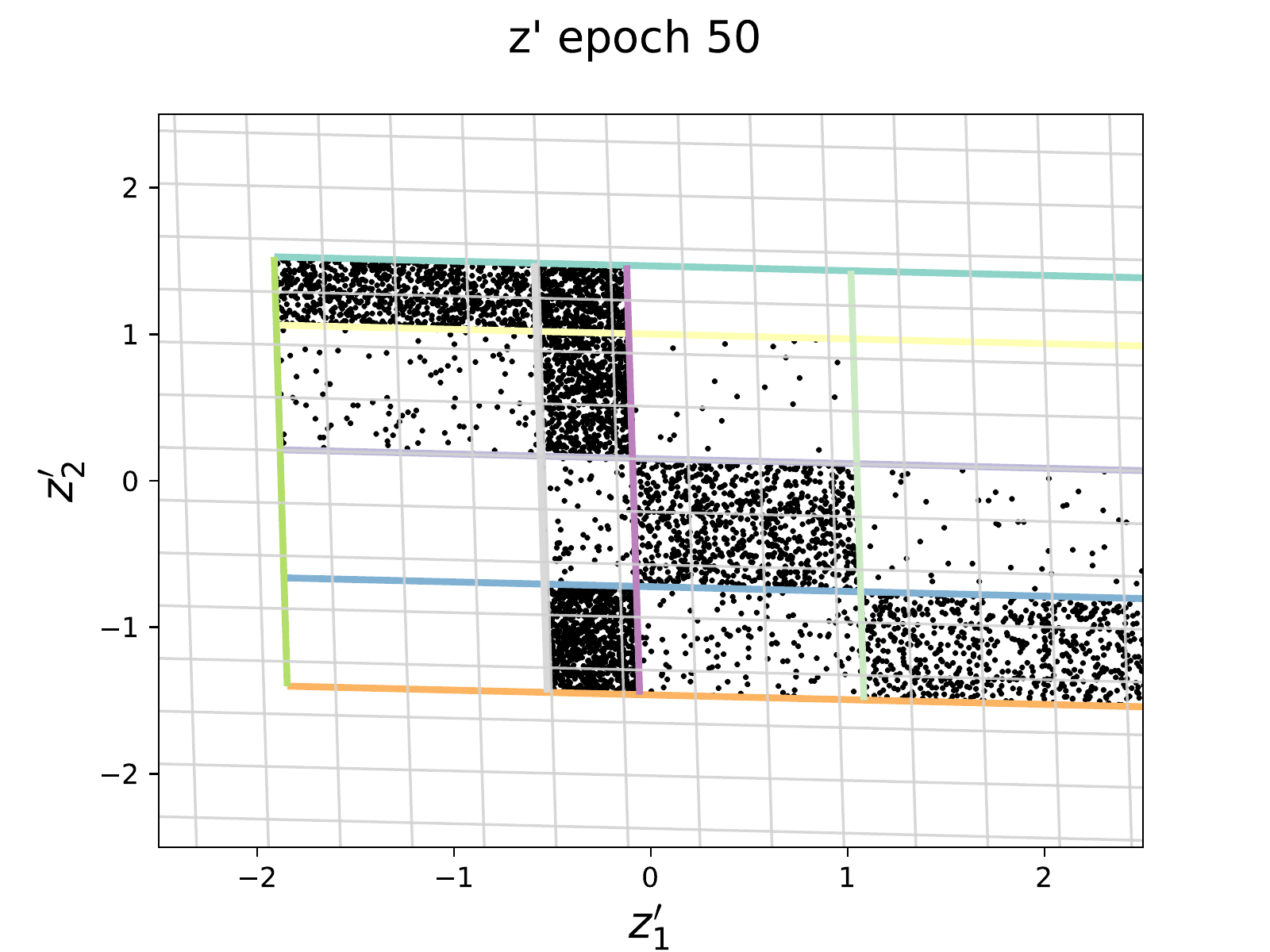} 
            \label{fig:gridalign-z_unfac}
        }
    }
    \caption{The true latent factors (a) \textbf{do not have factorized support} and are \textbf{correlated}. The observations (b) are the result of a \textbf{linear} map applied to the factors. Our method (c) obtains a factorized representation corresponding to the ground-truth factors. Both linear ICA (d) and Hausdorff Factorized Support (e) \cite{roth2022disentanglement} fail to learn the axis-aligned true latent factors.}
    \label{fig:unfactorized}
\end{figure}

We develop a criterion for learning axis-aligned discontinuities and present a proof-of-concept experiment for the case where the mixing map $f$ is linear. This is a simple prototype to demonstrate the feasibility of quantized identification based on independent discontinuities in the PDF, and how it can be advantageous compared to other methods already in the linear case. We present a tentative criterion for nonlinear transformations in Appendix \ref{app:criterion_nonlinear}, but it remains to be thoroughly tested and experimented. We reserve the proposal and analysis of a full practical criterion for future work.

The method we present here aims to align the gradients of the joint density of the factors with the axes.
We perform a density estimation $\hat{p}_{\sigma}$ of $Z$ and use it to obtain the gradients $\frac{\partial \log \hat{p}_{\sigma}}{\partial z}$. Gradients of high magnitude hint at potential discontinuities. We encourage the alignment of these gradient vectors with the standard basis vectors (axes) by maximizing their cosine similarity. The algorithm and experimental setup is detailed in Appendix~\ref{app:experiments}.

Figure~\ref{fig:unfactorized} presents the visualization of the reconstruction of the latent variables for our model, compared to Linear ICA (using the FastICA algorithm \citep{fastica}) and to Hausdorff Factorized Support (HFS) \citep{roth2022disentanglement}, for the case where the ground-truth latent factors are neither independent nor have a factorized support. The results show that the true latent structure is well-reconstructed by our model. The learned factor grid is axis-aligned and corresponds to the original grid up to permutation and axis reversal, as anticipated by the theory. The quantized cells are correctly identified up to this indeterminacy. Meanwhile, both FastICA and Hausdorff Factorized Support learn the factors up to a rotation and shearing (besides permutation and scaling), because their reconstruction is not axis-aligned. Appendix~\ref{app:experiments} presents the results in the case where the support is factorized and we remark that our model is again able to axis-align the factors, while even HFS' reconstruction does not present factorized support \footnote{The code for reproducing these results is available at \\ \href{https://github.com/facebookresearch/quantized_identifiability/}{\texttt{https://github.com/facebookresearch/quantized\_identifiability/}}.}.

\section{Conclusion and future work}

In this theoretical work, we have introduced the novel paradigm of \emph{quantized factor identifiability}. We have then shown that  \emph{fully unsupervised identifiability} of quantized factors is possible under \emph{diffeomorphisms}.
This is significant given that the prevailing literature is dominated by impossibility results. 
We are able to achieve the identification of quantized factors, provided that we assume independent discontinuities in the latent factor's distribution (which naturally form a grid). The novel relaxed (weaker) form of identifiability is meant as a step towards more realistic assumptions for disentanglement: no restrictive inductive bias on the mapping and no assumed independence of factors, rather aiming for potential causal footprints \citep{lopezpaz-2017}.

However, there are important limitations to this theory, the most obvious being that it requires actual \emph{discontinuities}. This is required due to the flexibility of general diffeomorphisms, as justified in Appendix~\ref{app:necessity_landmarks}. Future work shall try to relax this to just sharp (but not infinitely sharp) changes in the density, under slightly less general Lipschitz smooth mappings. 

When moving to the finite sample setting, we must resort to density estimation, which yields a smoothed estimate of $p_X$, and as a result, discontinuities will become non-infinite sharp changes. These ``softer'' discontinuities can still be detected by considering the magnitude of the gradient of the density (as we display in Figure~\ref{fig:exo}). The development of an effective practical training criterion and algorithm to train a nonlinear reverse mapping $g$ to recover an axis-aligned grid is left for future work (Appendix~\ref{app:criterion_nonlinear} proposes a possible starting direction). 


\clearpage
\acks{
The authors thank Léon Bottou for sharing his original motivation for discontinuities in the density from a causal perspective, as well as David Lopez-Paz for related discussions on causal footprints.
The present work only barely mentions these motivations due to its identifiability focus, but we are grateful for Léon's and David's encouragement in exploring this direction.
The authors thank Diane Bouchacourt for discussions on Theorem 1 and the experimental validation of the theory on real-world datasets, as well as Mark Ibrahim for his contribution to early discussions while this project was taking shape. We also thank Sébastien Lachapelle for feedback on this project, and Mohammad Pezeshki for feedback on the paper.

Vitória Barin-Pacela is partially supported by the Canada CIFAR AI Chair Program, as well as a grant from Samsung Electronics Co., Ldt.,  administered by Mila, in support of her PhD studies at the University of Montreal.
Simon Lacoste-Julien and Pascal Vincent are CIFAR Associate Fellows in the Learning in Machines \& Brains program.

This research has made use of the NASA Exoplanet Archive, which is operated by the California Institute of Technology, under contract with the National Aeronautics and Space Administration under the Exoplanet Exploration Program.
}

\bibliography{main}

\begin{thebibliography}{49}
\providecommand{\natexlab}[1]{#1}
\providecommand{\url}[1]{\texttt{#1}}
\expandafter\ifx\csname urlstyle\endcsname\relax
  \providecommand{\doi}[1]{doi: #1}\else
  \providecommand{\doi}{doi: \begingroup \urlstyle{rm}\Url}\fi

\bibitem[Ahuja et~al.(2022{\natexlab{a}})Ahuja, Hartford, and
  Bengio]{ahuja2022weakly}
Kartik Ahuja, Jason~S Hartford, and Yoshua Bengio.
\newblock {Weakly Supervised Representation Learning with Sparse
  Perturbations}.
\newblock In \emph{Advances in Neural Information Processing Systems},
  2022{\natexlab{a}}.

\bibitem[Ahuja et~al.(2022{\natexlab{b}})Ahuja, Mahajan, Syrgkanis, and
  Mitliagkas]{ahuja2022towards}
Kartik Ahuja, Divyat Mahajan, Vasilis Syrgkanis, and Ioannis Mitliagkas.
\newblock Towards efficient representation identification in supervised
  learning.
\newblock In \emph{1st Conference on Causal Learning and Reasoning},
  2022{\natexlab{b}}.

\bibitem[Ahuja et~al.(2022{\natexlab{c}})Ahuja, Wang, Mahajan, and
  Bengio]{ahuja2022interventional}
Kartik Ahuja, Yixin Wang, Divyat Mahajan, and Yoshua Bengio.
\newblock Interventional causal representation learning.
\newblock In \emph{40th International Conference on Machine Learning},
  2022{\natexlab{c}}.

\bibitem[Akeson et~al.(2013)Akeson, Chen, Ciardi, Crane, Good, Harbut, Jackson,
  Kane, Laity, Leifer, Lynn, McElroy, Papin, Plavchan, Ram{\'{\i} }rez, Rey,
  von Braun, Wittman, Abajian, Ali, Beichman, Beekley, Berriman, Berukoff,
  Bryden, Chan, Groom, Lau, Payne, Regelson, Saucedo, Schmitz, Stauffer, Wyatt,
  and Zhang]{Akeson_2013}
R.~L. Akeson, X.~Chen, D.~Ciardi, M.~Crane, J.~Good, M.~Harbut, E.~Jackson,
  S.~R. Kane, A.~C. Laity, S.~Leifer, M.~Lynn, D.~L. McElroy, M.~Papin,
  P.~Plavchan, S.~V. Ram{\'{\i} }rez, R.~Rey, K.~von Braun, M.~Wittman,
  M.~Abajian, B.~Ali, C.~Beichman, A.~Beekley, G.~B. Berriman, S.~Berukoff,
  G.~Bryden, B.~Chan, S.~Groom, C.~Lau, A.~N. Payne, M.~Regelson, M.~Saucedo,
  M.~Schmitz, J.~Stauffer, P.~Wyatt, and A.~Zhang.
\newblock The {NASA} exoplanet archive: Data and tools for exoplanet research.
\newblock \emph{Publications of the Astronomical Society of the Pacific},
  125\penalty0 (930):\penalty0 989--999, 2013.

\bibitem[Bengio et~al.(2013)Bengio, Courville, and
  Vincent]{bengio2013representation}
Yoshua Bengio, Aaron Courville, and Pascal Vincent.
\newblock Representation learning: A review and new perspectives.
\newblock \emph{IEEE transactions on pattern analysis and machine
  intelligence}, 35\penalty0 (8):\penalty0 1798--1828, 2013.

\bibitem[Brady et~al.(2023)Brady, Zimmermann, Sharma, Sch{\"o}lkopf, and von
  K{\"u}gelgen]{brady2023provably}
Jack Brady, Roland~S. Zimmermann, Yash Sharma, Bernhard Sch{\"o}lkopf, and
  Wieland~and von K{\"u}gelgen, Julius~Brendel.
\newblock {Provably Learning Object-Centric Representations}.
\newblock In \emph{International Conference on Machine Learning}, 2023.

\bibitem[Brehmer et~al.(2022)Brehmer, de~Haan, Lippe, and
  Cohen]{brehmer2022weakly}
Johann Brehmer, Pim de~Haan, Phillip Lippe, and Taco~S Cohen.
\newblock Weakly supervised causal representation learning.
\newblock In \emph{Advances in Neural Information Processing Systems}, 2022.

\bibitem[Buchholz et~al.(2022)Buchholz, Besserve, and
  Sch{\"o}lkopf]{buchholz2022function}
Simon Buchholz, Michel Besserve, and Bernhard Sch{\"o}lkopf.
\newblock {Function Classes for Identifiable Nonlinear Independent Component
  Analysis}.
\newblock In \emph{Conference on Neural Information Processing Systems}, 2022.

\bibitem[Capinski and Kopp(2013)]{capinski2013measure}
Marek Capinski and Peter~Ekkehard Kopp.
\newblock \emph{Measure, Integral and Probability}.
\newblock Springer Undergraduate Mathematics Series. Springer London, 2013.
\newblock ISBN 9781447106456.

\bibitem[Comon(1994)]{comon1994independent}
Pierre Comon.
\newblock Independent component analysis, a new concept?
\newblock \emph{Signal processing}, 36\penalty0 (3):\penalty0 287--314, 1994.

\bibitem[Constantinescu et~al.(2016)Constantinescu, O’Reilly, and
  Behrens]{constantinescu16}
Alexandra~O. Constantinescu, Jill~X. O’Reilly, and Timothy E.~J. Behrens.
\newblock Organizing conceptual knowledge in humans with a gridlike code.
\newblock \emph{Science}, 352\penalty0 (6292):\penalty0 1464--1468, 2016.

\bibitem[Dittadi et~al.(2021)Dittadi, Tr{\"a}uble, Locatello, Wuthrich,
  Agrawal, Winther, Bauer, and Sch{\"o}lkopf]{dittadi2021realistic}
Andrea Dittadi, Frederik Tr{\"a}uble, Francesco Locatello, Manuel Wuthrich,
  Vaibhav Agrawal, Ole Winther, Stefan Bauer, and Bernhard Sch{\"o}lkopf.
\newblock {On the Transfer of Disentangled Representations in Realistic
  Settings}.
\newblock In \emph{International Conference on Learning Representations}, 2021.

\bibitem[do~Carmo(1976)]{docarmo1976}
Manfredo~P. do~Carmo.
\newblock \emph{Differential Geometry of Curves and Surfaces}.
\newblock Prentice-Hall, 1976.

\bibitem[Friede et~al.(2023)Friede, Reimers, Stuckenschmidt, and
  Niepert]{friede2023learning}
David Friede, Christian Reimers, Heiner Stuckenschmidt, and Mathias Niepert.
\newblock Learning disentangled discrete representations.
\newblock \emph{arXiV preprint arXiv:2307.14151}, 2023.

\bibitem[Gresele et~al.(2021)Gresele, Von~K{\"u}gelgen, Stimper, Sch{\"o}lkopf,
  and Besserve]{gresele2021independent}
Luigi Gresele, Julius Von~K{\"u}gelgen, Vincent Stimper, Bernhard
  Sch{\"o}lkopf, and Michel Besserve.
\newblock Independent mechanism analysis, a new concept?
\newblock In \emph{Advances in Neural Information Processing Systems}, 2021.

\bibitem[H\"{a}lv\"{a} et~al.(2021)H\"{a}lv\"{a}, Le~Corff, Leh\'{e}ricy, So,
  Zhu, Gassiat, and Hyvarinen]{halva21}
Hermanni H\"{a}lv\"{a}, Sylvain Le~Corff, Luc Leh\'{e}ricy, Jonathan So,
  Yongjie Zhu, Elisabeth Gassiat, and Aapo Hyvarinen.
\newblock {Disentangling Identifiable Features from Noisy Data with Structured
  Nonlinear ICA}.
\newblock In \emph{Advances in Neural Information Processing Systems}, 2021.

\bibitem[Hsu et~al.(2023)Hsu, Dorrell, Whittington, Wu, and
  Finn]{hsu2023disentanglement}
Kyle Hsu, Will Dorrell, James C.~R. Whittington, Jiajun Wu, and Chelsea Finn.
\newblock {Disentanglement via Latent Quantization}.
\newblock In \emph{Neural Information Processing Systems}, 2023.

\bibitem[Hyttinen et~al.(2022)Hyttinen, Barin-Pacela, and
  Hyv{\"a}rinen]{hyttinen22}
Antti Hyttinen, Vit{\'o}ria Barin-Pacela, and Aapo Hyv{\"a}rinen.
\newblock Binary independent component analysis: a non-stationarity-based
  approach.
\newblock In \emph{38th Conference on Uncertainty in Artificial Intelligence},
  2022.

\bibitem[Hyv{\"{a}}rinen(1999)]{fastica}
Aapo Hyv{\"{a}}rinen.
\newblock Fast and robust fixed-point algorithms for independent component
  analysis.
\newblock \emph{{IEEE} Transactions on Neural Networks}, 10\penalty0
  (3):\penalty0 626--634, 1999.

\bibitem[Hyvarinen and Morioka(2016)]{tcl}
Aapo Hyvarinen and Hiroshi Morioka.
\newblock Unsupervised feature extraction by time-contrastive learning and
  nonlinear {ICA}.
\newblock In \emph{Advances in Neural Information Processing Systems},
  volume~29, 2016.

\bibitem[Hyvarinen and Morioka(2017)]{hyvarinen17a}
Aapo Hyvarinen and Hiroshi Morioka.
\newblock Nonlinear {ICA} of temporally dependent stationary sources.
\newblock In \emph{Artificial Intelligence and Statistics}, 2017.

\bibitem[Hyv{\"a}rinen and Pajunen(1999)]{hyvarinen1999nonlinear}
Aapo Hyv{\"a}rinen and Petteri Pajunen.
\newblock Nonlinear independent component analysis: Existence and uniqueness
  results.
\newblock \emph{Neural networks}, 12\penalty0 (3):\penalty0 429--439, 1999.

\bibitem[Hyv\"{a}rinen et~al.(2019)Hyv\"{a}rinen, Sasaki, and
  Turner]{hyvarinen19}
Aapo Hyv\"{a}rinen, Hiroaki Sasaki, and Richard~E. Turner.
\newblock {Nonlinear ICA Using Auxiliary Variables and Generalized Contrastive
  Learning.}
\newblock In \emph{International Conference on Artificial Intelligence and
  Statistics}, 2019.

\bibitem[Irie et~al.(2023)Irie, Csordás, and Schmidhuber]{irie2023topological}
Kazuki Irie, Róbert Csordás, and Jürgen Schmidhuber.
\newblock {Topological Neural Discrete Representation Learning \`a la Kohonen}.
\newblock In \emph{ICML 2023 Workshop: Sampling and Optimization in Discrete
  Space}, 2023.

\bibitem[Khemakhem et~al.(2020{\natexlab{a}})Khemakhem, Kingma, Monti, and
  Hyvarinen]{khemakhem2020variational}
Ilyes Khemakhem, Diederik Kingma, Ricardo Monti, and Aapo Hyvarinen.
\newblock Variational autoencoders and nonlinear {ICA}: {A} unifying framework.
\newblock In \emph{International Conference on Artificial Intelligence and
  Statistics}, 2020{\natexlab{a}}.

\bibitem[Khemakhem et~al.(2020{\natexlab{b}})Khemakhem, Monti, Kingma, and
  Hyvarinen]{khemakhem2020ice}
Ilyes Khemakhem, Ricardo Monti, Diederik Kingma, and Aapo Hyvarinen.
\newblock Ice-beem: Identifiable conditional energy-based deep models based on
  nonlinear {ICA}.
\newblock \emph{Advances in Neural Information Processing Systems},
  2020{\natexlab{b}}.

\bibitem[Kivva et~al.(2022)Kivva, Rajendran, Ravikumar, and
  Aragam]{kivva2022identifiability}
Bohdan Kivva, Goutham Rajendran, Pradeep~Kumar Ravikumar, and Bryon Aragam.
\newblock Identifiability of deep generative models without auxiliary
  information.
\newblock In \emph{Advances in Neural Information Processing Systems}, 2022.

\bibitem[Klain and Rota(1997)]{klain-affine-grassmannian-1997}
Daniel~A. Klain and Gian-Carlo Rota.
\newblock \emph{Introduction to Geometric Probability}.
\newblock Cambridge University Press, 1997.

\bibitem[Klindt et~al.(2021)Klindt, Schott, Sharma, Ustyuzhaninov, Brendel,
  Bethge, and Paiton]{klindt2020towards}
David~A. Klindt, Lukas Schott, Yash Sharma, Ivan Ustyuzhaninov, Wieland
  Brendel, Matthias Bethge, and Dylan Paiton.
\newblock Towards nonlinear disentanglement in natural data with temporal
  sparse coding.
\newblock In \emph{International Conference on Learning Representations}, 2021.

\bibitem[Lachapelle et~al.(2022)Lachapelle, Rodriguez, Sharma, Everett,
  Le~Priol, Lacoste, and Lacoste-Julien]{lachapelle2022disentanglement}
S{\'e}bastien Lachapelle, Pau Rodriguez, Yash Sharma, Katie~E Everett, R{\'e}mi
  Le~Priol, Alexandre Lacoste, and Simon Lacoste-Julien.
\newblock Disentanglement via mechanism sparsity regularization: A new
  principle for nonlinear {ICA}.
\newblock In \emph{Conference on Causal Learning and Reasoning}, 2022.

\bibitem[Lachapelle et~al.(2023)Lachapelle, Mahajan, Mitliagkas, and
  Lacoste-Julien]{lachapelle2023additive}
Sébastien Lachapelle, Divyat Mahajan, Ioannis Mitliagkas, and Simon
  Lacoste-Julien.
\newblock Additive decoders for latent variables identification and
  cartesian-product extrapolation.
\newblock In \emph{Conference on Neural Information Processing Systems}, 2023.

\bibitem[Lee(2012)]{Lee00}
John~M. Lee.
\newblock \emph{Introduction to Smooth Manifolds}, volume 218.
\newblock Springer, second edition, 2012.

\bibitem[Lim et~al.(2021)Lim, Wong, and Ye]{Lim-affine-grassmannian-2021}
Lek-Heng Lim, Ken Sze-Wai Wong, and Ke~Ye.
\newblock {The Grassmannian of affine subspaces}.
\newblock \emph{Foundations of Computational Mathematics}, 21:\penalty0
  537–--574, 2021.

\bibitem[Lippe et~al.(2022)Lippe, Magliacane, L{\"o}we, Asano, Cohen, and
  Gavves]{lippe2022citris}
Phillip Lippe, Sara Magliacane, Sindy L{\"o}we, Yuki~M Asano, Taco Cohen, and
  Stratis Gavves.
\newblock {CITRIS: Causal identifiability from temporal intervened sequences}.
\newblock In \emph{International Conference on Machine Learning}, 2022.

\bibitem[Locatello et~al.(2019)Locatello, Bauer, Lucic, Raetsch, Gelly,
  Sch{\"o}lkopf, and Bachem]{locatello2019challenging}
Francesco Locatello, Stefan Bauer, Mario Lucic, Gunnar Raetsch, Sylvain Gelly,
  Bernhard Sch{\"o}lkopf, and Olivier Bachem.
\newblock Challenging common assumptions in the unsupervised learning of
  disentangled representations.
\newblock In \emph{International Conference on Machine Learning}, 2019.

\bibitem[Locatello et~al.(2020)Locatello, Poole, R{\"a}tsch, Sch{\"o}lkopf,
  Bachem, and Tschannen]{locatello2020weakly}
Francesco Locatello, Ben Poole, Gunnar R{\"a}tsch, Bernhard Sch{\"o}lkopf,
  Olivier Bachem, and Michael Tschannen.
\newblock Weakly-supervised disentanglement without compromises.
\newblock In \emph{International Conference on Machine Learning}, 2020.

\bibitem[Lopez-Paz et~al.(2017)Lopez-Paz, Nishihara, Chintalah, Sch\"{o}lkopf,
  and Bottou]{lopezpaz-2017}
David Lopez-Paz, Robert Nishihara, Soumith Chintalah, Bernhard Sch\"{o}lkopf,
  and L\'{e}on Bottou.
\newblock {Discovering Causal Signals in Images}.
\newblock In \emph{Computer Vision and Pattern Recognition}, 2017.

\bibitem[Marsden and Hoffman(1993)]{marsden1993elementary}
Jerrold~E. Marsden and Michael~J. Hoffman.
\newblock \emph{Elementary Classical Analysis}.
\newblock W. H. Freeman, 1993.

\bibitem[Mentzer et~al.(2024)Mentzer, Minnen, Agustsson, and
  Tschannen]{mentzer2023finite}
Fabian Mentzer, David Minnen, Eirikur Agustsson, and Michael Tschannen.
\newblock {Finite Scalar Quantization: VQ-VAE Made Simple}.
\newblock In \emph{International Conference on Learning Representations}, 2024.

\bibitem[Moran et~al.(2022)Moran, Sridhar, Wang, and
  Blei]{moran2022identifiable}
Gemma~E. Moran, Dhanya Sridhar, Yixin Wang, and David~M. Blei.
\newblock {Identifiable Deep Generative Models via Sparse Decoding}.
\newblock \emph{Transactions on Machine Learning Research}, 2022.

\bibitem[Roth et~al.(2023)Roth, Ibrahim, Akata, Vincent, and
  Bouchacourt]{roth2022disentanglement}
Karsten Roth, Mark Ibrahim, Zeynep Akata, Pascal Vincent, and Diane
  Bouchacourt.
\newblock {Disentanglement of Correlated Factors via Hausdorff Factorized
  Support}.
\newblock In \emph{International Conference on Learning Representations}, 2023.

\bibitem[Taleb and Jutten(1999)]{taleb}
A.~Taleb and C.~Jutten.
\newblock Source separation in post-nonlinear mixtures.
\newblock \emph{IEEE Transactions on Signal Processing}, 47\penalty0
  (10):\penalty0 2807--2820, 1999.

\bibitem[Tr{\"a}uble et~al.(2021)Tr{\"a}uble, Creager, Kilbertus, Locatello,
  Dittadi, Goyal, Sch{\"o}lkopf, and Bauer]{traeuble2021correlation}
Frederik Tr{\"a}uble, Elliot Creager, Niki Kilbertus, Francesco Locatello,
  Andrea Dittadi, Anirudh Goyal, Bernhard Sch{\"o}lkopf, and Stefan Bauer.
\newblock On disentangled representations learned from correlated data.
\newblock In \emph{Proceedings of the 38th International Conference on Machine
  Learning}, 2021.

\bibitem[van~den Oord et~al.(2017)van~den Oord, Vinyals, and
  Kavukcuoglu]{vqvae}
Aaron van~den Oord, Oriol Vinyals, and Koray Kavukcuoglu.
\newblock Neural discrete representation learning.
\newblock In \emph{Advances in Neural Information Processing Systems}, 2017.

\bibitem[Wang and Jordan(2021)]{wang2021desiderata}
Yixin Wang and Michael~I Jordan.
\newblock Desiderata for representation learning: A causal perspective.
\newblock \emph{arXiv preprint arXiv:2109.03795}, 2021.

\bibitem[Whittington et~al.(2020)Whittington, Muller, Mark, Chen, Barry,
  Burgess, and Behrens]{Whittington770495}
James~CR Whittington, Timothy~H Muller, Shirley Mark, Guifen Chen, Caswell
  Barry, Neil Burgess, and Timothy~EJ Behrens.
\newblock The tolman-eichenbaum machine: Unifying space and relational memory
  through generalisation in the hippocampal formation.
\newblock \emph{Cell}, 183\penalty0 (5), 2020.

\bibitem[Yao et~al.(2022)Yao, Chen, and Zhang]{yao2022learning}
Weiran Yao, Guangyi Chen, and Kun Zhang.
\newblock Learning latent causal dynamics.
\newblock \emph{arXiv preprint arXiv:2202.04828}, 2022.

\bibitem[Zheng and Zhang(2023)]{zheng2023generalizing}
Yujia Zheng and Kun Zhang.
\newblock {Generalizing Nonlinear ICA Beyond Structural Sparsity}.
\newblock In \emph{37th Conference on Neural Information Processing Systems},
  2023.

\bibitem[Zheng et~al.(2022)Zheng, Ng, and Zhang]{zheng2022identifiability}
Yujia Zheng, Ignavier Ng, and Kun Zhang.
\newblock {On the Identifiability of Nonlinear ICA: Sparsity and Beyond}.
\newblock In \emph{36th Conference on Neural Information Processing Systems},
  2022.

\end{thebibliography}

\clearpage

\appendix
\section*{APPENDIX}

\section{Related work}
\label{sec:relatedwork}

We categorize the existing literature on causal representation learning into the following two categories: i) the theory imposes assumptions on both the mixing map and the latent factors, leading to typically fully unsupervised models; ii) the theory imposes assumptions on the distribution of latent factors and not strictly on the mixing map, leading to models that mostly require weak supervision or auxiliary variables.
None of these studies considered the recovery of quantized factors like we do in this work.

\textbf{Identifiability of latent factors in the unsupervised i.i.d setting:} In linear Independent Component Analysis (ICA),
\citet{comon1994independent} established that under a linear and invertible mixing map and independent non-Gaussian latent factors, these latent factors ca be identified up to order and scale indeterminacies. Beyond the linear case, \citet{taleb} analyze a post-nonlinear mapping, obtaining the same indeterminacies as in linear mixtures. \citet{gresele2021independent} demonstrated that with independent latent factors and a mixing function that adheres to the independent mechanism assumption, some of the non-identifiability counterexamples highlighted in \citet{hyvarinen1999nonlinear} can be avoided. Expanding on the role of mixing maps, \citet{buchholz2022function} scrutinized different classes of maps that restrict the Jacobian of the mixing maps. Their study specifically focused on conformal maps and orthogonal coordinate transformations. \citet{kivva2022identifiability} proposes that when the mixing map is piece-wise linear and the latent distribution is a Gaussian mixture (with latent components conditionally independent given a discrete, unobserved confounder), the true latent factors can be identified up to scaling and permutation as well. \citet{ahuja2022interventional} asserted that the true latent factors can be identified, barring permutation and scaling errors, when the mixing map is polynomial and latent factors satisfy the support independence assumption, as proposed in \citet{wang2021desiderata,roth2022disentanglement}.
\citet{brady2023provably, lachapelle2023additive} obtain identifiability for additive decoders, while \citet{moran2022identifiable, zheng2022identifiability, zheng2023generalizing} 
obtain identifiability by assuming a sparse structure on the Jacobian of the mixing function.

\textbf{Identifiability of latent factors with weak supervision:}
Research in this category largely makes assumptions on the latent distribution but imposes few constraints on the mixing map. To compensate for this lack of restrictions, these studies necessitate additional information, typically in one of two forms: a) identification driven by auxiliary information (e.g., labels, time stamps), or b) identification driven by weak supervision (e.g., data augmentations) \citep{hyvarinen17a, hyvarinen19, tcl}.  A key example of auxiliary information-driven identification is the work on identifiable variational autoencoders \citep{khemakhem2020variational}, which assumes the existence of an additionally observed variable such that the latent variables are conditionally independent given it, and the conditional probability density of the latent variable given this auxiliary variable comes from an exponential family.
This work has been expanded upon in several subsequent studies \citep{ khemakhem2020ice, lachapelle2022disentanglement, ahuja2022towards, hyttinen22}, which modify some of its assumptions.  \citet{locatello2020weakly, klindt2020towards} assumed access to paired data, which can emerge from data augmentation or natural video frames with sparse changes, resulting in supervision-driven identification. Several follow-up studies \citep{halva21, ahuja2022weakly, brehmer2022weakly,  yao2022learning, lippe2022citris} have built upon this work, moving beyond the independence assumptions on the latent factors and incorporating general transition dynamics.

\section{Practical criterion}

\subsection{Proof-of-concept experimentation for linear maps}
\label{app:experiments}
\begin{figure}
    \centering

    \subfigure[True factors.]{
        \includegraphics[width=0.47\linewidth]{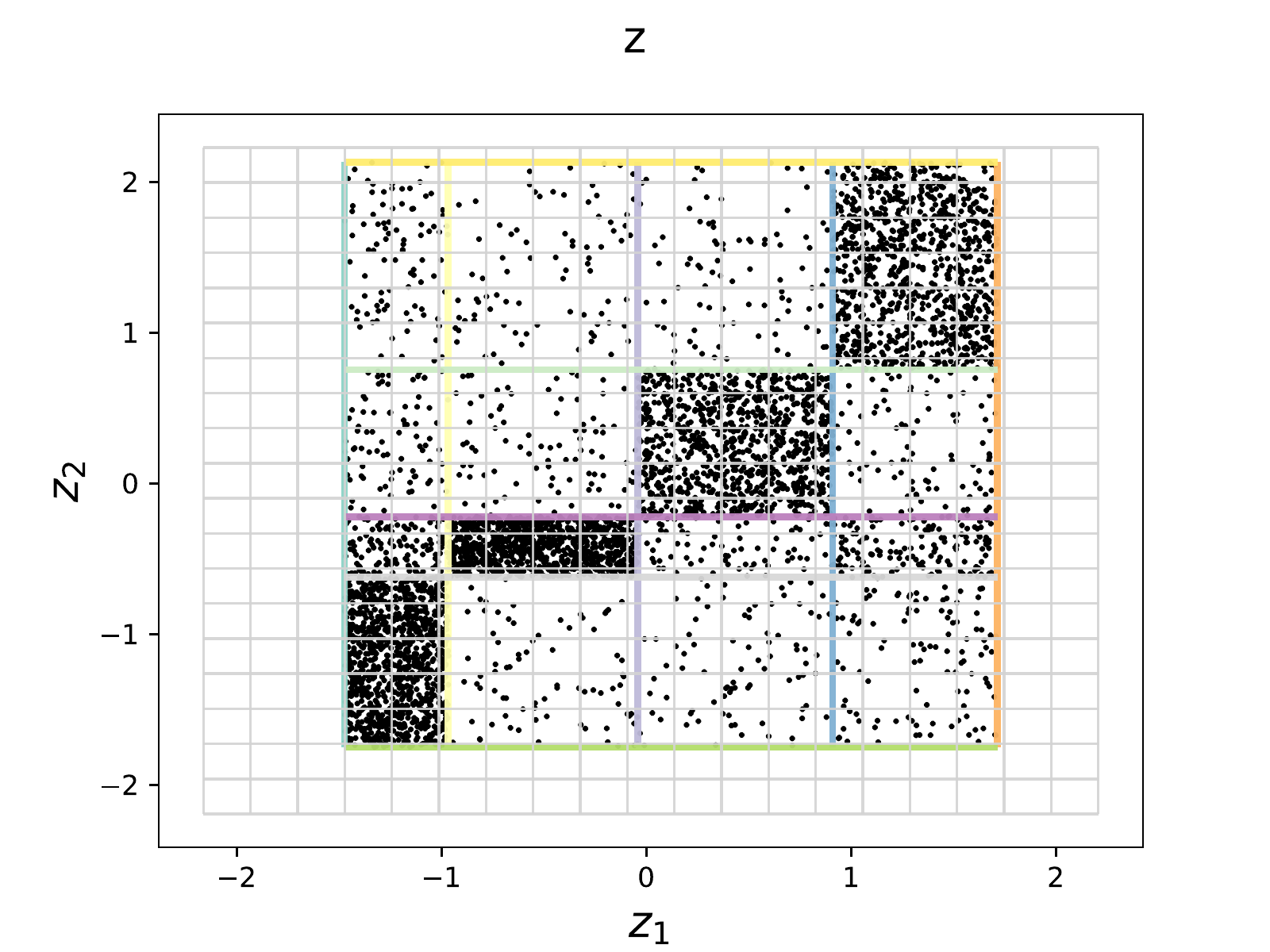}
        \label{fig:truez}
    }\hfill
    \subfigure[Observed variables.]{
        \includegraphics[width=0.47\linewidth]{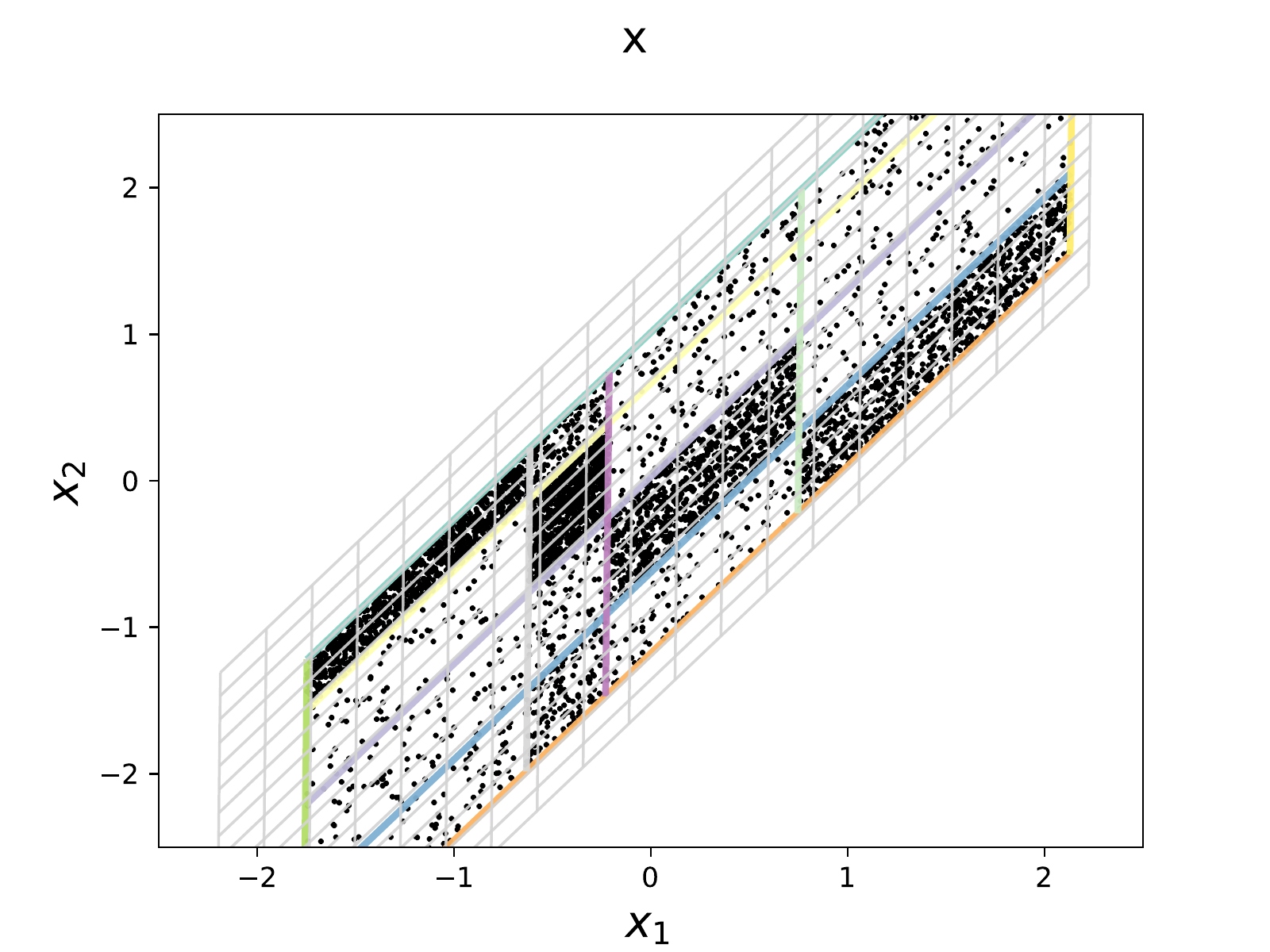}
        \label{fig:xx}
    }
    
    \subfigure[Linear ICA reconstruction of the factors.]{
        \includegraphics[width=0.47\linewidth]{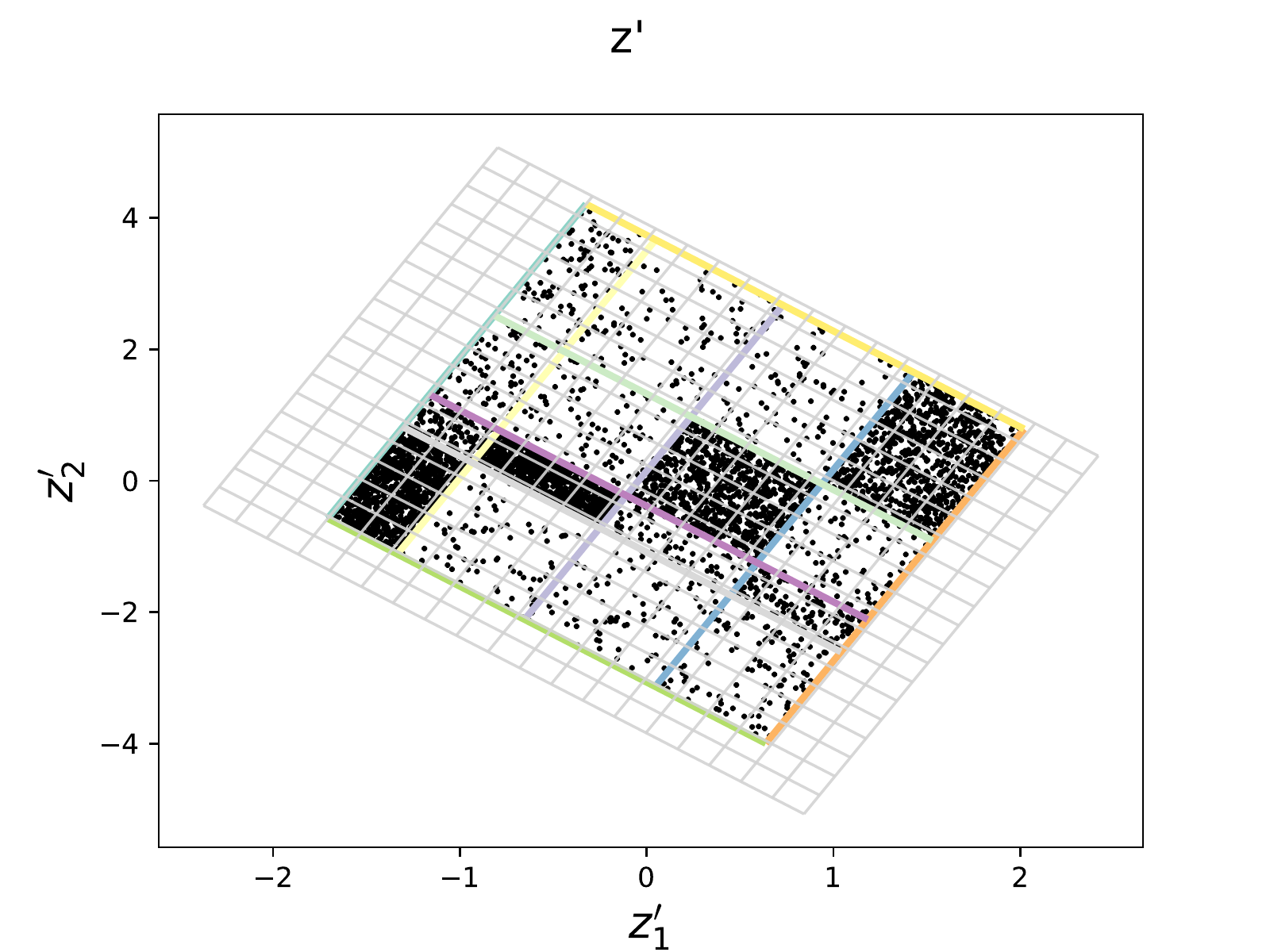}
        \label{fig:z-ica}
    }\hfill
    \subfigure[Reconstruction of the factors by Hausdorff Factorized Support.]{
        \includegraphics[width=0.47\linewidth]{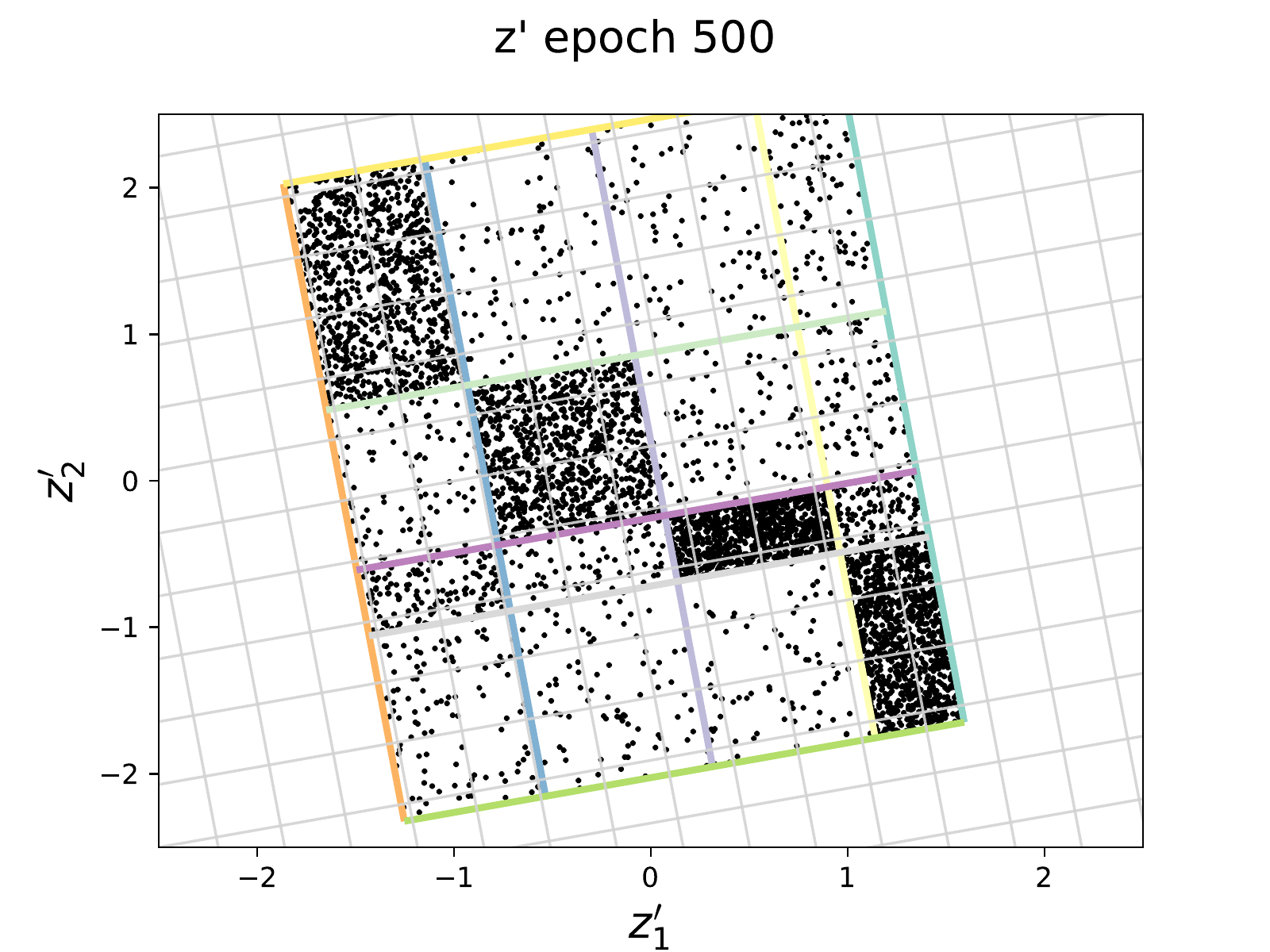}
        \label{fig:fact_corr}
    }\hfill
    \subfigure[Our model's reconstruction of the factors.]{
        \includegraphics[width=0.47\linewidth]{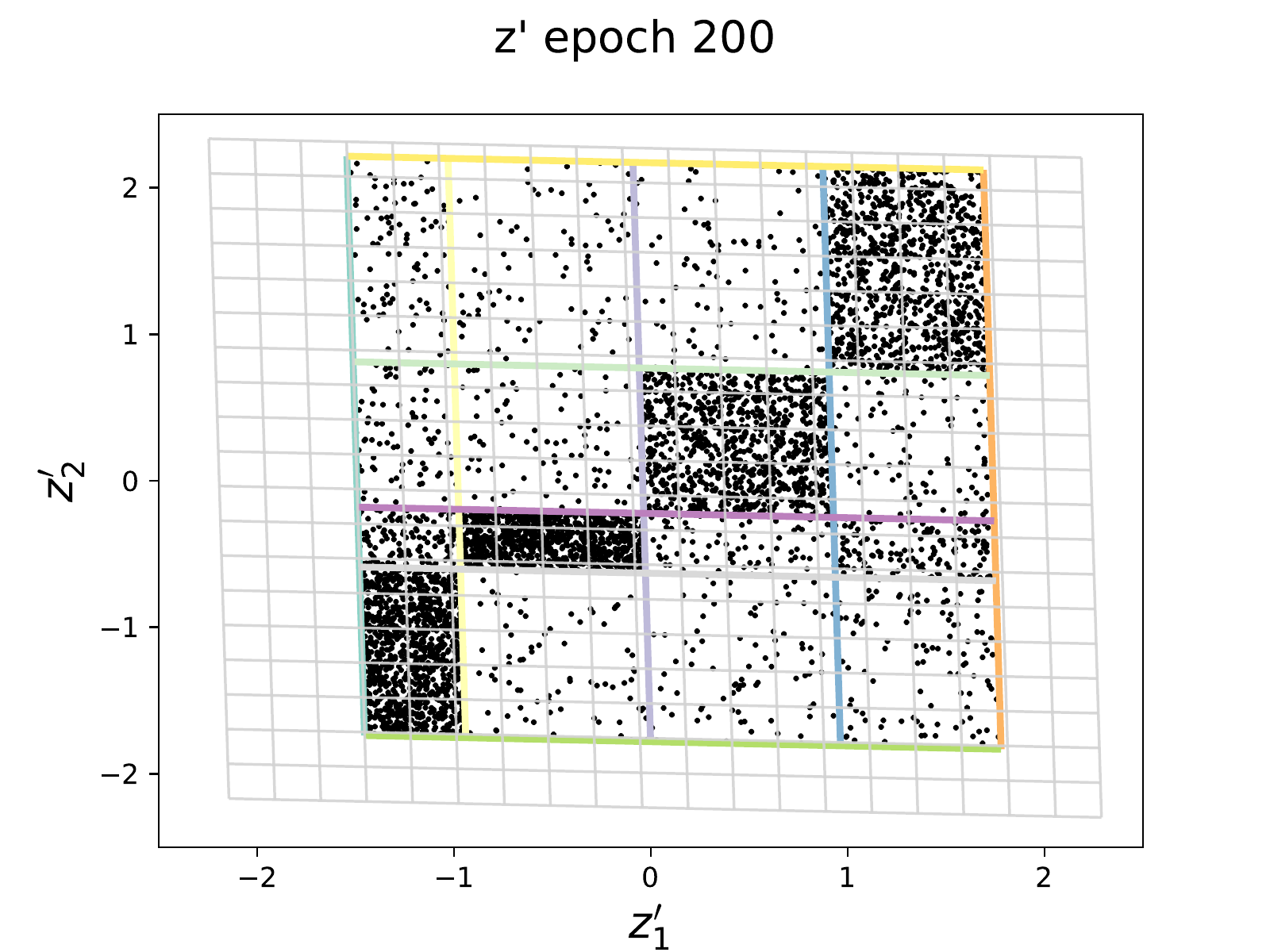}
        \label{fig:gridalign-z}
    } 
    \caption{When the true latent factors (\ref{fig:truez}) are correlated, our method (\ref{fig:gridalign-z}) obtains a factorized representation corresponding to the ground-truth factors, as opposed to linear ICA (\ref{fig:z-ica}) and Hausdorff Factorized Support (\ref{fig:fact_corr}) which reconstruct the factors up to a rotation.}
    \label{fig:ica}
    
\end{figure}

\textbf{Synthetic data generation:}

We generate a grid of points by establishing a prior for each cell, such that the sum of the priors of all the cells equals 1. We define a $4\times4$ grid, the position of each separator being drawn uniformly inside the range of the grid. In order to generate correlated data, first we draw the prior probabilities from a standard Uniform distribution. Then, we redefine the prior probability of the cells in the diagonal to be higher than the probability of the other cells, followed by normalization. 
The dataset is composed of 50,000 samples from this distribution.
In the dataset with unfactorized support (Figure \ref{fig:unfactorized}), the correlation coefficient between the true factors of variation is of 0.61.

\textbf{Algorithm:}

The steps for implementing the training criterion are:

\begin{enumerate}
\item Randomly initialize a parametric mapping $g:\mathcal{X}\rightarrow\mathcal{Z}$
to be learned. 

\item From the matrix of observed samples $\mathbf{X}$ of size $n \times D$, compute the matrix of estimated latent variables $\mathbf{Z}=g(\mathbf{X})$ of size $n \times d$ -- where $g$ is applied separately
to each row\footnote{Note that we dropped the apostrophe ' in $\mathbf{Z}$ to lighten
notation.}.

\item Estimate the density of $\mathbf{Z}$ using a kernel density estimator (Parzen
window) $\hat{p}_{\sigma}$ and compute the gradient $\mathbf{V}_{i, \cdot}=\frac{\partial \log \hat{p}_{\sigma}}{\partial z}(\mathbf{Z}_{i, \cdot})$
at every point of $\mathbf{Z}$.

\item Define importance weighting terms $\alpha$ based on the gradient magnitudes $\alpha_{i}=\frac{\|V_{i, \cdot}\|}{\sum_{i'=1}^{n}\|V_{i', \cdot}\|}$. A large magnitude of the gradient indicates a sharp jump, hence, this weight indicates how close the sample is to a density jump, that is, how likely it belongs to an axis-separator of the grid. 
\end{enumerate}

Let $\bar{\mathbf{V}}_{i}=\frac{\mathbf{V}_{i}}{\|\mathbf{V}_{i}\|}$ be the normalized version of $\mathbf{V}_{i}$.
For the individual gradient vectors to be axis-aligned, the maximum cosine similarity with the canonical axis vectors should be maximized:

\begin{equation}
    \mathrm{maximize}\;\max_{j \in \{ 1, \dots, d \} }|\mathrm{cosim}(V_{i \cdot},\vec{\mathbf{1}}_{j})|=\max_{j}\frac{|\mathbf{V}_{ij}|}{\|\mathbf{V}_{i \cdot}\|_{2}}=\frac{\|\mathbf{V}_{i \cdot}\|_{\infty}}{\|\mathbf{V}_{i}\|_{2}}=\|\bar{\mathbf{V}}_{i \cdot}\|_{\infty}.
\end{equation}

Then, the final loss function to be optimized over all the points is:

\begin{equation}
    \mathrm{minimize}\; \ell_\mathrm{grad-axis} = -\sum_{i=1}^{n}\alpha_{i}\|\bar{\mathbf{V}}_{i, \cdot}\|_{\infty}.
\end{equation}

\textbf{Details of model and algorithm -- Dataset with unfactorized support:}

We minimize this loss using stochastic gradient descent with a learning rate of $0.1$, momentum of $0.9$, and a batch size of $5000$ samples. Mini-batches are employed due to the high memory cost of loading the full dataset.
The results displayed are for when the training loss stops decreasing.

The kernel density estimation employs a bandwidth of $0.1$.
With finite samples, we use a density estimator $\hat{p}_{Z'}$, since we do not have access to the exact $p_{Z'}$. We remark that any density estimation will result in some smoothing of the true distribution. So even if there were real discontinuities in the exact density, they will appear as smoothed discontinuities: the gradients of the density have large magnitude, not infinite magnitude.

\textbf{Hausdorff Factorized Support training details}:

We train HFS using the \texttt{hausdorff\_hard} distance approximation which is used throughout the experiments from \citet{roth2022disentanglement}.
In this simple linear case, we simply optimize to minimize the Hausdorff distance between the learned factors and their counterpart with factorized support. We did not find an advantage in using the reconstruction term as the representation does not collapse into a single point.
Training is done using stochastic gradient descent with a step size of 0.0001 and a batch size of 5000 samples. 

\textbf{Experiment with factorized support:}
We conduct a similar study on a dataset in which the support of the true factors is factorized.
In this dataset (Figure \ref{fig:ica}, the true factors have a correlation coefficient of 0.64.
In this experiment, we attempt to have a fair comparison with HFS and demonstrate that using (discontinuity) information from inside the support can help achieve better axis-alignment (and as a result, better factorized support) of the learned factors of variation.

We compare our model with linear ICA and show that our model is able to learn a factorized representation of the factors, while Fast ICA \citep{fastica} fails due to the correlation of the factors violating the independence assumption, as illustrated in Figure \ref{fig:ica}. HFS also learns the reconstructed factors up to a rotation (but no shearing), even though its factorized support assumption is satisfied, showing that in this case our criterion is effective in aligning the factors with the axes.

\subsection{Towards a criterion for nonlinear maps}
\label{app:criterion_nonlinear}

\paragraph{Alignment of discontinuities in the joint density}

For nonlinear maps, only encouraging the gradients to be axis aligned does not suffice because the distortions yield a curved latent space. It is also desirable to straighten this deformed grid, which can then be axis-aligned. Here, we outline a few terms that could encourage this behavior in the training dynamics.
We can align both the point samples and their gradient vectors.
Moreover, the alignment comes in two forms: local alignment in a neighborhood of points, and alignment to the axes.

\begin{itemize}
\item \textbf{Gradient local alignment term:} encourage pairs
of neighboring points of high gradient magnitude to have gradients
aligned by maximizing their cosine similarity. We can make the criterion
a weighted average of cosine similarities, with significant weights
only if they are neighboring points and both have large gradient magnitudes):
\begin{align*}
\\
\beta_{i,i'}= & \alpha_{i}\alpha_{i'}\exp(-\frac{1}{2\sigma_{2}^{2}}\|\mathbf{Z}_{i}-\mathbf{Z}_{i'}\|^{2})\\
\bar{\beta}_{i,i'} & =\frac{\beta_{i,i'}}{\sum_{i,i'}\beta_{i,i'}}\\
\mathrm{maximize} & \sum_{i,i'}\bar{\beta}_{i,i'}\,\mathrm{cosim}(\mathbf{V}_{i},\mathbf{V}_{i'})\\
i.e.\;\mathrm{minimize} \; \ell_\mathrm{grad-local} = & -\sum_{i,i'}\bar{\beta}_{i,i'}\left\langle \bar{\mathbf{V}}_{i},\mathbf{\bar{V}}_{i'}\right\rangle 
\end{align*}

\item \textbf{Points local axis alignment term:} encourages neighboring
points with large density gradient magnitude to lie on or close to
the same axis separator. For this, it suffices that they share one
of their coordinates. In other words, is suffices to minimize the minimum over
coordinates of the squared difference: 
\[
\mathrm{minimize}\; \ell_\mathrm{points-local} = \sum_{i,i'}\bar{\beta}_{i,i'}\min_{j}\left(\frac{\mathbf{Z}_{ij}-\mathbf{Z}_{i'j}}{\|\mathbf{Z}_{i}-\mathbf{Z}_{i'}\|}\right)^{2}
\]

\item \textbf{Points-gradient-orthogonality term:} encourages the gradient
vector to be orthogonal to the vectors joining neighboring points by penalizing their squared cosine similarity: 
\[
\mathrm{minimize}\; \ell_\mathrm{points-grad} =  \sum_{i,i'}\bar{\beta}_{i,i'}\left(\left\langle \bar{\mathbf{V}}_{i},\frac{\mathbf{Z}_{i'} - \mathbf{Z}_{i}}{\|\mathbf{Z}_{i'} - \mathbf{Z}_{i}\|}\right\rangle \right)^{2}.
\]
\end{itemize}
We can, then, define a training criterion that is a weighted sum of
these terms (with appropriate sign), possibly together with the minimization of a reconstruction
error $\ell_\mathrm{rec}$ (from a decoder network $\hat{f}$ that tries to reconstruct $\mathbf{X}$ from
$\mathbf{Z}$).
$$ \ell_\mathrm{rec}= \frac{1}{n} \sum_{i=1}^n \| \hat{f}(\mathbf{Z}_i) - \mathbf{X}_i \|^2$$

The complete loss to minimize is, thus, 
$$ 
L(\theta) = \lambda_1 \ell_\mathrm{grad-local} 
+ \lambda_2 \ell_\mathrm{grad-axis}
+ \lambda_3 \ell_\mathrm{points-local}
+ \lambda_4 \ell_\mathrm{points-grad}
+ \lambda_5 \ell_\mathrm{rec}
$$
where $\theta$ is the set of (network) parameters of both encoder $g$ and decoder $\hat{f}$.

\section{Additional evidence of axis-aligned discontinuities in real data}
\label{app:evidence}

\begin{figure}[H]
    \centering
    \includegraphics[scale=0.4]{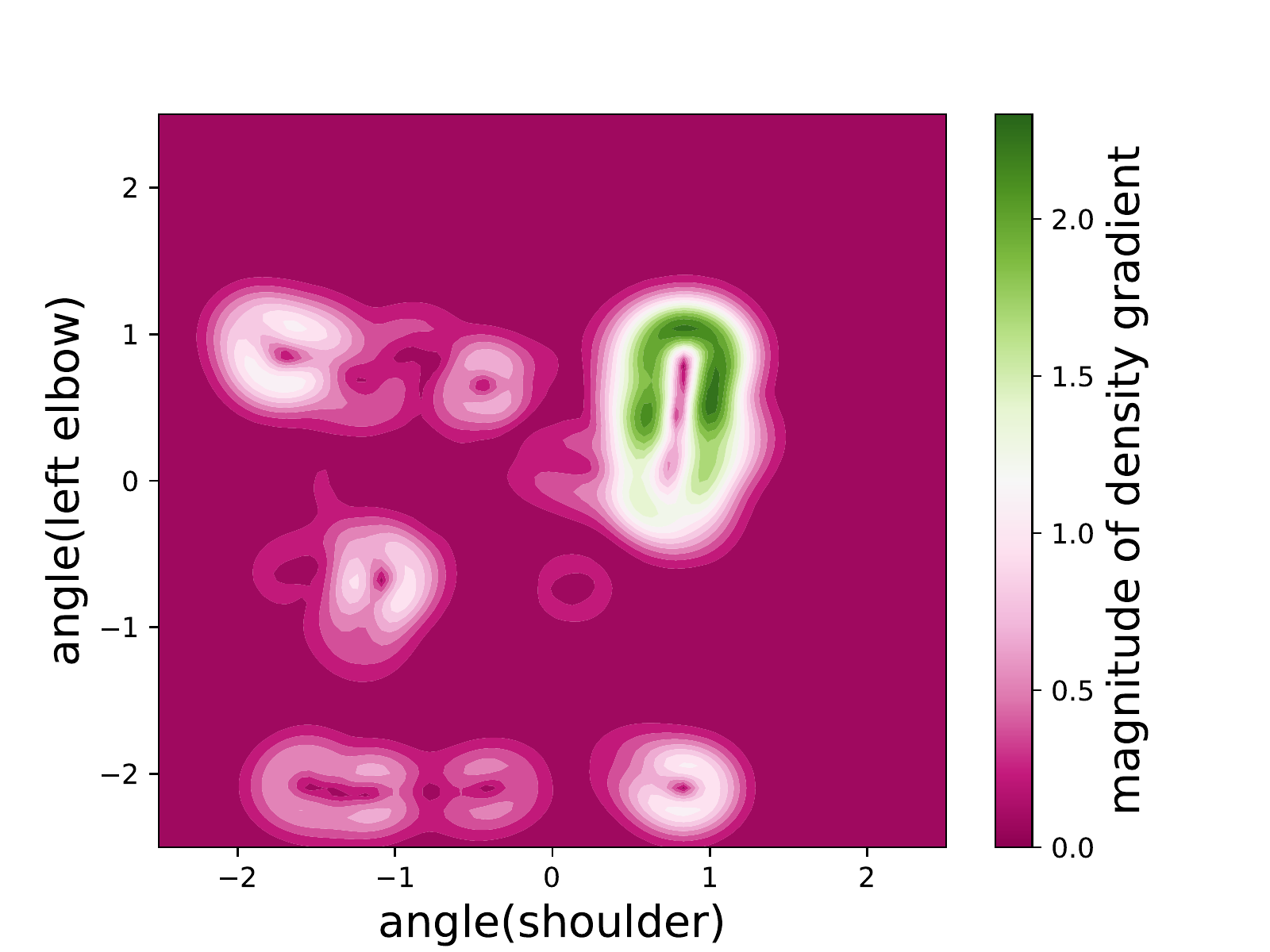}
    \caption{Evidence of axis-aligned discontinuities found in the CMU Motion Capture Dataset.}
    \label{fig:mocap}
\end{figure}

Figure \ref{fig:mocap} presents additional evidence of axis-aligned discontinuities in real data from the CMU Motion Capture dataset (obtained from \texttt{mocap.cs.cmu.edu}). We preprocess the variables to obtain the angles of the joints of the body in different frames captured. In particular, the angle of the left elbow is the angle formed by the markers “LELB”, “LUPA”, and “LWRA”. The variable angle(shoulders) measures the angle of the shoulders (defined by the markers “RSHO” and “LSHO”) with respect to the vertical axis. Then these angles are standardized. In the plot, we can observe axis-aligned discontinuities in green, which represents a high magnitude of the gradient of the density.

\section{Illustrations of the definitions}
\label{app:illustrations_def}

\begin{figure}
    \centering

    \subfigure[Independent non-Gaussian factors]{
        \includegraphics[width=0.3\linewidth]{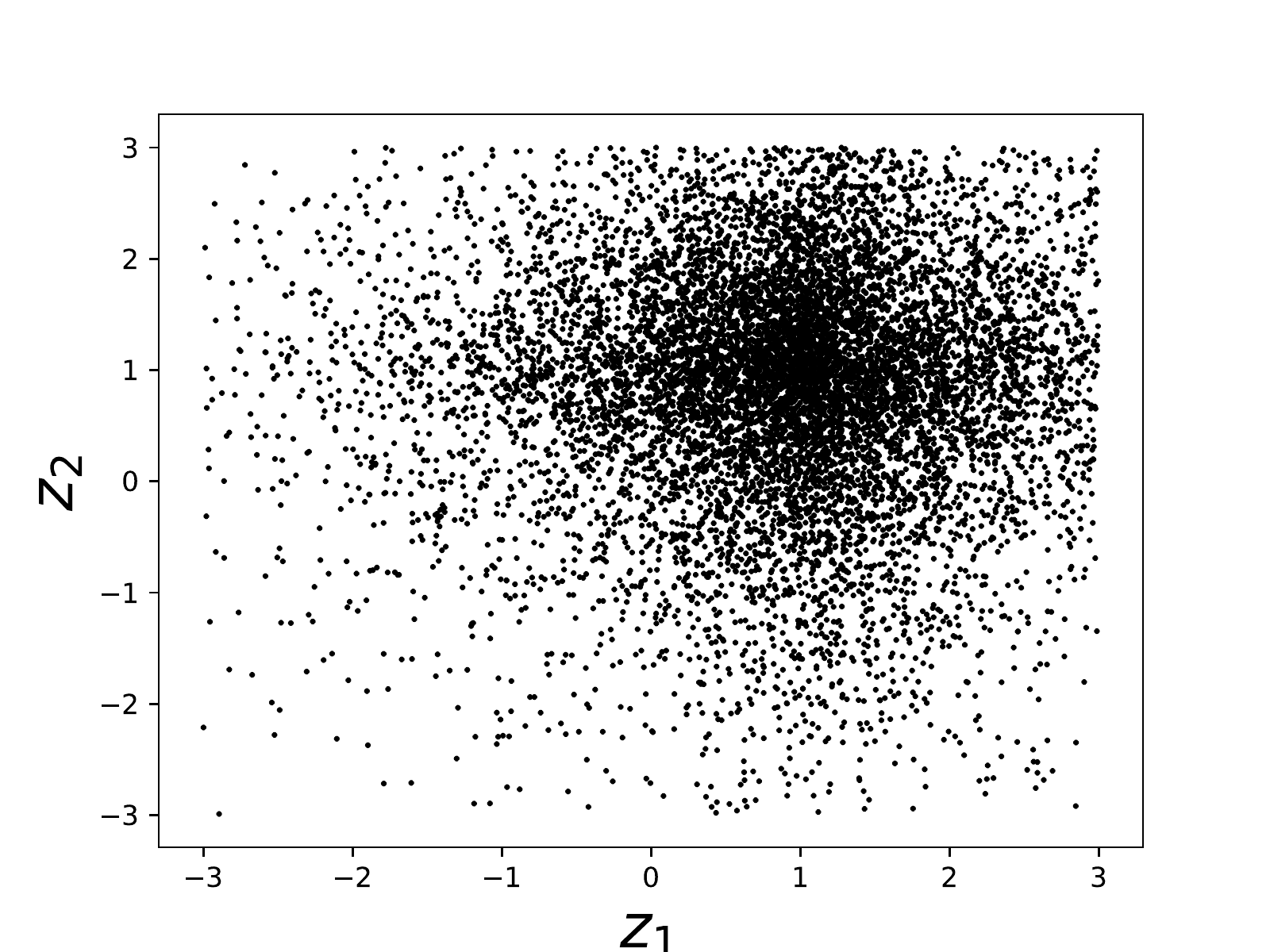}%
        \label{subfig:indep_z}%
    }\hfill
    \subfigure[Latent factors whose PDF has axis-aligned discontinuities]{
        \includegraphics[width=0.3\linewidth]{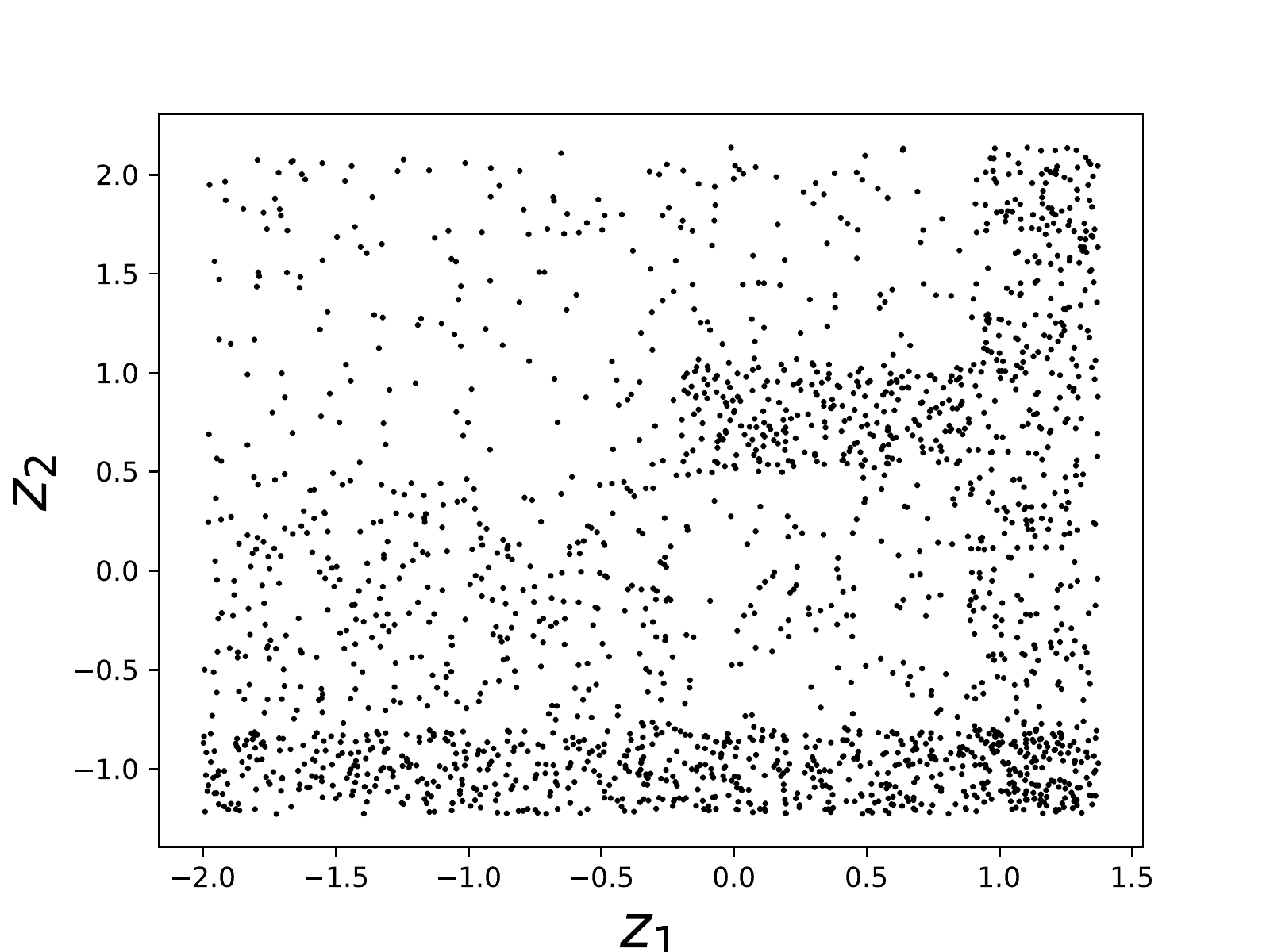}%
        \label{subfig:grid-scatter}%
    }\hfill
    \subfigure[Underlying grid structure]{
        \includegraphics[width=0.3\linewidth]{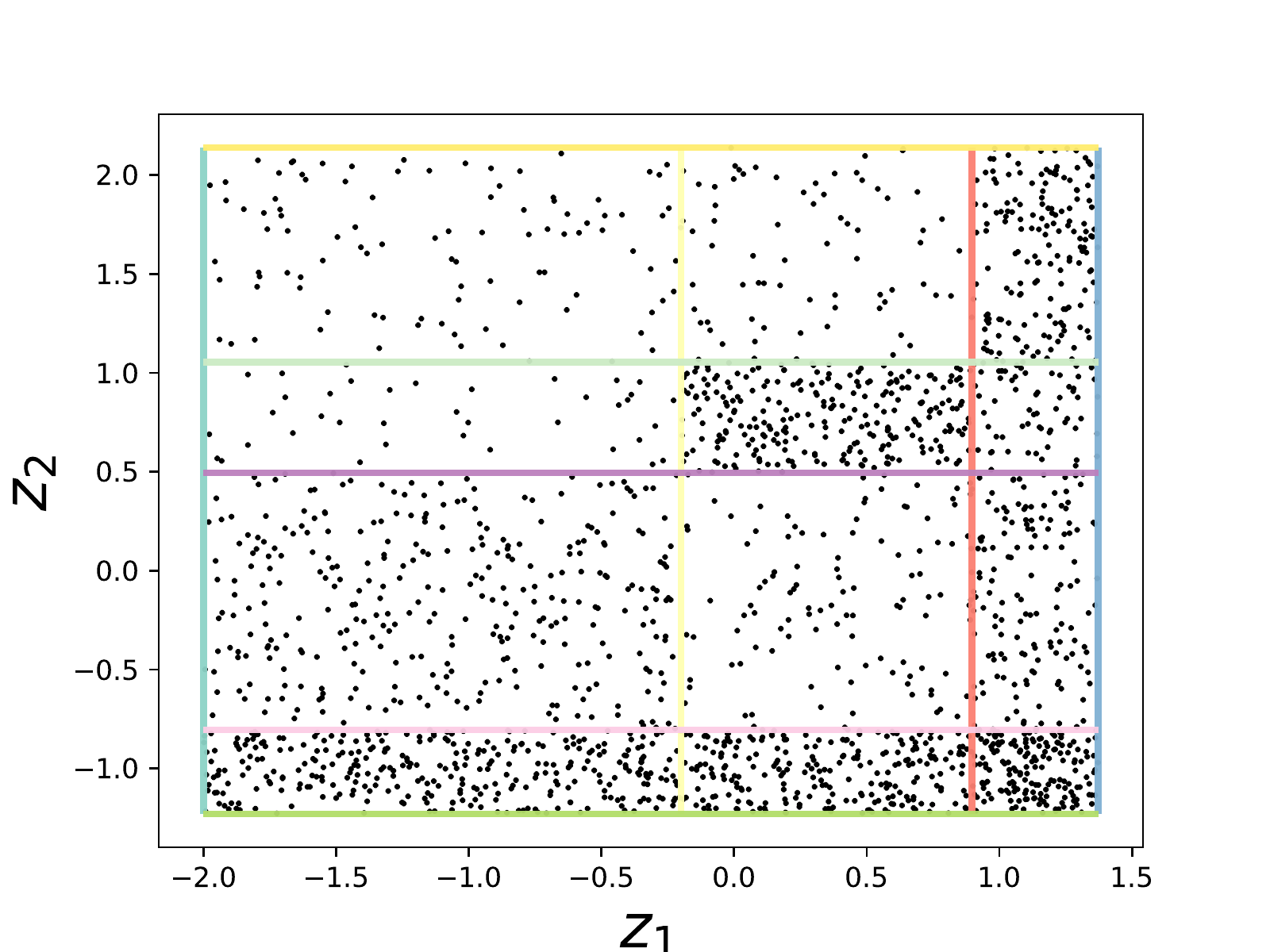}%
        \label{subfig:grid-sep}%
    }

    \caption{Illustration of different kinds of assumptions on the distribution of latent factors. \textbf{Left:} samples from traditional assumption of independent non-Gaussian factors (here using a truncated Laplace distribution). \textbf{Middle:} samples from a distribution that follows our assumption of axis-aligned discontinuities in the probability density. \textbf{Right:} Underlying grid structure revealing discontinuities in the density landscape as colored \emph{axis-separators}, forming a \emph{grid}.
    Traditional independence assumption yields non-identifiability result under a general nonlinear smooth mapping (diffeomorphism). Our assumption yields, under a diffeomorphism, provable recovery of a discretized coordinate system. It allows to map back observed points into the proper latent grid cell -- a novel relaxed form of identifiability, which we term \emph{quantized factor identifiability}.
    }
    
    \label{fig:latents}
\end{figure}

Figures \ref{fig:setup1}, \ref{fig:setup3}, \ref{fig:setup2}, \ref{fig:holes}, and \ref{fig:backbone} illustrate the concepts used in the definitions of section \ref{sec:grid-structure-recovery}.

\begin{figure}[H]
    \centering
    \includegraphics[scale=0.2]{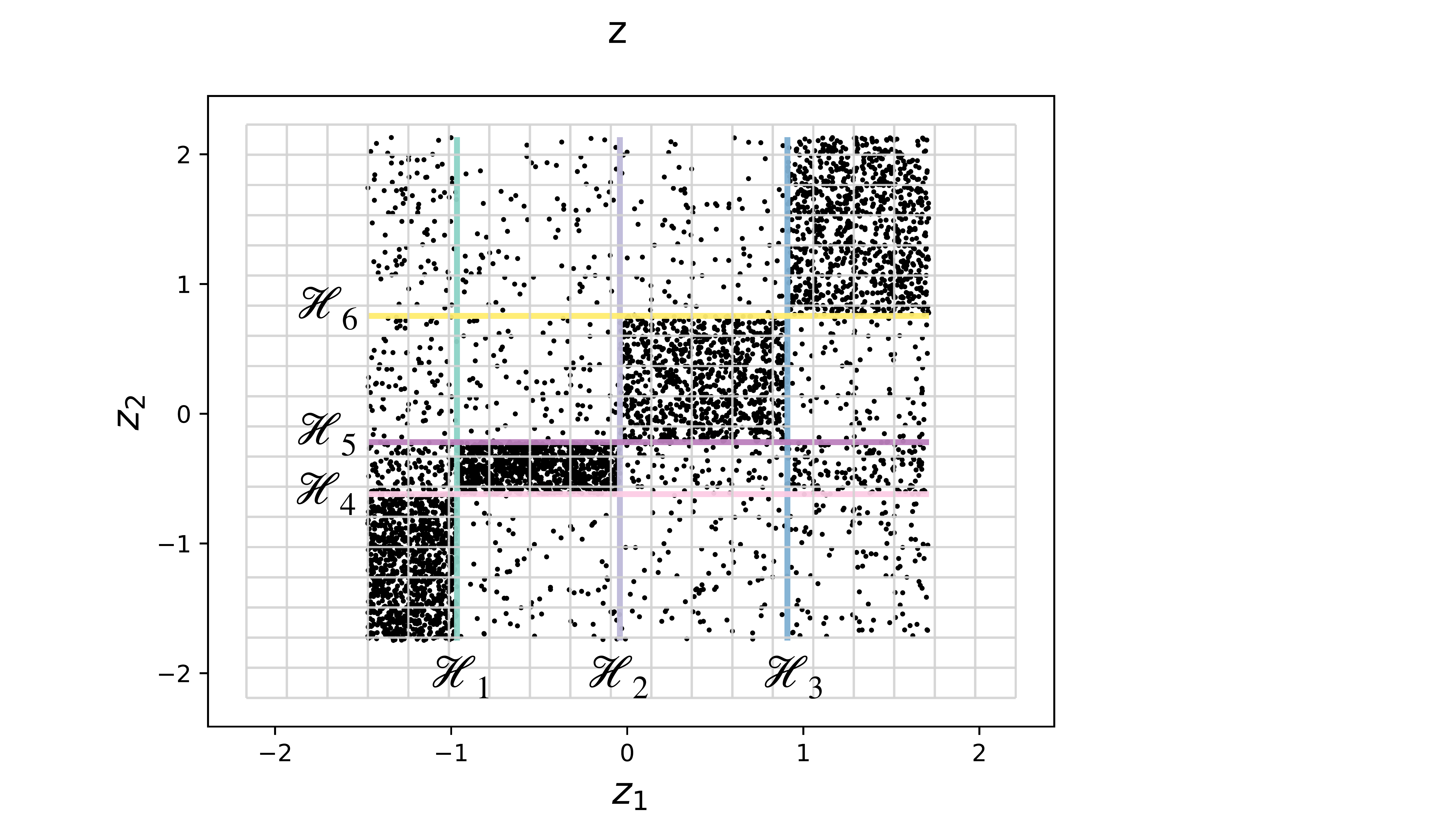}
    \caption{Axis separators $\mathcal{H}_1, \dots, \mathcal{H}_6 $ of $\mathcal{S}$ (Definition \ref{def:axis-separator}).}
    \label{fig:setup1}
\end{figure}

\begin{figure}[H]
    \centering
    \includegraphics[scale=0.2]{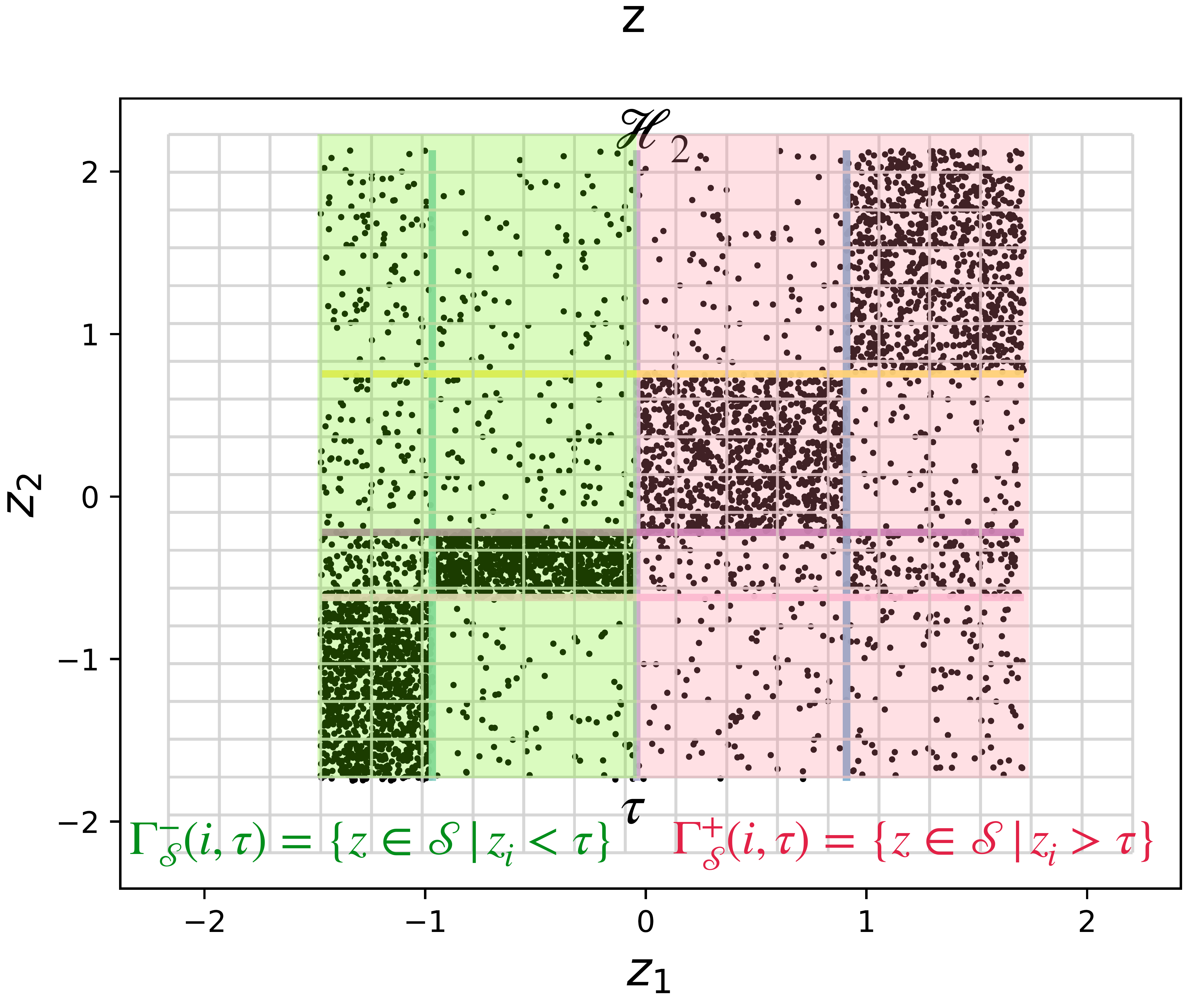}
    \caption{An\textbf{ axis-separator} of $\mathcal{S}$ splits $\mathcal{S}$ in two halves $\Gamma_{\mathcal{S}}^{+}(i,\tau)=\{z\in\mathcal{S}|z_{i}>\tau\}$ and $\Gamma_{\mathcal{S}}^{-}(i,\tau)=\{z\in\mathcal{S}|z_{i}<\tau\}$, which are each nonempty and connected (Definition \ref{def:axis-separator}). }
    \label{fig:setup3}
\end{figure}

\begin{figure}[H]
    \centering
    \includegraphics[scale=0.1]{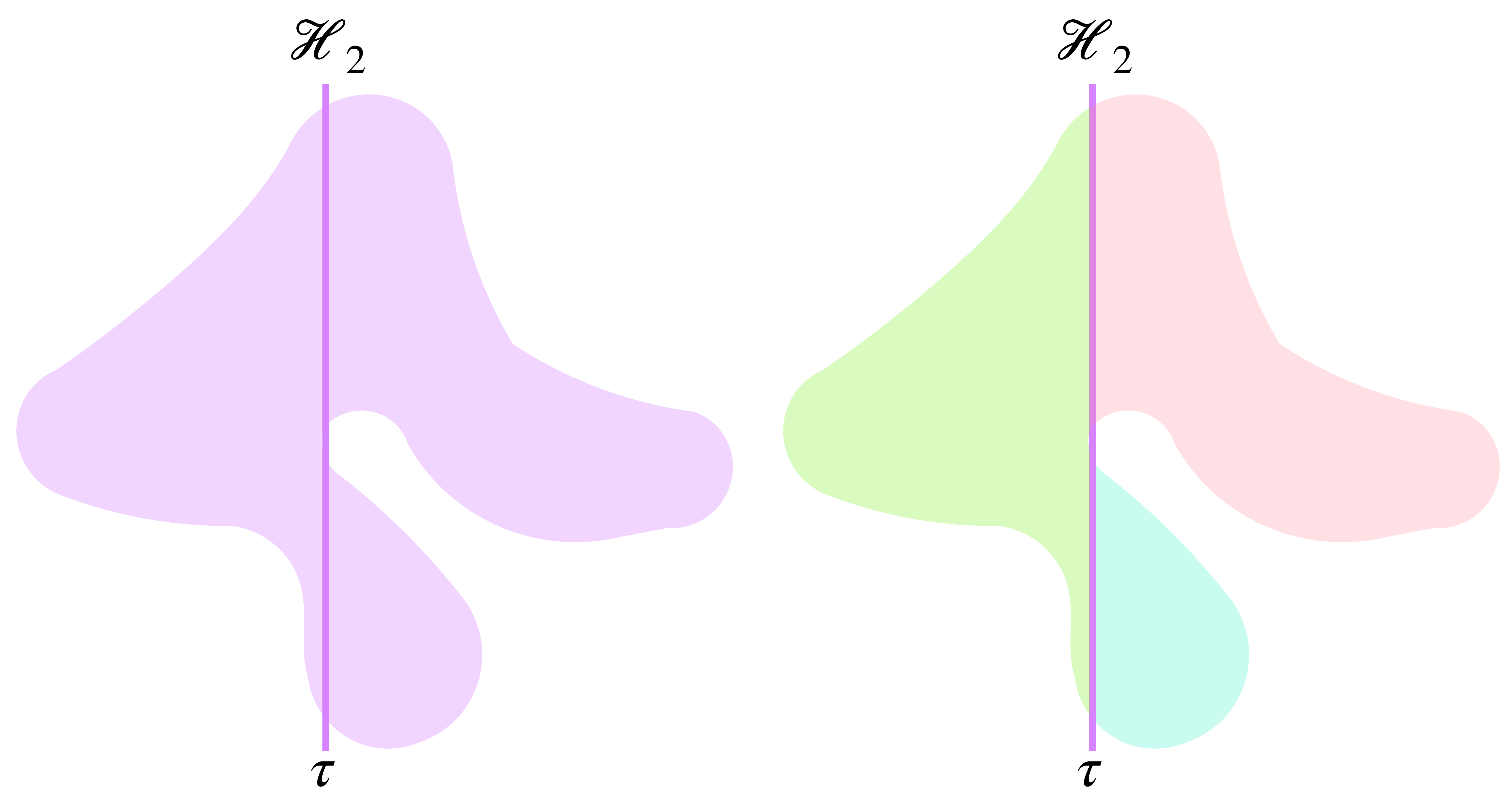}
    \caption{ Axis-separator of $\mathcal{S}$ counterexample: Even though the support of the set on the left is connected, $\mathcal{H}_2$ splits it into three parts, not two halves, so it does not satisfy the axis-separator condition (Definition \ref{def:axis-separator}). }
    \label{fig:holes}
\end{figure}

\begin{figure}[H]
    \centering
    \includegraphics[scale=0.2]{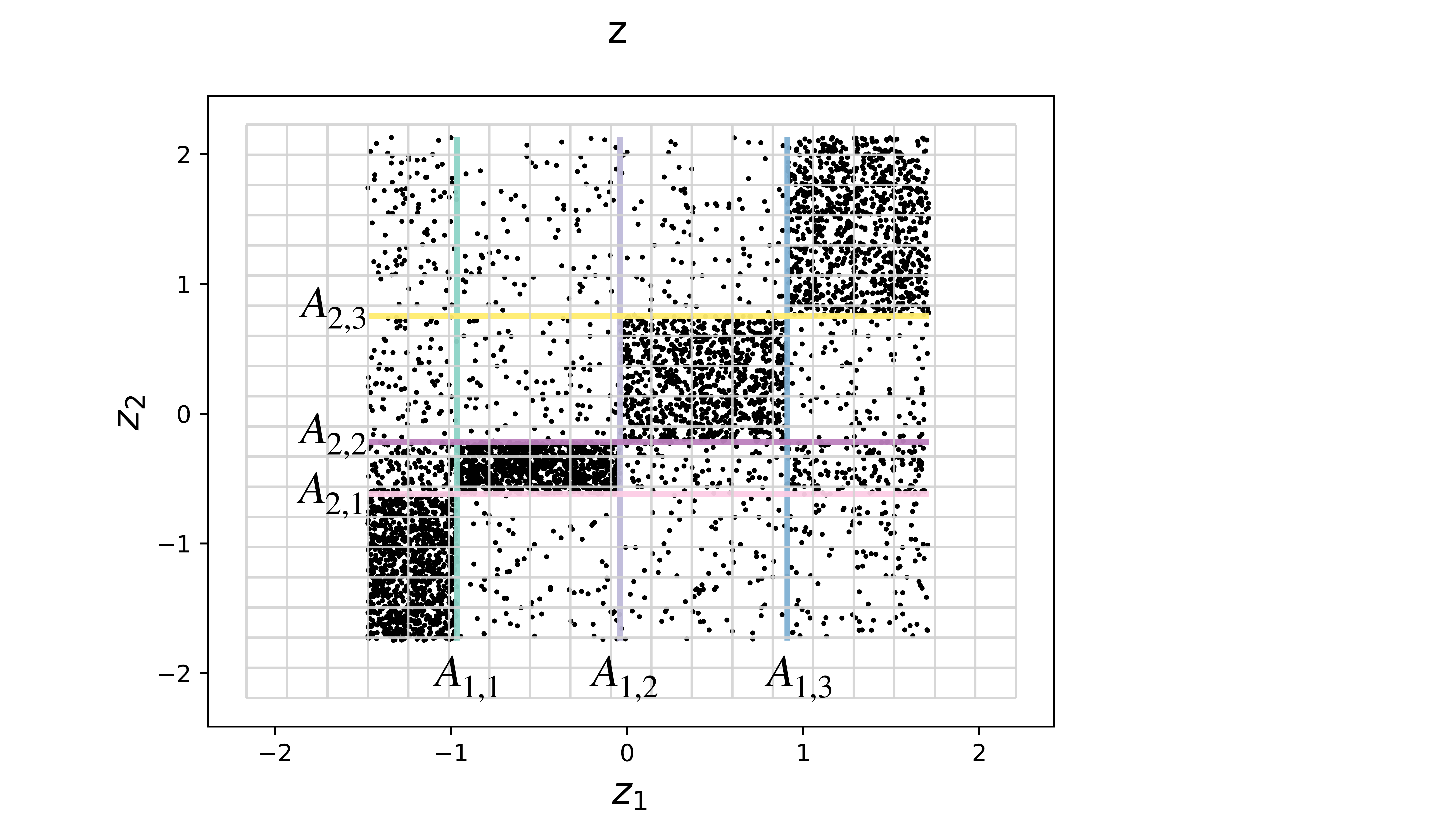}
    \caption{A \textbf{discrete coordination} $\mathbf{A}$
is a tuple $\mathbf{A}=(\mathbf{A}_{1},\ldots,\mathbf{A}_{d})$ where
each $\mathbf{A}_{i}$ is itself a tuple of real numbers in increasing
order $\mathbf{A}_{i}=(\mathbf{A}_{i,1},\ldots,\mathbf{A}_{i,n_{i}})$
such that $\mathbf{A}_{i,k+1}>\mathbf{A}_{i,k}$. These represent the coordinates
of axis-separators along each of the $d$ coordinate axes (Definition \ref{def:discrete_coordination}).}
    \label{fig:setup2}
\end{figure}

\begin{figure}[H]
    \centering
    \includegraphics[scale=0.15]{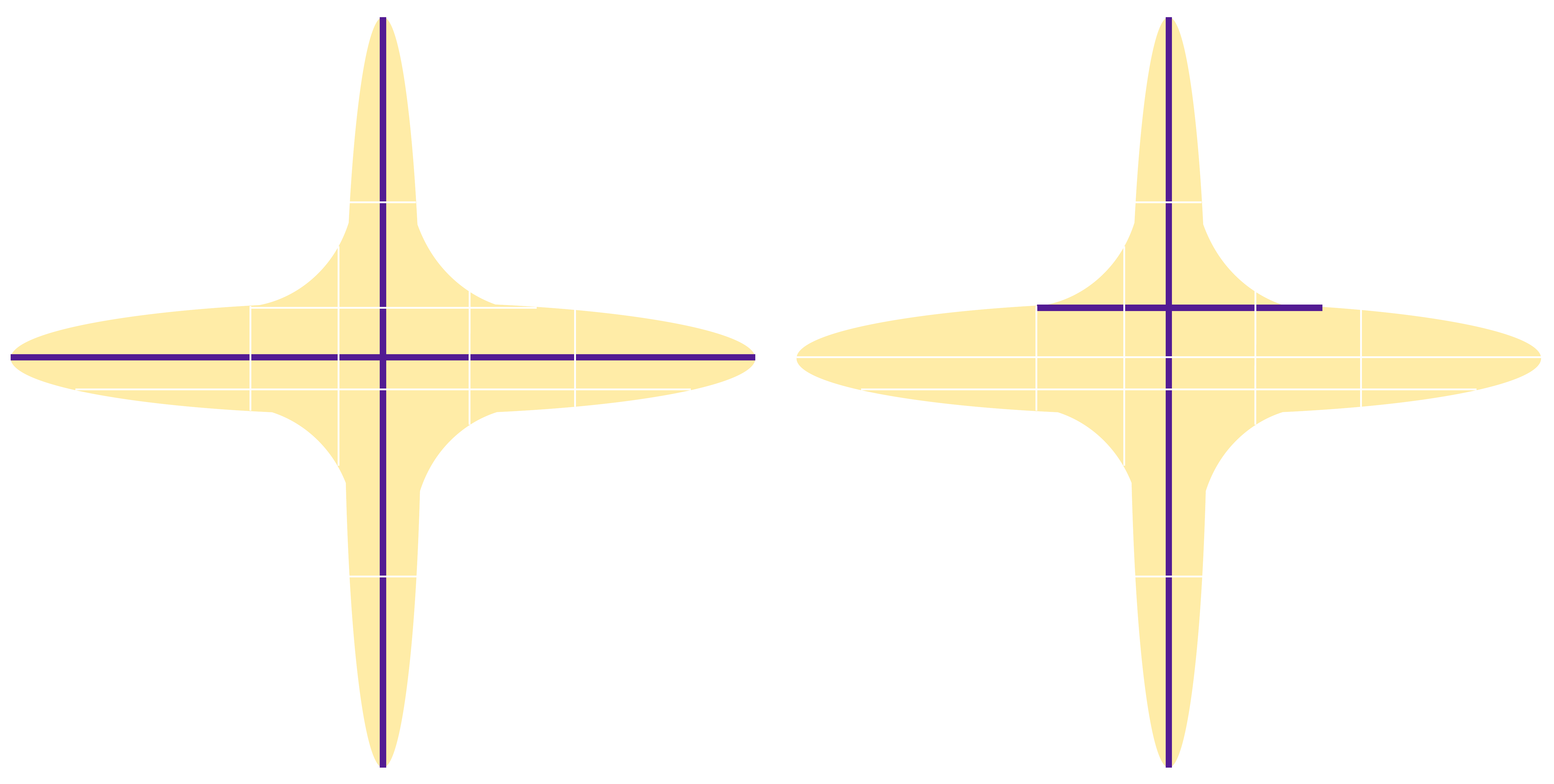}
    \caption{\textbf{Left}: The set of axis separators in {\color{blueviolet} dark blue} are a backbone since they intersect all the others (Definition \ref{def:backbone}). \textbf{Right}: The set of separators in dark blue is \textbf{not} a backbone since the horizontal axis separator does not intersect the right-most vertical separator. }
    \label{fig:backbone}
\end{figure}

\section{Proof of non-removable discontinuity preservation (Theorem~\ref{thm:discontinuities-preservation})}
\label{app:proof_nonremovable}
\begin{proof}
Let us denote $J_{h}(z)=\frac{\partial h}{\partial z}(z)$ the Jacobian
of $h$, and $J_{h^{-1}}(z')=\frac{\partial h^{-1}}{\partial z'}(z')$
the Jacobian of $h^{-1}$. Suppose $p_Z$ is one of the PDFs of $Z$, from this we can obtain a PDF of $Z'$ using the change of variable formula:
$p_{Z'}(z')=p_{Z}(h^{-1}(z'))|\det J_{h^{-1}}(z')|$.
Symmetrically, we can say that if $p_{Z'}$ is a PDF of $Z'$, we obtain a PDF version of $Z$ as follows:
$p_{Z}(z)=p_{Z'}(h(z))|\det J_{h}(z)|$.

Suppose the PDF $p_{Z}$ has a non-removable discontinuity at $z_0$. Pick one of the PDFs of $Z'$, let us call it $p_{Z'}$. There are three possibilities for what could happen at $Z'=h(z_0)$. 
\begin{itemize}
\item $p_{Z'}$ is continuous at $h(z_0)$. We can apply the change of variables formula and obtain a PDF of $Z$ that is given as $p_{Z}(z)=p_{Z'}(h(z))|\det J_{h}(z)|$. Since the RHS is a product of two terms that are continuous at $z_0$, we conclude that $p_{Z}$ is continuous at $z_0$. This contradicts the fact that $p_{Z}$ has a non-removable discontinuity at $z_0$. 
\item $p_{Z'}$ is discontinuous at $h(z_0)$ but the discontinuity is removable. Therefore, there exists a  PDF $p_{Z'}$ that is continuous at $h(z_0)$. We can now follow the same argument as the above bullet to construct a PDF of $Z$ that is continuous at $z_0$, which would contradict the fact that $p_Z$ has a removable discontinuity at $z_0$.
\item Finally, we are only left with the case that $p_{Z'}$ has a non-removable discontinuity at $h(z_0)$, which is what we set out to prove.
\end{itemize}

\end{proof}

\section{Proof of grid structure preservation and recovery (Theorem~\ref{thm:gridStructure})}
\label{app:grid_struct_proof}

\subsection{Proof of Step 1 -- Recovery of all separators}
\label{sec:step-1}
Knowing, from Theorem~\ref{thm:discontinuities-preservation}, that the set of \emph{points} making up the axis-aligned grid
$G$ maps through $h$ to the set of \emph{points} making up the axis-aligned
grid $G'$ (i.e. $G'=h(G)$), our first major step consists in establishing
that the \emph{axis-separators} that make up $G$ (i.e. the elements
of $\mathcal{G}$) map one-to-one to the \emph{axis-separators} that
make up $G'$ (i.e. the elements of $\mathcal{G}'$). We can denote
this simply as $G'=h(G)\Longrightarrow\mathcal{G}'=h(\mathcal{G})$.

\subsubsection*{Overview}
The high-level proof is as follows:
\begin{itemize}
\item Since $H\in\mathcal{G}$ is a connected smooth hypersurface in $\mathcal{S}$, and a $C^{\infty}$-diffeomorphism maps connected sets to connected sets and smooth hypersurfaces to smooth hypersurfaces, we get that $h(H)$ is a connected smooth hypersurface in $\mathcal{S}'$.
\item From Theorem~\ref{thm:discontinuities-preservation}, we also know that $h(H)\subset G'$.
\item Next, we establish that the only smooth connected hypersurfaces in $\mathcal{S}'$ that are included in $G'$ are necessarily subsets of a single axis-separator of $\mathcal{G}'$ . This is fundamentally due to the fact that a connected smooth hypersurface
cannot spread from one separator of the grid to another along their orthogonal intersection, as it would no longer be smooth (having a ``kink''), so it has to stay within a single
separator.
\item We conclude that $h(H)$ is necessarily a subset of a single axis-separator $H'\in\mathcal{G'}$.
\item We, then, show that not only is $h(H)$ a subset of a single axis-separator $H'\in\mathcal{G'}$, but that it has to be that entire separator. From the previous point, the reverse diffeomorphism $h^{-1}$ must map back the would-be remaining part of $H'$ (i.e. $H'\setminus h(H)\neq \emptyset$) to a subset of the same separator as it maps back $h(H)$, i.e. to $H$. But this leads to a contradiction, since that remaining part did not come from $H$ initially. (See proof of Lemma~\ref{lem:h-yields-separator} in Appendix).
\item  We have thus shown that $H\in\mathcal{G} \implies h(H)\in \mathcal{G'}$ . It suffices to apply this result in the other direction using $h^{-1}$ to establish the converse. We thus have a bijection: the one-to-one mapping we needed to prove. Which we can write succinctly as $\mathcal{G}'=h(\mathcal{G})$.
\end{itemize}

\subsubsection{Detailed proof of Step 1 -- recovery of all separators}
\label{sec:detailed-proof-step-1}

The goal of step 1 is to establish
that the \emph{axis-separators} that make up $G$ (i.e. the elements
of $\mathcal{G}$) map one-to-one to the \emph{axis-separators} that
make up $G'$ (i.e. the elements of $\mathcal{G}'$). We can denote
this simply as $G'=h(G)\Longrightarrow\mathcal{G}'=h(\mathcal{G})$.
We provide a detailed proof here. Note that we always assume finite axis-separator sets.

\subsubsection*{Preliminaries }

Whenever we say hypersurface, it is always defined as a $d-1$ dimensional
regular submanifold embedded in $d$-dimensional \emph{ambient space
}$\mathcal{S}\subset\mathbb{R}^{d}$, where $\mathcal{S}$ is a $d$-dimensional connected
open submanifold of $\mathbb{R}^{d}$. In our application $\mathcal{S}$
will be the interior of the support of the density we consider.
\begin{itemize}
\item \textbf{\underline{Definition:}}  
 \textbf{Intersection set.} given a
grid $G=\cup\mathcal{G}=\cup_{H\in\mathcal{G}}H$, we define its \emph{intersection
set }$I(\mathcal{G})$ as the set of points that belong to intersections
of 2 or more distinct separators of $\mathcal{G}$. Formally: $I(\mathcal{G})=\cup_{H\in\mathcal{G},H'\in\mathcal{G},H'\ne H}(H\cap H')$.
\item \textbf{\underline{Definition:}} \textbf{Exclusive point.} We say that
a point $z$ is exclusive to a separator $H$ of a grid $G=\cup\mathcal{G}$
if it belongs to $H$ but does not belong to any other separator of
the grid (i.e. it does not belong to $I$). Similarly we will say
that a set is exclusive to a separator if all its elements are exclusive
points of that separator. The set of points of a separator $H$ that are exclusive to it will be denoted $\breve{H}=H\setminus I$.
\item \textbf{\underline{Definition:}} \textbf{Tangent space} we view the
tangent spaces to hypersurfaces embedded in an ambient space included
in $\mathbb{R}^{d}$ literally as affine subspaces of $\mathbb{R}^{d}$, i.e. we use the traditional view\footnote{This traditional extrinsic view of tangent space is preferred here to more modern definitions, because it simplifies a step in our proof. It is also arguably easier to intuit and follow for readers who may not be familiar with differential geometry.} of tangent space~\citep{docarmo1976}, which is a natural generalization of the notion of a plane tangent to a surface at a point, to higher dimensional hypersurfaces embedded in $\mathbb{R}^{d}$.  The tangent space at
$z\in A$ to a hypersurface $A$ will be denoted $T^A(z) = T_{z}A$. A smooth
hypersurface has the property that it has at every $z\in A$ a well-defined
tangent space $T^A(z) = T_{z}A$ of the same dimension as the hypersurface. When $A$ is a smooth hypersurface,  $T^A$ is a smooth map $T^A: A \rightarrow \mathrm{Graff}_{d-1}(\mathbb{R}^d)$ that maps any point $z$ of $A$ to a point of the \emph{affine-Grassmannian manifold}~\citep{klain-affine-grassmannian-1997, Lim-affine-grassmannian-2021} $\mathrm{Graff}_{d-1}(\mathbb{R}^d)$, i.e. the space of all $d-1$ dimensional affine-subspaces of $\mathbb{R}^d$.
Since $T^A$ is a continuous map between smooth manifolds, $T^A(z)$ will be continuous in any local (or global) parametrization of $A$ around $z$. (continuity based on the topology of the affine-Grassmannian manifold for comparing tangent spaces as affine-subspaces of $\mathbb{R}^d$).\\
\underline{Note} that the
tangent space to any axis-separator $H$ is a constant: it is the
affine subspace confounded with the hyperplane that includes the separator,
and will be denoted $\mathcal{T}_{H}$. i.e. we have $\forall z\in H,\;T^H(z)=T_{z}H=\mathcal{T}_{H}$.
Note also that with this affine subspace definition of tangent space,
$\mathcal{T}_{H}$ is different for every separator $H$ of an axis-aligned
grid: $\forall H_1 \in \mathcal{G}, \forall H_2 \in \mathcal{G}, \mathcal{T}_{H_1}=\mathcal{T}_{H_2} \Leftrightarrow H_1=H_2$.
\item \textbf{\underline{Useful properties:}} We will also use the following
properties that are either well-established differential geometry
knowledge or straightforward corollaries thereof \citep{Lee00}.

\begin{itemize}
\item \textbf{Property 1:} \emph{A $C^\infty$-diffeomorphism maps a smooth hypersurface
to a smooth hypersurface.}
\item \textbf{Property 2: }\emph{A diffeomorphism maps a path-connected set to a path-connected set.}
\item \textbf{Property 3:} \emph{Smooth connected hypersurfaces in $\mathbb{R}^{d}$ have a $d-1$ dimensional tangent space that is well-defined all over the hypersurface and continuous (in the sense defined above, see tangent space).}
\item \textbf{Property 4:} \emph{A non-empty open subset of a smooth hypersurface in ambient space is itself a smooth hypersurface in ambient space.}
\item \textbf{Property 5:}\emph{ A hypersurface that is a subset of another hypersurface has at every of its points the same tangent space as the hypersurface it is a subset of.}
\end{itemize}
\end{itemize}

\subsubsection*{Detailed proof}

\begin{lemma}
    \label{lem:no-I}No subset of the intersection set $I(\mathcal{G})$
of an axis-aligned grid $G=\cup\mathcal{G}$ can be a hypersurface in ambient space.
\end{lemma}

\begin{proof}
Consider $I(\mathcal{G})$ the intersection set of a grid $G=\cup\mathcal{G}$.
Formally: \\ $I(\mathcal{G})=\cup_{H\in\mathcal{G},H'\in\mathcal{G},H'\ne H}(H\cap H')$.
Each $H\cap H'$, if it is non-empty, is the intersection of two orthogonal
(thus transversal) connected hypersurfaces (i.e. $d-1$ dimensional
submanifolds embedded in ambient space), so that their intersection
can be at most a $d-2$ dimensional embedded submanifold of ambient
space. The union of a finite number of at most $d-2$ dimensional
submanifolds cannot be more than $d-2$ dimensional, so $I(\mathcal{G})$
cannot be more than $d-2$ dimensional. Consequently no subset of
$I(\mathcal{G})$ can be more than $d-2$ dimensional, thus it cannot be a hypersurface
in ambient space.
\end{proof}

\begin{lemma}
    \emph{\label{lem:contains-exclusive-point}}Let $A$ be a connected smooth
hypersurface included in an axis-alined grid $G=\cup\mathcal{G}$
with axis-separator set $\mathcal{G}$. Let $z\in A$. All open neighborhoods
of $z$ in $A$ will necessarily contain at least one point that is
exclusive to a separator of $\mathcal{G}$. 
\end{lemma}

\begin{proof}
An open neighborhood $\mathcal{B}_{z}^{A}$ of $z$ in $A$ is an
open subset of $A$, thus from Property 4, $\mathcal{B}_{z}^{A}$
is a hypersurface in ambient space. From Lemma \ref{lem:no-I} no
subset of $I(\mathcal{G})$, (the set of points of $G$ that belong to more than
one separator) can be a hypersurface. So $\mathcal{B}_{z}^{A}$ cannot
be a subset of $I(\mathcal{G})$, i.e. it must contain at least one
point exclusive to a separator of $\mathcal{G}$.
\end{proof}

\begin{lemma}
    \label{lem:exclusive-neighborhood}Let $A$ be a connected smooth hypersurface
included in an axis-alined grid $G=\cup\mathcal{G}$ with axis-separator
set $\mathcal{G}$. Let $z$ be a point of $A$ that is exclusive to
a separator $H\in\mathcal{G}$ (i.e. $z\ 
 \in H\setminus I(\mathcal{G})$: it belongs to no other separator
of $\mathcal{G}$), then there exists an open connected neighborhood
$\mathcal{B}_{z}^{A}$ of $z$ in $A$ that is exclusive to $H$. 
\end{lemma}

\begin{proof}
We reason using the usual Euclidean distance in $\mathbb{R}^{d}$.
Consider an open $d$-ball $\mathcal{B}_{z}^{d}$ in $\mathbb{R}^{d}$
centered on $z$ and whose radius $\epsilon$ is chosen to be less than
the smallest distance of $z$ to any other separator, i.e. such that
$0<\epsilon<\inf_{z'\in(G\setminus H)}\|z-z'\|$. Since $z$ is exclusive
to separator $H$ and the number of separators is finite, this distance will be greater than
0. Then all points of $G$ within a distance less than $\epsilon$
of $z$ will necessarily belong exclusively to $H$, i.e. $\mathcal{B}_{z}^{d}\cap G\subset \breve{H}$, where $\breve{H}=H\setminus I(\mathcal{G})$.
Now we can choose a sufficiently small connected open neighborhood
$\mathcal{B}_{z}^{A}$ of $z$ in $A$ so that the distance in ambient
space between $z$ and any other point of $\mathcal{B}_{z}^{A}$ is
less than $\epsilon$. Thus $\mathcal{B}_{z}^{A}\subset\mathcal{B}_{z}^{d}$.
Since we also have $\mathcal{B}_{z}^{A}\subset A\subset G$ this implies that $\mathcal{B}_{z}^{A}\subset\mathcal{B}_{z}^{d}\cap G$
and consequently that $\mathcal{B}_{z}^{A}\subset \breve{H}$. We have thus
shown that there exists an open connected neighborhood of $z$ in $A$
that is exclusive to $H$. \textbf{}
\end{proof}

\begin{lemma}
    \label{lem:boundary-z}Let $A$ be a connected smooth hypersurface included
in an axis-alined grid $G=\cup\mathcal{G}$ with axis-separator set
$\mathcal{G}$. Then for any point $z\in A$ there exists a non-empty
open subset $B$ whose boundary contains $z$ and such that $B$ is
a non-empty open subset exclusive to one of the separators. 
\end{lemma}

\begin{proof}
There are two cases to consider for $z$: either $z$ is an exclusive
point of a separator of the grid, or it is an intersection point of
separators (belonging to $I(\mathcal{G}))$.

\underline{First case:} $z$ is a point exclusive to a separator $H\subset\mathcal{G}$.\\
Then, by Lemma \ref{lem:exclusive-neighborhood}, we know that there
exists an open connected neighborhood $\mathcal{B}_{z}^{A}$ of $z$
in $A$ that is exclusive to $H$. We can then easily pick an open
subset $B$ of $\mathcal{B}_{z}^{A}$ whose boundary contains $z$
(For instance, pick a close neighbor $z_{1}$ of $z$ in $\mathcal{B}_{z}^{A}$,
and construct $B$ as the intersection of $\mathcal{B}_{z}^{A}$ with
an open ball centered on $z_{1}$ and of radius $\|z_{1}-z\|$). $B$
is an open subset exclusive to $H$, the separator that $z$ belongs
to.

\underline{Second case:} $z$ is not exclusive to any separator of
the grid.\\
Let $\mathcal{G}=\{H_{1},\ldots,H_{k}\}$ be the finite set of separators
of grid $G=\cup\mathcal{G}$. Let $\breve{H}_{i}=H_{i}\setminus I(\mathcal{G})$
be the corresponding subset of exclusive points to each separator $H_{i}$,
and let $\breve{A}_{i}=\breve{H}_{i}\cap A$, for each $i\in\{1,\ldots,k\}$.
So $\breve{A}_{i}$, if it is not empty, will contain only points
exclusive to $H_{i}$. From Lemma~\ref{lem:exclusive-neighborhood} we deduce that every point
of $\breve{A}_{i}$ has an open neighborhood in $A$ exclusive to
$H_{i}$: this open neighborhood is thus included in $A\cap\breve{H}_{i}$
and is thus a subset of $\breve{A}_{i}$. We have thus shown that
every point of $\breve{A}_{i}$ has an open neighborhood in $A$ that
is included in $\breve{A}_{i}$. From this we conclude that each $\breve{A}_{i}$
is an open subset (possibly empty) of $A$. \\
We know that $z$ belongs to none of the $\breve{A}_{i}$, since it
is not exclusive to any separator. Now we will show that $z$ belongs
to the \emph{boundary} of at least one of the $\breve{A}_{i}$. We will
 reason using the metric $d^{A}$ induced on embedded submanifold
$A\subset\mathbb{R}^{d}$ by the usual Euclidean metric in ambient
space $\mathbb{R}^{d}$. Let $\epsilon=\min_{i\in\{1,\ldots,k\}}d^{A}(z,\breve{A}_{i})$. We can use a min since it is over a finite number $k$ of separators. Note that $d^{A}(z,\breve{A}_{i})=\inf_{z'\in\breve{A}_{i}}d^{A}(z,z')$ will be $+\infty$ if $\breve{A}_{i}$ is empty, by the definition of the infimum. 
If $\epsilon$ was strictly greater than 0, then this would mean that
no point of $A$ exclusive to any separator would be at a distance
strictly less than $\epsilon$ from $z$ (since any point of $A$
exclusive to a separator belongs to one of the $\breve{A}_{i}$). Thus the open ball $\mathcal{B}^{A}(z,\epsilon)=\{z'\in A,d^{A}(z,z')<\epsilon\}$
would not contain any point exclusive to any separator. But this would
contradict Lemma~\ref{lem:contains-exclusive-point}. So necessarily $\epsilon=0$. This implies that
there is at least one of the $\breve{A}_{i}$ whose distance to $z$
is 0, i.e. there exists a $k^{*}\in\{0,\ldots,k\}$ such that $d^{A}(z,\breve{A}_{k^{*}})=0$.
Since $z\notin\breve{A}_{k^{*}}$ we conclude that $z$ belongs to
the \emph{boundary} of this $\breve{A}_{k^{*}}$. Moreover this $\breve{A}_{k^{*}}$
is non-empty (otherwise that distance would be $+\infty$). It is thus an open-subset of $A$, exclusive to separator $H_{k}$.
We have thus established that there exists a non-empty open subset of $A$ exclusive to one of the separators, and whose boundary contains $z$.

\end{proof}

\begin{lemma}
    \emph{\label{lem:path-included-in-H}}Let $A$ be a connected smooth
hypersurface included in an axis-alined grid $G=\cup\mathcal{G}$
with axis-separator set $\mathcal{G}$. Let $\gamma$ be a continuous
path, included in $A$, that starts at a point $z_{1}$, where $z_{1}$
is exclusive to a separator $H\in\mathcal{G}$. Then $\gamma$ will
necessarily be included entirely in $H$. 
\end{lemma}

\begin{proof}
Consider path $\gamma:[0,1]\rightarrow A$, where
$\gamma(0)=z_{1}$ is exclusive to $H$. From Lemma \ref{lem:boundary-z},
for each point $\gamma(t)\in A$, there exists an open subset $B_{t}$
of $A$ whose boundary contains $\gamma(t)$ and such that $B_{t}$
is an open subset exclusive to one of the separators. 
Let us call this separator $H_{t}$ (Note that there may
be multiple possible choices for $B_{t}$ and $H_{t}$). Consider
any point $z\in B_{t}$. Since $B_{t}$ is a non-empty open subset
of smooth hypersurface $A$, by Property 4 it is a hypersurface and
by Property 5 $B_{t}$ and $A$ will have the same tangent space, so that $T_{z}B_{t}=T_{z}A$.
Since $B_{t}$ is also a subset of smooth hypersurface $H_{t}$, we
have by Property 5 that $T_{z}B_{t}=T_{z}H_{t}$. Thus $T_{z}A=T_{z}B_{t}=T_{z}H_{t}$.
Now the tangent space to any axis-separator $H'$ is the constant
$\mathcal{T}_{H'}$. We can thus write, for any $z\in B_{t}$ , $T_{z}A=T_{z}B_{t}=T_{z}H_{t}=\mathcal{T}_{H_{t}}$.
Since $A$ is a connected smooth hypersurface, it has a continuous
and well defined tangent space at every point (see Property 3 and how we defined tangent space). Thus, if the tangent space is constant on an open subset $B_{t}\subset A$, it will have that same constant value at its boundary. So the tangent space to $A$ at point $\gamma(t)$, which belongs to the boundary of $B_{t}$,
will also be $T_{\gamma(t)}A=\mathcal{T}_{H_{t}}$. For the same reason of the continuity of the tangent space of a path-connected smooth hypersurface $A$, we cannot have, along the curve $\gamma(t)$, an abrupt change in the tangent space $T_{\gamma(t)}A$, consequently $\mathcal{T}_{H_{t}}$ cannot change abruptly along the path. The only way for it not to change abruptly is that $H_{t}$ stays constant along the path: $H_{t}=\mathrm{constant}\;\forall t\in[0,1]$. In other words, for any point $\gamma(t)$ along the path there must exist an open subset $B_{t}$ of $A$ that is included in and exclusive to the same constant separator along the path. Now, if $z_{1}=\gamma(0)$ is exclusive to a separator $H$, then $H_{0}=H$ and we must thus have $H_{t}=H$, $\forall t\in[0,1]$. We have thus shown that if the path starts at a point $z_{1}=\gamma(0)$ which is exclusive to a separator $H$, then all points of the path necessarily belong to $H$ (though not necessarily exclusively to $H$). Thus, the path is entirely included in $H$.
\end{proof}


\begin{lemma}
    \emph{\label{lem:subset-of-separator}}A path-connected smooth hypersurface
$A$ included in an axis-alined grid $G=\cup\mathcal{G}$ is necessarily
a subset of one separator of $\mathcal{G}$.
\end{lemma}

\begin{proof}
Let $z_{0}$ be a point of $A$ that is exclusive to a separator $H\in\mathcal{G}$.
We know from Lemma \ref{lem:contains-exclusive-point} that such a
point exists. Since $A$ is path-connected, there exists in $A$ a
continuous path connecting $z_{0}$ to any point $z\in A$. Thus, Lemma
\ref{lem:path-included-in-H} leads to conclude that $\forall z\in A,z\in H$.
Thus $A\subset H$.
\end{proof}


\begin{lemma}
\label{lem:h-yields-separator}
\emph{
Let $G=\cup\mathcal{G}$ be an axis-aligned
grid in $\mathcal{S}$. Let $h:\mathcal{S}\rightarrow\mathcal{S}'$ be a diffeomorphism. Let $G'=\cup\mathcal{G}'$ be an axis-aligned grid in $\mathcal{S}'$. If $h(G)=G'$, the image of a separator $H_{1}\in\mathcal{G}$ by the diffeomorphism $h$ will be a separator $H'\in\mathcal{G}'$, i.e. $H_{1}\in\mathcal{G}\implies h(H_{1})\in\mathcal{G}'$.}
\end{lemma}

\begin{proof}
An axis-separator $H_{1}\subset G$ is a path-connected smooth hypersurface. From Property 1 and Property 2, its image by the diffeomorphism $h$ will be a path-connected smooth hypersurface $h(H_{1})\subset G'$.
So $h(H_{1})$ is a path-connected smooth hypersurface included in axis-aligned grid $G'$. Consequently, Lemma \ref{lem:subset-of-separator}
guarantees that we have $h(H_{1})\subset H'$ for some $H'\in\mathcal{G}'$.
We will now prove that $h(H_{1})=H'$. Suppose by contradiction that $h(H_{1})\subsetneq H'$, and let $B' = H' - h(H_{1})\ne\emptyset$.
Similarly, if we apply the reverse diffeomorphism $h^{-1}$, we will have $h^{-1}(H')\subset H_{2}$ for some $H_{2}\in\mathcal{G}$. Consequently,
the two disjoint sets composing $H'=B'\cup h(H_{1})$ will both map back to subsets of $H_{2}$, i.e. $h^{-1}(B')\subset H_{2}$ and $h^{-1}(h(H_{1}))\subset H_{2}$.
The latter can be rewritten as $H_{1}\subset H_{2}$, which implies $H_{2}=H_{1}$ since no two distinct separators of $\mathcal{G}$ are included in one another. So we have $h^{-1}(B')\subset H_{1}$.
Thus $h(h^{-1}(B'))\subset h(H_{1})$, hence $B'\subset h(H_{1})$.
We had defined $B'$ as $B'=H'-h(H_{1})\ne\emptyset$ but a non-empty $B'$ cannot at the same time correspond to a set from which we removed
$h(H_{1})$ \emph{and} be included in $h(H_{1})$. We have a contradiction, so we cannot have $h(H_{1})\subsetneq H'$, therefore $h(H_{1})=H'$.
\end{proof}

\begin{proposition}
    \emph{\label{prop:Step-1} } The diffeomorphism $h$ maps separators in $\mathcal{G}$ one-to-one to separators in $\mathcal{G}'$, i.e. $h(G)=G'\implies h(\mathcal{G})=\mathcal{G}'$.
\end{proposition}

\begin{proof}
We have shown in Lemma \ref{lem:h-yields-separator} that $H\in\mathcal{G}\implies h(H)\in\mathcal{G}'$.
It suffices to apply this result in the other direction using $h^{-1}$
to establish the converse. We thus have a bijection: the one-to-one
mapping we needed to prove. Which we can write succinctly $h(\mathcal{G})=\mathcal{G}'$.
\end{proof}

\subsection{Step 2 -- Recovery of partition into sets of parallel separators }

We have established in step 1 that we recover the set of all separators $\mathcal{G}'=h(\mathcal{G})$.
Our next step is to recover its partition into subsets of parallel separators
(each subset associated to an axis): $\mathcal{G}'^{(j)}=h(\mathcal{G}^{(i)})$
(with permutation $j=\sigma(i)$).

\subsubsection*{Proof for step 2}

Consider $d$ separators forming a backbone of $\mathcal{G}$, recall
that a backbone is constituted of $d$ distinct axis-separators that
intersect in a single point, i.e.
$\text{\ensuremath{\mathcal{H}_{1}^{*}}}\in\mathcal{G}^{(1)},\ldots,\text{\ensuremath{\mathcal{H}_{d}^{*}}}\in\mathcal{G}^{(d)}$
, $\bigcap_{i=1}^{d}\text{\ensuremath{\mathcal{H}_{i}^{*}}}=\{z^{*}\}$.\\
We have that $\forall j\ne i,\,\mathcal{H}_{i}^{*}\ne\mathcal{H}_{j}^{*}\implies\forall j\ne i,\,h(\mathcal{H}_{i}^{*})\ne h(\mathcal{H}_{j}^{*})$.\\
We also have that $\bigcap_{i=1}^{d}\text{\ensuremath{\mathcal{H}_{i}^{*}}}=\{z^{*}\}\implies\bigcap_{i=1}^{d}h(\text{\ensuremath{\mathcal{H}_{i}^{*}}})=\{h(z^{*})\}$ (as $h$ is a bijection).\\
Moreover, we know from step 1 that $\mathcal{H}_{i}^{*}\in\mathcal{G}\implies h(\mathcal{H}_{i}^{*})\in\mathcal{G}'$.
In short, the $h(\mathcal{H}_{1}^{*}),\ldots,h(\mathcal{H}_{d}^{*})$
are $d$ distinct separators, each an element of $\mathcal{G}'$, that
intersect in a single point $h(z^{*})$. The only sets of $d$ \emph{distinct}
separators in $\mathcal{G}'$ that pass through a same point are $d$
separators defined along each of the $d$ different axes of $\mathcal{Z}'=\mathbb{R}^{d}$.
Thus there exists a permutation $\sigma$ such that for such backbone
separators, $\text{\ensuremath{\mathcal{H}_{i}^{*}}}\in\mathcal{G}^{(i)}\implies h(\text{\ensuremath{\mathcal{H}_{i}^{*}})}\in\mathcal{G}^{'(\sigma(i))}$.

Now consider any other separator $H\in\mathcal{G}^{(i)}$. From the
definition of the backbone, we know that $\,H\cap\mathcal{H}_{j}^{*}\ne\emptyset$,
$\forall j\ne i$.\\
This implies that $h(H)\cap h(\mathcal{H}_{j}^{*})\ne\emptyset$,
$\forall j\ne i$. The fact that $h(H)$ intersects a separator $h(\text{\ensuremath{\mathcal{H}_{j}^{*}})}\in\mathcal{G}^{'(\sigma(j))}$
implies that it does not belong to parallel-separator-set $\mathcal{G}^{'(\sigma(j))}$.
Thus $\forall j\ne i,\,h(H)\notin\mathcal{G}^{'(\sigma(j))}$. So
there is just one parallel separator set left which $h(H)$ can belong
to: $h(H)\in\mathcal{G}^{'(\sigma(i))}$.\\
In short, we have proved that $H\in\mathcal{G}^{(i)}\implies h(H)\in\mathcal{G}^{'(\sigma(i))}$.

Since distinct separators map to distinct separators, and each has
to belong to exactly one of the $\mathcal{G}^{'(k)}$, this mapping
is a bijection and we can write 
$H\in\mathcal{G}^{(i)}\Longleftrightarrow h(H)\in\mathcal{G}^{'(\sigma(i))}$,
or in short $h(\mathcal{G}^{(i)})=\mathcal{G}'^{(j)}$ with $j=\sigma(i)$.

\subsection{Step 3 -- Recovery of coordinate ordering}

The last step consists in showing that the ordering of the separators in a parallel-separators-set is preserved (up to possible order reversal). 

\subsubsection*{Proof for step 3 -- Overview}

The gist of the proof is as follows (a  complete detailed proof is provided in Appendix~\ref{sec:detailed-proof-step-3}): 

We first establish that $h$ preserves separators and halves. This follows directly from the preservation of inclusion, connectedness and set operations under diffeomorphisms. 
Then, we use the fact that inclusion defines a strict order relationship between positive halves
associated to a coordination, and similarly between negative halves.
As inclusion is preserved by a diffeomorphism, this order relationship is preserved. We can use this to show that the order implied by $\mathbf{A}_{i}$
is either conserved, as is, in $\mathbf{B}{}_{j}$ (negative halves of coordination $\mathbf{A}$ being mapped to negative halves of $\mathbf{B}$)
or simply reversed (negative halves of $\mathbf{A}$ are being mapped to positive halves of $\mathbf{B}$). This directly yields the result of the main
Theorem (\ref{thm:gridStructure}).

\subsubsection{Detailed proof for Step 3}
\label{sec:detailed-proof-step-3}

\subsubsection*{Preliminary lemma}

\begin{lemma}
    \label{lem:Preservation-of-separator-and-halves}
    Preservation of separator and halves under a diffeomorphism: If $h$ is a diffeomorphism and $\mathcal{C}$ is a separator of $\mathcal{S}$ that splits it in two halves $\mathcal{C}^{+}$ and $\mathcal{C}^{-}$, then $h(\mathcal{C})$ is a separator of $h(\mathcal{S})$ that splits it in two halves $h(\mathcal{C}^{+})$ and $h(\mathcal{C}^{-})$\\
Formally:
\begin{align*}
 & \mathcal{C} \subset \mathcal{S}, \ \mathcal{C}, \ \mathrm{connected}, \ \mathrm{split}(\mathcal{S}, \ \mathcal{C}) = \{ \mathcal{C}^{+}, \mathcal{C}^{-} \} \\
\Longleftrightarrow &\  h(\mathcal{C}) \subset h(\mathcal{S}), \ h(\mathcal{C}), \ \mathrm{connected}, \ \mathrm{split}(h(\mathcal{S}), \ h(\mathcal{C})) = \{ h(\mathcal{C}^{+}), h(\mathcal{C}^{-}) \}
\end{align*}
\end{lemma}

\begin{proof}
This follows from preservation of inclusion, connectedness, and set operations (union, intersection, difference) under a diffeomorphism.\\
Formally: $\mathcal{C}\subset\mathcal{S}\implies h(\mathcal{C})\subset h(\mathcal{S})$.
\\
$\text{\ensuremath{\mathcal{C}}}^{+}$ and $\mathcal{C}^{-}$ being the connected components of $\mathcal{S}-\mathcal{C}$ implies that $\mathcal{C}^{+}$ and $\mathcal{C}^{-}$ are each connected, and that $\mathcal{S} \setminus \mathcal{C}=\mathcal{C}^{+}\cup\mathcal{C}^{-}$, where $\mathcal{C}^{+}\cup\mathcal{C}^{-}$ is not connected.\\
Each of $\mathcal{S},$$\mathcal{C}$, $\mathcal{C}^{+}$, $\mathcal{C}^{-}$ connected $\implies$Each of $\mathcal{S},h(\mathcal{C})$, $h(\mathcal{C}^{+})$,
$h(\mathcal{C}^{-})$ connected. \\
$\mathcal{S}\setminus \mathcal{C}=\mathcal{C}^{+}\cup\mathcal{C}^{-}\implies h(\mathcal{S})\setminus h(\mathcal{C})=h(\mathcal{C}^{+})\cup h(\mathcal{C}^{-})$\\
$\mathcal{C}^{+}\cup\mathcal{C}^{-}$ not connected $\implies$ $h(\mathcal{C}^{+})\cup h(\mathcal{C}^{-})$
not connected.\\
That $h(\mathcal{C}^{+})\cup h(\mathcal{C}^{-})$ is not connected
but $h(\mathcal{C}^{+})$ and $h(\mathcal{C}^{-})$ are each connected,
implies that $h(\mathcal{C}^{+})$ and $h(\mathcal{C}^{-})$ are the
two connected components of $h(\mathcal{C}^{+})\cup h(\mathcal{C}^{-})$
i.e. of $h(\mathcal{S})-h(\mathcal{C})$.\\
This implies that $\mathrm{split}(h(\mathcal{S}),h(\mathcal{C}))=\{h(\mathcal{C}^{+}),h(\mathcal{C}^{-})\}$.
The implication in the other direction can be obtained in the by applying
the same reasoning using $h^{-1}$.
\end{proof}

\subsubsection*{Proof of step 3}
Let $j=\sigma(i)$ and $K=|\mathbf{A}_{i}|=|\mathbf{B}_{j}|$ and
denote the corresponding set of axis separators as

$\mathcal{A}=\{\Gamma_{\mathcal{S}}(i,\mathbf{A}_{i,1}),\ldots,\Gamma_{\mathcal{S}}(i,\mathbf{A}_{i,K})\}$
and $\mathcal{B}=\{\Gamma_{\mathcal{S}'}(j,\mathbf{B}_{j,1}),\ldots,\Gamma_{\mathcal{S}'}(j,\mathbf{B}_{j,K})\}$

and denote the corresponding sets of halves:

$\mathcal{A}^{+}=\{\Gamma_{\mathcal{S}}^{+}(i,\mathbf{A}_{i,1}),\ldots,\Gamma_{\mathcal{S}}^{+}(i,\mathbf{A}_{i,K})\}$,
$\mathcal{A}^{-}=\{\Gamma_{\mathcal{S}}^{-}(i,\mathbf{A}_{i,1}),\ldots,\Gamma_{\mathcal{S}}^{-}(i,\mathbf{A}_{i,K})\}$,
$\mathcal{A}^{\pm}=\mathcal{A}^{+}\cup\mathcal{A}^{-}$

and $\mathcal{B}^{+}=\{\Gamma_{\mathcal{S}'}^{+}(j,\mathbf{B}_{j,1}),\ldots,\Gamma_{\mathcal{S}'}^{+}(j,\mathbf{B}_{j,K})\}$,
$\mathcal{B}^{-}=\{\Gamma_{\mathcal{S}'}^{-}(j,\mathbf{B}_{j,1}),\ldots,\Gamma_{\mathcal{S}'}^{-}(j,\mathbf{B}_{j,K})\}$,
$\mathcal{B}^{\pm}=\mathcal{B}^{+}\cup\mathcal{B}^{-}$

Proof Step 2, states that $h(\mathcal{A})=\mathcal{B}$.

And we have from the above Lemma that 
\begin{align*}
 & \mathrm{split}(\mathcal{S},\mathcal{C})=\{\mathcal{C}^{+},\mathcal{C}^{-}\}\\
\Longleftrightarrow & \mathrm{split}(h(\mathcal{S}),h(\mathcal{C}))=\{h(\mathcal{C}^{+}),h(\mathcal{C}^{-})\}
\end{align*}

thus the equality of the sets of separators $h(\mathcal{A})=\mathcal{B}$
obtained in Proof Step 2 implies an equality of the sets of halves:

\[
h(\mathcal{A}^{\pm})=\mathcal{B}^{\pm}
\]

Now, the only halves, among all halves, that do not include any of
the separators are $\Gamma_{\mathcal{S}}^{-}(i,\mathbf{A}_{i,1})$
and $\Gamma_{\mathcal{S}}^{+}(i,\mathbf{A}_{i,K})$ i.e. formally:

\[
\{\mathcal{C}\in\mathcal{A}|\forall\mathcal{H}\in\mathcal{A}^{\pm},\mathcal{C}\cap\mathcal{H}=\emptyset\}=\{\Gamma_{\mathcal{S}}^{-}(i,\mathbf{A}_{i,1}),\Gamma_{\mathcal{S}}^{+}(i,\mathbf{A}_{i,K})\}
\]

this property will naturally translate to their mapping by the diffeomorphism
$h$ (due to the preservation of inclusion an intersections)

hence

\[
\{\mathcal{C}\in h(\mathcal{A})|\forall\mathcal{H}\in h(\mathcal{A}^{\pm}),\mathcal{C}\cap\mathcal{H}=\emptyset\}=\{h(\Gamma_{\mathcal{S}}^{-}(i,\mathbf{A}_{i,1})),h(\Gamma_{\mathcal{S}}^{+}(i,\mathbf{A}_{i,K}))\}
\]

i.e. 
\[
\{\mathcal{C}\in\mathcal{B}|\forall\mathcal{H}\in\mathcal{B}^{\pm},\mathcal{C}\cap\mathcal{H}=\emptyset\}=\{h(\Gamma_{\mathcal{S}}^{-}(i,\mathbf{A}_{i,1})),h(\Gamma_{\mathcal{S}}^{+}(i,\mathbf{A}_{i,K}))\}
\]

but we also have, similarly, 

\[
\{\mathcal{C}\in\mathcal{B}|\forall\mathcal{H}\in\mathcal{B}^{\pm},\mathcal{C}\cap\mathcal{H}=\emptyset\}=\{\Gamma_{\mathcal{S}'}^{-}(i,\mathbf{B}_{i,1})),\Gamma_{\mathcal{S}'}^{+}(i,\mathbf{B}_{i,K}))\}
\]

From this we conclude that: 

\[
\{h(\Gamma_{\mathcal{S}}^{-}(i,\mathbf{A}_{i,1})),h(\Gamma_{\mathcal{S}}^{+}(i,\mathbf{A}_{i,K}))\}=\{\Gamma_{\mathcal{S}'}^{-}(i,\mathbf{B}_{i,1})),\Gamma_{\mathcal{S}'}^{+}(i,\mathbf{B}_{i,K}))\}
\]

Thus we have either one of two cases:

Case 1: $h(\Gamma_{\mathcal{S}}^{-}(i,\mathbf{A}_{i,1}))=\Gamma_{\mathcal{S}'}^{-}(i,\mathbf{B}_{i,1})$
and $h(\Gamma_{\mathcal{S}}^{+}(i,\mathbf{A}_{i,K}))=\Gamma_{\mathcal{S}'}^{+}(i,\mathbf{B}_{i,K})$.
We associate this case with $s_{i}=+1$.

Case 2: $h(\Gamma_{\mathcal{S}}^{-}(i,\mathbf{A}_{i,1}))=\Gamma_{\mathcal{S}'}^{+}(i,\mathbf{B}_{i,K})$
and $h(\Gamma_{\mathcal{S}}^{+}(i,\mathbf{A}_{i,K}))=\Gamma_{\mathcal{S}'}^{-}(i,\mathbf{B}_{i,1})$.
We associate this case with $s_{i}=-1$.

\paragraph{Case 1: $s_{i}=+1$, $h(\Gamma_{\mathcal{S}}^{-}(i,\mathbf{A}_{i,1}))=\Gamma_{\mathcal{S}'}^{-}(i,\mathbf{B}_{i,1})$
and $h(\Gamma_{\mathcal{S}}^{+}(i,\mathbf{A}_{i,K}))=\Gamma_{\mathcal{S}'}^{+}(i,\mathbf{B}_{i,K})$}

The half-spaces in $\mathcal{A}^{\pm}$ that include $\Gamma_{\mathcal{S}}^{-}(i,\mathbf{A}_{i,1})$
are only the $\Gamma_{\mathcal{S}}^{-}$, formally:

\[
\{\mathcal{H}\in\mathcal{A}^{\pm}|\Gamma_{\mathcal{S}}^{-}(i,\mathbf{A}_{i,1})\subset\mathcal{H}\}=\mathcal{A}^{-}
\]

this relationship will be maintained under a diffeomorphism $h$ i.e.

\[
\{\mathcal{H}\in h(\mathcal{A}^{\pm})|h(\Gamma_{\mathcal{S}}^{-}(i,\mathbf{A}_{i,1}))\subset\mathcal{H}\}=h(\mathcal{A}^{-})
\]

thus, since $h(\mathcal{A}^{\pm})=\mathcal{B}^{\pm}$ and $h(\Gamma_{\mathcal{S}}^{-}(i,\mathbf{A}_{i,1}))=\Gamma_{\mathcal{S}'}^{-}(i,B_{i,1})$
this can be rewritten as

\begin{align*}
\{\mathcal{H}\in\mathcal{B}^{\pm}|\Gamma_{\mathcal{S}'}^{-}(i,\mathbf{B}_{i,1})\subset\mathcal{H}\} & =h(\mathcal{A}^{-})\\
\mathcal{B}^{-} & =h(\mathcal{A}^{-})\\
\end{align*}

or, written less compactly: 
\[
\{h(\Gamma_{\mathcal{S}}^{-}(i,\mathbf{A}_{i,1})),\ldots,h(\Gamma_{\mathcal{S}}^{-}(i,\mathbf{A}_{i,K}))\}=\{\Gamma_{\mathcal{S}'}^{-}(j,\mathbf{B}{}_{j,1}),\ldots,\Gamma_{\mathcal{S}'}^{-}(j,\mathbf{B}{}_{j,K})\}
\]

Furthermore, strict inclusion defines an order relationship between
the elements of $\mathcal{A}^{-}$, which will be preserved under the
diffeomorphism, and thus defines a strict ordering between them:

\begin{align*}
 & \Gamma_{\mathcal{S}}^{-}(i,\mathbf{A}_{i,1})\subsetneq\Gamma_{\mathcal{S}}^{-}(i,\mathbf{A}_{i,2})\subsetneq\ldots\subsetneq\Gamma_{\mathcal{S}}^{-}(i,\mathbf{A}_{i,K})\\
\implies & h(\Gamma_{\mathcal{S}}^{-}(i,\mathbf{A}_{i,1}))\subsetneq h(\Gamma_{\mathcal{S}}^{-}(i,\mathbf{A}_{i,2}))\subsetneq\ldots\subsetneq h(\Gamma_{\mathcal{S}}^{-}(i,\mathbf{A}_{i,K}))
\end{align*}

we know that the $h(\Gamma_{\mathcal{S}}^{-}(i,A_{i,k}))$ are the
elements of $\mathcal{B}^{-}$ (as we have just sown that $\mathcal{B}^{-}=h(\mathcal{A}^{-})$),
i.e. the $\Gamma_{\mathcal{S}'}^{-}(j,\mathbf{B}_{j,k})$. Their order
is defined uniquely by strict inclusion as 
\[
\Gamma_{\mathcal{S}'}^{-}(j,\mathbf{B}{}_{j,1})\subsetneq\Gamma_{\mathcal{S}'}^{-}(j,\mathbf{B}{}_{j,2})\subsetneq\ldots\subsetneq\Gamma_{\mathcal{S}'}^{-}(j,\mathbf{B}{}_{j,K})
\]

thus we can conclude not only (as we showed with $\mathcal{B}^{-}=h(\mathcal{A}^{-})$)
that
\[
\{h(\Gamma_{\mathcal{S}}^{-}(i,\mathbf{A}_{i,1})),\ldots,h(\Gamma_{\mathcal{S}}^{-}(i,\mathbf{A}_{i,K}))\}=\{\Gamma_{\mathcal{S}'}^{-}(j,\mathbf{B}{}_{j,1}),\ldots,\Gamma_{\mathcal{S}'}^{-}(j,\mathbf{B}{}_{j,K})\}
\]

but also that their ordering is preserved i.e.

\[
(h(\Gamma_{\mathcal{S}}^{-}(i,\mathbf{A}_{i,1})),\ldots,h(\Gamma_{\mathcal{S}}^{-}(i,\mathbf{A}_{i,K})))=(\Gamma_{\mathcal{S}'}^{-}(j,\mathbf{B}{}_{j,1}),\ldots,\Gamma_{\mathcal{S}'}^{-}(j,\mathbf{B}{}_{j,K}))
\]

or expressed differently: 
\[
\forall k\in\{1,\ldots,K\},h(\Gamma_{\mathcal{S}}^{-}(i,\mathbf{A}_{i,k}))=\Gamma_{\mathcal{S}'}^{-}(j,\mathbf{B}{}_{j,k})
\]

it is straightforward to conclude from this that we also have

\begin{align*}
 & \forall k\in\{1,\ldots,K\},\\
 & h(\Gamma_{\mathcal{S}}^{-}(i,\mathbf{A}_{i,k}))=\Gamma_{\mathcal{S}'}^{-}(j,\mathbf{B}{}_{j,k})\\
 & h(\Gamma_{\mathcal{S}}^{+}(i,\mathbf{A}_{i,k}))=\Gamma_{\mathcal{S}'}^{+}(j,\mathbf{B}{}_{j,k})\\
 & h(\Gamma_{\mathcal{S}}(i,\mathbf{A}_{i,k}))=\Gamma_{\mathcal{S}'}(j,\mathbf{B}{}_{j,k})
\end{align*}

or stated differently, that:

\begin{align*}
 & \forall k\in\{1,\ldots,K\},\forall z\in\mathcal{S}\\
 & z\in\Gamma_{\mathcal{S}}^{-}(i,\mathbf{A}_{i,k})\Longleftrightarrow h(z)\in\Gamma_{\mathcal{S}'}^{-}(j,\mathbf{B}{}_{j,k})\\
 & z\in\Gamma_{\mathcal{S}}^{+}(i,\mathbf{A}_{i,k})\Longleftrightarrow h(z)\in\Gamma_{\mathcal{S}'}^{+}(j,\mathbf{B}{}_{j,k})\\
 & z\in\Gamma_{\mathcal{S}}(i,\mathbf{A}_{i,k})\Longleftrightarrow h(z)\in\Gamma_{\mathcal{S}'}(j,\mathbf{B}{}_{j,k})
\end{align*}

or equivalently

\begin{align*}
 & \forall k\in\{1,\ldots,K\},\forall z'\in\mathcal{S}',\\
 & h^{-1}(z')\in\Gamma_{\mathcal{S}}^{-}(i,\mathbf{A}_{i,k})\Longleftrightarrow z'\in\Gamma_{\mathcal{S}'}^{-}(j,\mathbf{B}{}_{j,k})\\
 & h^{-1}(z')\in\Gamma_{\mathcal{S}}^{+}(i,\mathbf{A}_{i,k})\Longleftrightarrow z'\in\Gamma_{\mathcal{S}'}^{+}(j,\mathbf{B}{}_{j,k})\\
 & h^{-1}(z')\in\Gamma_{\mathcal{S}}(i,\mathbf{A}_{i,k})\Longleftrightarrow z'\in\Gamma_{\mathcal{S}'}(j,\mathbf{B}{}_{j,k})
\end{align*}

which we may also write

\begin{align*}
 & \forall k\in\{1,\ldots,K\},\forall z'\in\mathcal{S}',\\
 & z'_{j}<\mathbf{B}_{j,k}\Longleftrightarrow h^{-1}(z')_{i}<\mathbf{A}_{i,k}\\
 & z'_{j}>\mathbf{B}_{j,k}\Longleftrightarrow h^{-1}(z')_{i}>\mathbf{A}_{i,k}\\
 & z'_{j}=\mathbf{B}{}_{j,k}\Longleftrightarrow h^{-1}(z')_{i}=\mathbf{A}_{i,k}
\end{align*}

which is what we needed to prove in the main grid structure recovery
theorem.

\paragraph{Case 2: axis reversal $s_{i}=-1$, $h(\Gamma_{\mathcal{S}}^{-}(i,\mathbf{A}_{i,1}))=\Gamma_{\mathcal{S}'}^{+}(i,\mathbf{B}_{i,K})$
and $h(\Gamma_{\mathcal{S}}^{+}(i,\mathbf{A}_{i,K}))=\Gamma_{\mathcal{S}'}^{-}(i,\mathbf{B}_{i,1})$}

We can follow the exact same reasoning steps as in case 1, starting
from $h(\Gamma_{\mathcal{S}}^{-}(i,\mathbf{A}_{i,1}))=\Gamma_{\mathcal{S}'}^{+}(i,\mathbf{B}_{i,K})$:
\begin{itemize}
\item to first show that $h(\mathcal{A}^{-})=\mathcal{B}^{+}$ i.e. 
\[
\{h(\Gamma_{\mathcal{S}}^{-}(i,\mathbf{A}_{i,1})),\ldots,h(\Gamma_{\mathcal{S}}^{-}(i,\mathbf{A}_{i,K}))\}=\{\Gamma_{\mathcal{S}'}^{+}(j,\mathbf{B}{}_{j,1}),\ldots,\Gamma_{\mathcal{S}'}^{+}(j,\mathbf{B}{}_{j,K})\}
\]
\item then use the preservation of the order relation defined by inclusion
of halves to establish that
\[
(h(\Gamma_{\mathcal{S}}^{-}(i,\mathbf{A}_{i,1})),\ldots,h(\Gamma_{\mathcal{S}}^{-}(i,\mathbf{A}_{i,K})))=(\Gamma_{\mathcal{S}'}^{+}(j,\mathbf{B}{}_{j,K}),\ldots,\Gamma_{\mathcal{S}'}^{+}(j,\mathbf{B}{}_{j,1}))
\]
\item thus that
\begin{align*}
 & \forall k\in\{1,\ldots,K\},\\
 & h(\Gamma_{\mathcal{S}}^{-}(i,\mathbf{A}_{i,k}))=\Gamma_{\mathcal{S}'}^{+}(j,\mathbf{B}{}_{j,K-k+1})\\
 & h(\Gamma_{\mathcal{S}}^{+}(i,\mathbf{A}_{i,k}))=\Gamma_{\mathcal{S}'}^{-}(j,\mathbf{B}{}_{j,K-k+1})\\
 & h(\Gamma_{\mathcal{S}}(i,\mathbf{A}_{i,k}))=\Gamma_{\mathcal{S}'}(j,\mathbf{B}{}_{j,K-k+1})
\end{align*}
\item conclude that
\begin{align*}
\forall k & \in\{1,\ldots,K\},\forall z'\in\mathcal{S}',\\
 & z'_{j}>\mathbf{B}_{j,k}\Longleftrightarrow h^{-1}(z')_{i}<\mathbf{A}_{i,K-k+1}\\
 & z'_{j}<\mathbf{B}_{j,k}\Longleftrightarrow h^{-1}(z')_{i}>\mathbf{A}_{i,K-k+1}\\
 & z'_{j}=\mathbf{B}_{j,k}\Longleftrightarrow h^{-1}(z')_{i}=\mathbf{A}_{i,K-k+1}
\end{align*}
which is what we needed to prove in the main grid structure recovery
theorem.
\end{itemize}

\section{Background on non-removable discontinuities}
\label{app:nonremovable}

The definition of continuity of a function is leveraged in section \ref{sec:discontinuities-preservation}.

\begin{definition}
\label{def:cont}
    A function $f$ is continuous at a point $x_0$ if
    
    $\forall \epsilon > 0; \ \exists \delta > 0 \ : \ d(x, x_0) < \delta \ \Rightarrow \ \vert f(x) - f(x_0) \vert < \epsilon$

    for $x$ in the domain of $f$ and $d(x, x_0)$ being the distance between points $x$ and $x_0$.
\end{definition}

That is, for any positive real number $\epsilon$, which can be infinitely small, there exists a positive real number $\delta$ such that for $x$ in the interval $x_0 - \delta < x < x_0 + \delta$, the function of $x$ will be at the interval  $f(x_0)-\epsilon < f(x) < f(x_0)+\epsilon$. So for $f(x)$ to be in a small neighborhood around $f(x_0)$, $x$ can be chosen in a small neighborhood around $x_0$.

\begin{definition}
    \citep{marsden1993elementary} A function $f: A \subset M \rightarrow N$ is called \textbf{continuous on} the set $B \subset A$ if $f$ is continuous at each point of B. If we just say that $f$ is \textbf{continuous}, we mean that $f$ is continuous on its domain $A$.
\end{definition}

These definitions are relevant because when a function is continuous, Definition \ref{def:cont} will hold at all the points in its domain. However, there can be cases where a function is not continuous at a point, but it is continuous almost everywhere. 
We define the \textbf{removable discontinuity} of a PDF $p$ at point $x_0$ as a discontinuity that can be removed by mapping it to another PDF $p'$ that has the same probability measure. The idea is that the area under the curve is the same in the equivalent PDF without the discontinuity at any interval, as illustrated in Figure \ref{fig:removable}. More precisely, the PDF $p_x$ represents a probability distribution where we can evaluate the probability at an interval by integrating over it, such as $\text{Pr}[a \leq X \leq b] = \int_a^b p_X(x)dx$ for $a,b$ belonging to a \textit{measurable set} $\mathcal{M}$ \footnote{We refer to \citep{capinski2013measure} Chapter 2, definition 2.3, for coverage on measurable sets.}. Hence, a probability distribution $Q$ maps a measurable set $\mathcal{M}$ to $[0,1]$. $Q:\mathcal{M} \rightarrow [0,1]$. We notice and exemplify in the figure that multiple PDFs can represent the same probability distribution. When PDFs represent the same probability distribution, we will use the terminology that they belong to the same \textit{equivalence class}.

\begin{figure}[H]
    \centering
    \includegraphics[scale=0.6]{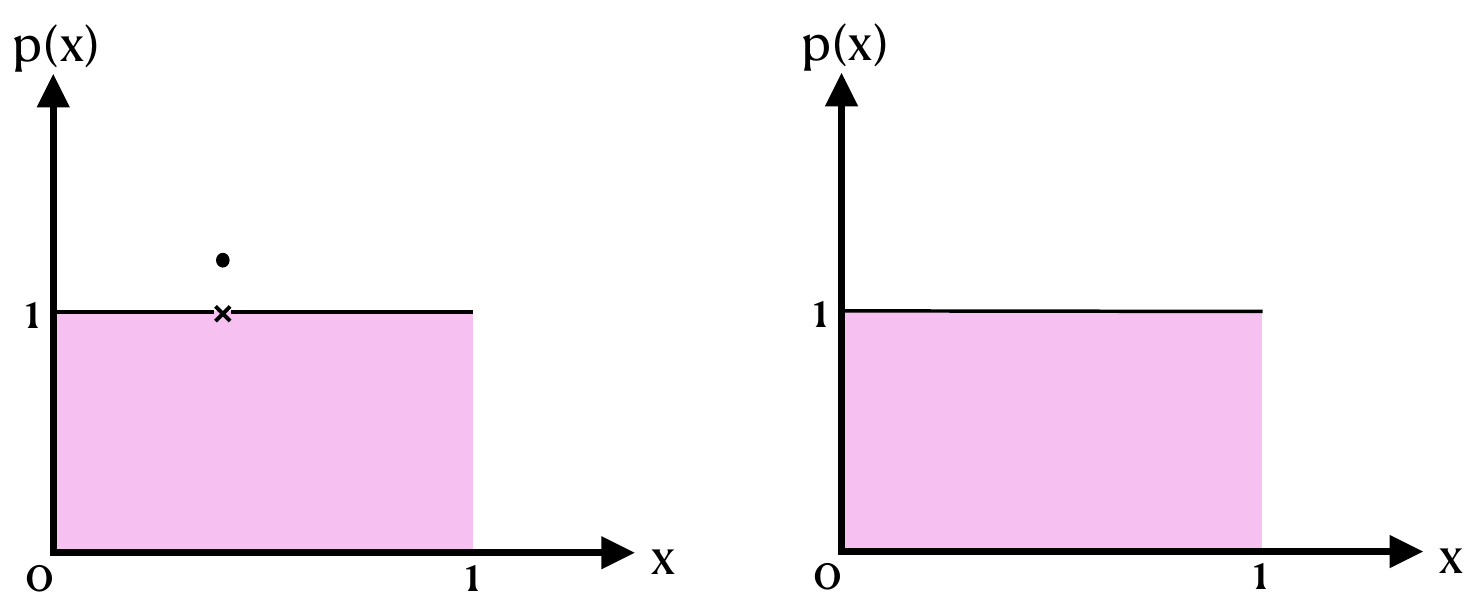}
    \caption{The PDF on the left has a removable discontinuity, but it can be mapped to the PDF on the right, which is identical but continuous everywhere. In whatever interval taken within the domain, the area under the PDF is exactly the same for both of them.}
    \label{fig:removable}
\end{figure}

A \textbf{non-removable discontinuity} is the type of discontinuity that cannot be removed because the discontinuity affects the area below the PDF and therefore all the PDFs in the same equivalence class present a discontinuity at $x_0$, as illustrated in Figure \ref{fig:nonremovable}.

\begin{figure}[H]
    \centering
    \includegraphics[scale=0.6]{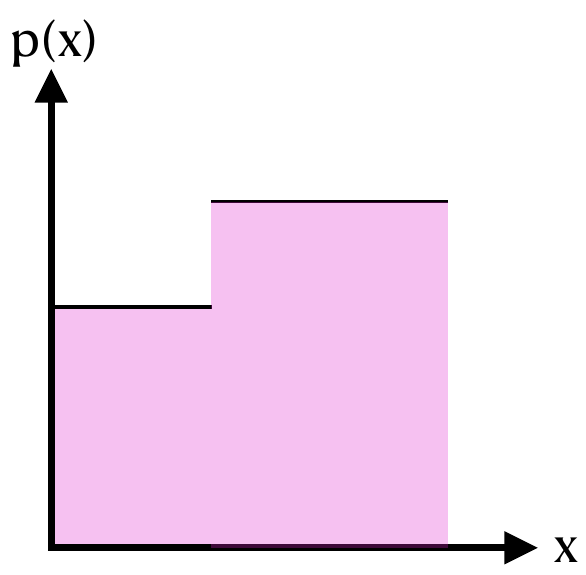}
    \caption{Example of a PDF with a non-removable discontinuity. The area under the PDF is affected by the discontinuity.}
    \label{fig:nonremovable}
\end{figure}

\section{On the necessity of axis-aligned landmarks in the probability density \texorpdfstring{$p_Z$}{ } }
\label{app:necessity_landmarks}

Our justification is built on a powerful result, Theorem 6 from \citet{buchholz2022function}. We restate and revisit the intuitions behind the result below.  We use the notation $\Phi_{*}\mathbb{P}$ to denote the pushforward of $\mathbb{P}$ under $\Phi$. 

\begin{theorem}
    Let $p_Z$ be a twice differentiable probability density with bounded gradient. Suppose that $x = \Phi(z)$ where the distribution $\mathbb{P}$ of $z$ has density $p_Z$ and $\Phi$ is a diffeomorphism with $\det D\Phi(z) = 1$ for $z \in \mathbb{R}^d$. Then there is a family of functions $\Phi_t : \mathbb{R} \times \mathbb{R}^d \rightarrow \mathbb{R}^d$ indexed by $t$ with $\Phi_0 = q$ and $q_t \not= \Phi_0$ for $t \not= 0$ such that $\det D\Phi_t(z) = 1$ and $(\Phi_t)_{*}\mathbb{P} = \Phi_{*} \mathbb{P}$.
\end{theorem}

Consider the case when $\Phi$ is an identity map. In that case, the above theorem implies that there exists a family of volume preserving transformations different from the identity map such that $(\Phi_t)_{*}\mathbb{P} = \mathbb{P}$. 

We revisit the intuition behind the proof, which will be reused for the rest of this section. We define the flow of a vector field as a map $\Phi:\mathbb{R}\times\mathbb{R}^{d}\rightarrow\mathbb{R}^d$ such that 

\begin{equation}
\Phi_0(z) = z,   \partial_t \Phi_t(z) = X(\Phi_t(z))  \label{eqn: flow_q}
\end{equation}

We write $ \Phi(t,z)$ as $\Phi_t(z)$ in the above expression.  We can interpret $\Phi_t(z)$ as the position of a particle, which started at $z$ at time $t=0$. Then $X(\Phi_t(z))$ is the velocity of the particle at time $t$. Define a vector field $X^{ij}:\mathbb{R}^{d}\rightarrow \mathbb{R}^{d}$ as 

\begin{equation}
    \begin{split}
        X^{ij}_{k} = \begin{cases} 
        \;\;\partial{p}_j \;\;\; k = i \\ 
        -\partial{p}_i \;\;\; k = j \\ 
        0 \;\;\;\;\;\;\;\;\; \text{otherwise}
        \end{cases},
    \end{split}
\end{equation}
where $X^{ij}_k$ is the $k^{th}$ component of the vector field. Observe that $X^{ij}$ is orthogonal to the isolines of the density $p_Z$ along the plane corresponding to components $i$ and $j$. Under the flows described in equation \eqref{eqn: flow_q}, the probability measures evolve in time and satisfy the continuity equation stated next. Formally stated, the density associated with $(q_t)_{*}\mathbb{P}$ satisfies. 

\begin{equation}
    \partial p_t + \mathsf{Div}(p_t X^{ij}) = 0, \;\; p_0 = p
\end{equation}

Observe that $\mathsf{Div}(X^{ij}) = \partial_i \partial_j p -\partial_j\partial_i p =0 $.  Also, observe that $\mathsf{Div}(p X^{ij}) = X_{ij}\cdot \nabla p = 0$. From $\mathsf{Div}(p X^{ij})= 0$, we can infer that the $(q_t)_{*}\mathbb{P} = \mathbb{P}$ and from $\mathsf{Div}(X^{ij}) =0$, $\Phi_t$ is a volume preserving diffeomorphism. Consider the following autoencoder where the decoder is $\tilde{f}_t = f \circ q_t$ and the encoder is $\tilde{g}_t = q_t^{-1} \circ f^{-1}$. Observe that these autoencoders achieve perfect reconstruction. The encoder fails at identifying the underlying latents as the estimated latents are related to the true latents by the map $\Phi_t^{-1}(\cdot)$. Thus we have seen that even if the learner knows the true density $p_Z$, there exists a family of autoencoders that cannot achieve identification. 

Suppose $Z=[Z_1,Z_2]$. Consider that $p_Z$ is a density defined over the unit square centered at origin. In Figure~\ref{fig:appendix_illustration_2d}, we show one such density. Suppose we are interested in separating the points in the four quadrants, where each quadrant provides a distinct quantization for all the values of $z$ assumed in it. We  consider a family of densities $p_Z$, whose isolines crosses $z_1=0$ at least once. We further also assume that these densities $p_Z$ are differentiable over the support and have a bounded gradient. The isoline of this density crosses $z_1=0$. Consider two points shown in pink and blue colors in Figure~\ref{fig:appendix_illustration_2d}.
At $t=0$, the pink point is to the right of $z_1=0$ and the blue point is to the left of $z_1=0$. Under the flow defined in equation \eqref{eqn: flow_q}, we would argue that after some time $\tau$ has elapsed, the pink point moves to the left of $z_1=0$, while the blue point is still to the left of $z_1=0$. Under the map $\Phi_{\tau}(z) = (q_{\tau}^{1}(z),q_{\tau}^{2}(z)) $, the two points are both to the left of $z_1=0$. If all the points $\Phi_{\tau}^{1}(z)<0$ are associated with the same quantization, then the pink point and the blue point are associated with same quantization, while their true quantization is different. We provide further details on the construction. We can assume that the pink point at $t=0$ is very close to $z_1=0$ but on the right of $z_1=0$. Further, we assume that the magnitude of the flow along the negative $z_1$ direction (which is the vector tangent to the isoline) in the neighborhood of the pink point at $t=0$ is bounded below by at least $\epsilon_1$. As a result, the point moves at least $\tau\epsilon$ distance along $z_1$ in time $\tau$. We can assume that the pink point started with $z_1<\epsilon\tau$ and thus it crosses to the left of $z_1=0$ after $\tau$ amount of time has elapsed.  At the same time, consider the blue point to the left of $z_1=0$ at time $t=0$. We assume that the flow along the $z_1$ direction in the $\rho$ radius neighborhood of the blue point at $t=0$ points from right to left, i.e., it is in the negative $z_1$ direction. We choose $\tau$ to be sufficiently small, $\tau < \rho/\gamma$, where $\gamma$ is the largest value that the velocity under the flow can take (this follows from the assumption that the density has a bounded gradient). Under this constraint, the blue point stays to the left of $z_1=0$ after $\tau$ amount of time has elapsed.

\begin{figure}
    \centering
    \includegraphics[trim=2in 4in 2in 2in, width=5in]{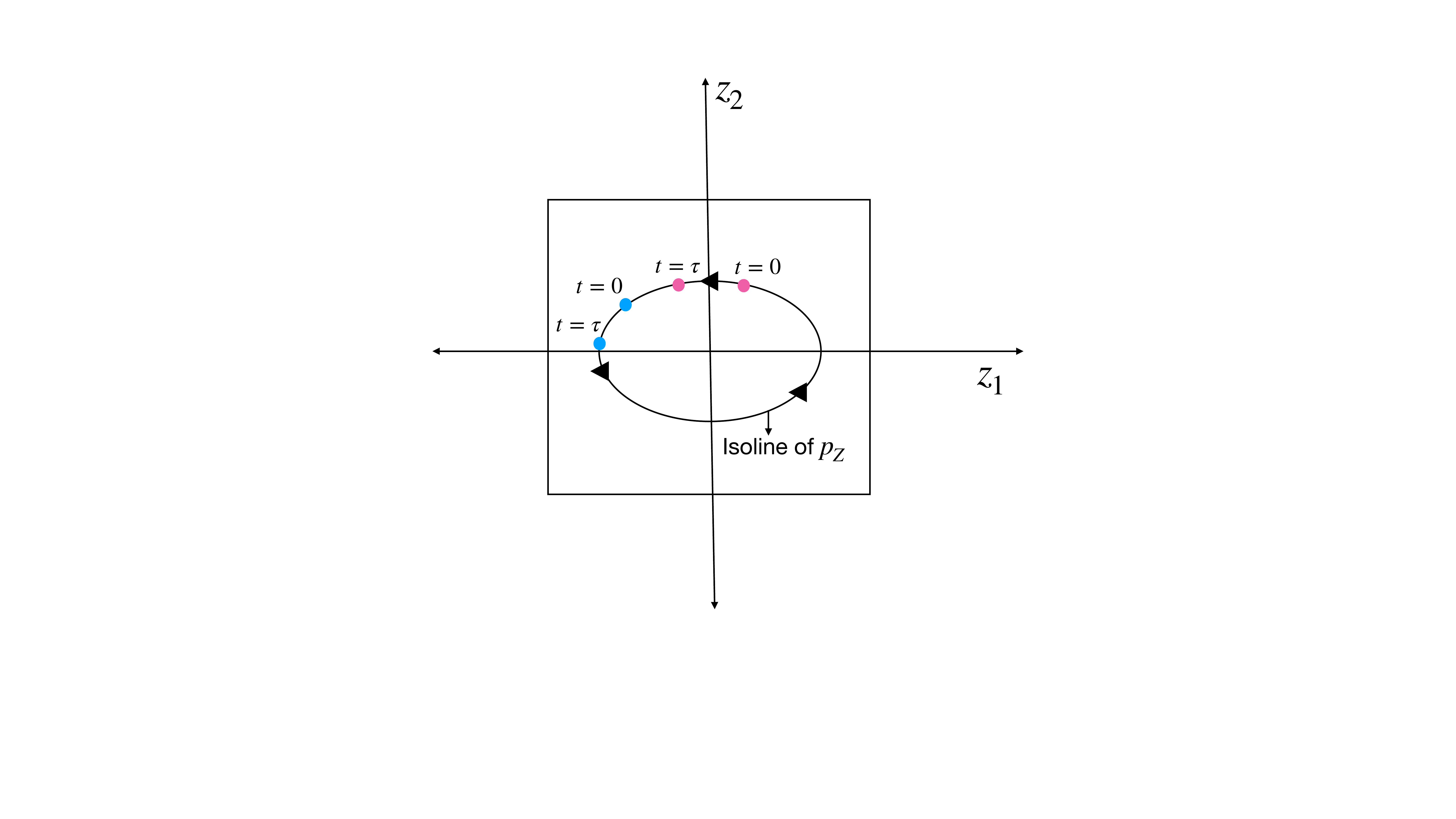}
    \caption{Two dimensional illustration to explain why isolines cannot cross the axis.}
    \label{fig:appendix_illustration_2d}
\end{figure}

The above argument explains that if the isolines of $p_Z$ cross the grids, then quantized identification is not possible even if we know $p_Z$. As a result, we need to focus on the densities $p_Z$ whose isolines are restricted to each of the four quadrants. Consider the family of densities $p_Z$ with axis-aligned discontinuities. Suppose the density in each of the quadrants is continuous, then for this class of densities the isolines cannot cross any of the axis $z_1=0$ and $z_2=0$. For this class of densities, our theory established quantized factor identification guarantees even without requiring knowledge of $p_Z$. Could there be other densities beyond discontinuous densities with axis-aligned landmarks that permit quantized identification? This is a fairly non-trivial and important question left for future work. 

\end{document}